\documentclass{article} 
\usepackage{iclr2027_conference,times}

\usepackage{amsmath,amsfonts,bm}

\def\eqref#1{equation~\ref{#1}}

\def\1{\bm{1}}

\DeclareMathAlphabet{\mathsfit}{\encodingdefault}{\sfdefault}{m}{sl}
\SetMathAlphabet{\mathsfit}{bold}{\encodingdefault}{\sfdefault}{bx}{n}

\usepackage{hyperref}
\usepackage{url}

\usepackage{amsmath}
\usepackage{amssymb}
\usepackage{mathtools}
\usepackage{amsthm}
\usepackage{wrapfig}
\usepackage{mathtools}
\usepackage{amsmath, amssymb, amsthm}
\usepackage{bm}
\usepackage{enumitem}
\usepackage{multirow}
\usepackage{makecell}
\usepackage{siunitx}
\usepackage{adjustbox}
\usepackage{tabularx}
\usepackage{diagbox}
\usepackage[most]{tcolorbox}
\usepackage{anyfontsize}
\usepackage{amsfonts}       
\usepackage{nicefrac}       
\usepackage{xcolor}         
\usepackage{natbib, mathtools}
\usepackage{multirow}
\usepackage{graphicx}
\usepackage[table]{xcolor}
\usepackage{array}
\usepackage{siunitx}
\usepackage{ragged2e}
\usepackage{listings}
\newtheorem{theorem}{Theorem}[section]
\newtheorem{proposition}[theorem]{Proposition}
\newtheorem{lemma}[theorem]{Lemma}
\newtheorem{corollary}[theorem]{Corollary}
\theoremstyle{remark}
\newtheorem{remark}[theorem]{Remark}
\renewcommand{\eqref}[1]{\textup{(\ref{#1})}}
\usepackage{multirow}
\usepackage{makecell}
\usepackage{booktabs}
\usepackage{tabularx}
\newcolumntype{Y}{>{\centering\arraybackslash}X}
\usepackage{subcaption}
\usepackage{algorithm}
\usepackage{algpseudocode}
\usepackage{longtable,pdflscape}
\usepackage{colortbl}

\title{Beyond Action Entropy: Quotient-Space \\ Exploration for Genome-Scale Metabolic Model Repair}

\author{Xuan Gong$^{\star,1}$, \quad Hanbo Huang$^{\star,1}$, \quad Wenbin Dai$^{\star,1}$, \quad Jing Wang$^{1}$, \quad Lei Bai$^{2}$, \\
  \bf Xiang Xiao$^{1}$, \quad Weishu Zhao$^{1}$, \quad Shiyu Liang$^{\dagger,1}$ \\[1ex]
  \normalfont $^1$Shanghai Jiao Tong University \quad
  $^2$Shanghai AI Laboratory \\
  \normalfont \texttt{\{gongxuan0610, lsy18602808513\}@sjtu.edu.cn}
}

\iclrfinalcopy 

\begin{document}

\maketitle
{\let\thefootnote\relax\footnotetext{\hspace{-1.8em}$^\star$Equal contribution. \;
$^\dagger$Corresponding author. Correspondence to \texttt{lsy18602808513@sjtu.edu.cn}. \\
The resources will be released at \url{https://github.com/coder-gx/GEM_repair}.}}

\lhead{Preprint}

\begin{abstract}
Repairing scientific models from functional observations differs fundamentally from supervised prediction: feedback may certify a solution without revealing which structural correction is responsible.
We study this setting for genome-scale metabolic model (GEM) repair, where multiple reaction edits can explain the same phenotypes and many apparently distinct edits correspond to the same biological mechanism.
This many-to-one structure creates a hidden failure mode for conventional exploration: diversity in the output space need not translate into diversity of scientific hypotheses.
We introduce \textsc{QuotientPO}, which collapses equivalent repairs into canonical mechanisms and optimizes exploration directly over the resulting quotient space.
To make quotient exploration informative under finite rollouts, we derive a kernelized R\'enyi estimator that resolves graded crowding among distinct repair cores beyond coarse exact-match counts.
On 2,212 held-out GEMs, \textsc{QuotientPO} improves Success@32 from 17.93\% to 20.10\% (+12.1\% relative) while consistently increasing distinct successful-core discovery under the same sampling budget.
These results establish quotient-space exploration as a principled approach
to mechanism-level discovery under verifier-induced equivalence.
\end{abstract}



\section{Introduction}
\label{sec:introduction}

AI4Sci~\citep{wang2023scientific} is increasingly moving beyond predicting scientific observations toward reasoning over, generating, and modifying scientific models and structures~\citep{abramson2024alphafold3,merchant2023scaling}.
Unlike ordinary prediction targets, these scientific objects act as hypotheses about underlying mechanisms whose validity is assessed indirectly through physical laws~\citep{karniadakis2021physics} or experimental feedback~\citep{boiko2023autonomous}.
Evaluating candidates against functional constraints rather than gold targets characterizes \emph{verifier-guided scientific AI}, posing the core challenge of learning structural corrections from indirect evidence alone.
This challenge directly materializes in genome-scale metabolic models (GEMs)~\citep{mardinoglu2025genomescale}, which operationalize genomic sequences into executable metabolic machinery.
To overcome this barrier, we introduce \emph{phenotype-supervised GEM repair}: an amortized framework that replaces costly manual curation to scale predictive mechanistic modeling to uncharacterized organisms.

This repair paradigm is fundamentally challenged by partial identification: phenotype observations validate models without uniquely determining underlying mechanisms~\citep{raue2009structural,benedict2014likelihood}.
This underdetermination creates a core tension: while multiple distinct mechanisms can satisfy identical phenotypes, combinatorial edits often implement the exact same mechanism, rendering surface diversity deceptive.
Prior work addresses non-surface redundancy via semantic classes, kernels, or symmetries \citep{kuhn2023semantic,nikitin2024kernellanguageentropyfinegrained,kim2025symmetryawaregflownets,li2026setposetlevelpolicyoptimization}. Yet in verifier-guided repair, a single functional intervention admits diverse raw realizations, meaning output similarity cannot define which candidates constitute the same scientific hypothesis. Here, identity is induced operationally by the verifier itself. This raises a key question: \textit{How should exploration be structured when diversification is defined up to verifier-preserving transformations?}

Our answer is to \emph{quotient before exploring}.
Conceptually, we partition repair edits into verifier-preserving equivalence classes through canonicalization, establishing scientific identity instead of syntactic multiplicity as the unit of diversity.
Computationally, exact identity provides only a coarse finite-sample exploration signal, becoming especially degenerate when exact mechanism collisions are sparse.
We therefore relax binary collision counting into a graded, kernelized R\'enyi entropy objective \citep{Principe2000,giraldo2014measuresentropydatausing}, allowing mechanistic crowding to be detected even among non-identical repairs.
Based on these principles, we introduce \textsc{QuotientPO} (Fig.~\ref{fig:overview}), an algorithm combining group-relative policy optimization with quotient-space exploration.
Together, our theory and experiments establish quotient-space exploration as a principled route to broader valid-mechanism coverage, advancing a Sci4AI perspective on verifier-guided learning.

Our contributions are threefold:
\begin{itemize}[leftmargin=*, itemsep=2pt, parsep=0pt, topsep=1pt]
    \item \textbf{Formulation and Benchmark for Scientific Repair:}
    We formulate \emph{phenotype-supervised GEM repair} to generalize edits across organisms without test-time phenotypes, and establish the first large-scale benchmark for amortized GEM correction.

    \item \textbf{Theoretical Foundation of Quotient Exploration:}
    We identify \emph{representation multiplicity} as an intrinsic failure mode in which standard policy entropy rewards syntactic redundancy, and establish theoretical connections between quotient-space diversity and deployed hypothesis coverage.

    \item \textbf{Tractable Algorithm and Empirical Validation:}
    We develop \textsc{QuotientPO}, making quotient exploration practical via kernelized R\'enyi estimation of graded mechanistic crowding, and demonstrate consistent gains in multi-budget repair and distinct mechanism recovery.
\end{itemize}

\begin{figure}[t]
    \centering
\includegraphics[width=\linewidth]{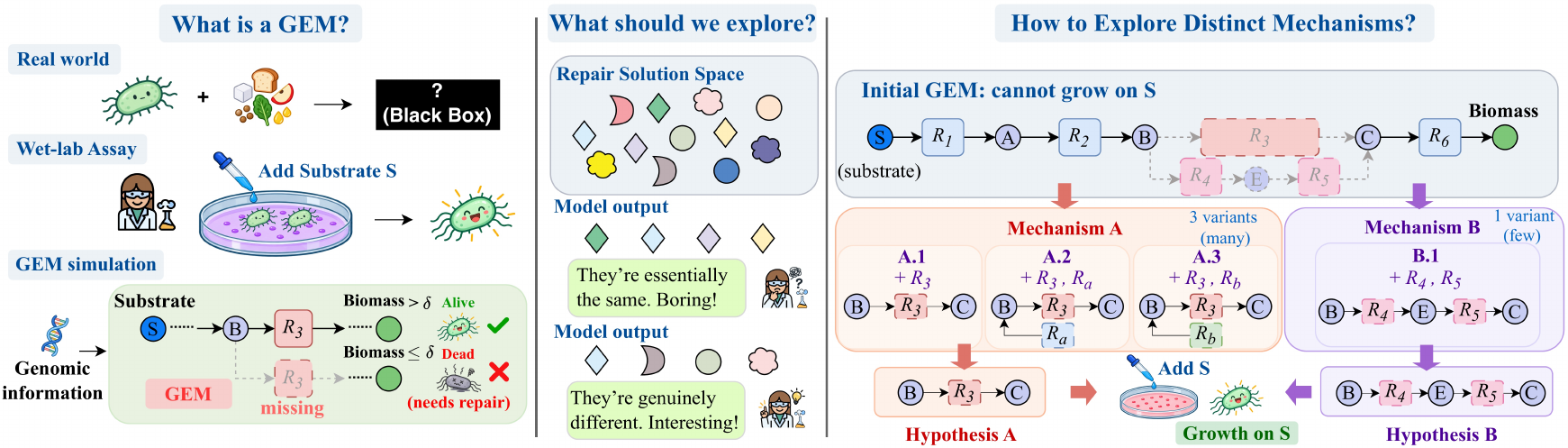}
\vspace{-0.4cm}
   \caption{\textbf{Overview.} \textsc{QuotientPO} canonicalizes equivalent surface edits to explore distinct mechanistic hypotheses rather than redundant syntactic variants for GEM repair.}
    \label{fig:overview}
   \vspace{-0.4cm}
\end{figure}

\vspace{-0.2cm}
\section{Background and Related Work}
\label{sec:related_work}
\vspace{-0.2cm}

\textbf{Genome-scale metabolic reconstruction and model repair.}
GEMs support mechanistic phenotype simulation and metabolic engineering, yet incomplete biological knowledge and reconstruction choices can leave missing and spurious reactions \citep{thiele2010protocol,carveme,zimmermann2021gapseq,hsieh2024comparative,huang2026ggboundgenomegroundedagentmicrobial}.
Classical constraint-based methods repair each GEM independently through gap filling or phenotype reconciliation, without sharing repair knowledge across organisms.
The same phenotypic constraints may admit multiple valid repairs \citep{satish2007optimization,kumar2009growmatch,thiele2014fastgapfill,hartleb2016improved}.
Recent learning-based methods recover missing reactions through link or hypergraph prediction \citep{yadati2020nhp,chen2023teasing,huang2024dual,liu2024generalizable,zhao2026multi}, but largely target synthetically removed reactions rather than functional bidirectional repair.
We instead learn a shared repair policy across organisms, using phenotypes only as training supervision to predict functionally validated bidirectional edits without test-time phenotype access.


\textbf{Diversity-aware reinforcement learning for structured prediction.}
Reinforcement learning (RL) optimizes structured outputs under non-differentiable functional rewards, including molecular design \citep{you2018graph}, executable code \citep{le2022coderl}, and outcome-based reasoning \citep{shao2024deepseekmathpushinglimitsmathematical}.
Diversity has been pursued through maximum-entropy policies \citep{haarnoja2018soft}, unsupervised skill discovery \citep{eysenbach2018diversityneedlearningskills,Sharma2020Dynamics-Aware}, quality-diversity optimization \citep{pmlr-v235-grillotti24a}, and GFlowNets for reward-weighted object generation \citep{bengio2021flow,pmlr-v162-jain22a}.
Recent work further considers semantic equivalence and kernelized set diversity \citep{kuhn2023semantic,li2026setposetlevelpolicyoptimization}.
However, these approaches do not resolve the verifier-induced many-to-one mapping central to GEM repair, where distinct edit sets can collapse to the same repair core.
Raw-output diversity can therefore reward representation multiplicity without expanding distinct repair hypotheses, motivating quotienting before exploration.

\section{Preliminaries and Problem Formulation}
\label{sec:problem}

\textbf{Genome-scale metabolic models.} A GEM represents an organism's metabolic network as a directed bipartite graph $M=(\mathcal C,\mathcal R,S)$ between metabolites $\mathcal C$ and reactions $\mathcal R$~\citep{10.1371/journal.pcbi.1000385,doi:10.1073/pnas.0803571105,10.1093/bioinformatics/btr145}. The stoichiometric matrix $S \in \mathbb R^{|\mathcal C| \times |\mathcal R|}$ defines weighted edges between them, specifying the consumption and production coefficients of each reaction.

\textbf{Formulation of GEM repair.}
We frame GEM correction as constrained graph editing through reaction additions and deletions. Given an initial model $M_0$ with reaction set $\mathcal R_0$ and a candidate reaction universe $\mathcal U$, a repair $E=(A,D)$ adds $A\subseteq\mathcal U\setminus\mathcal R_0$ and removes $D\subseteq\mathcal R_0\cap\mathcal U$, yielding the repaired model $M_E$. 
For each tested phenotype $i$, flux balance analysis (FBA;\citealp{Orth2010}) over its nutrient-specific assay(s) computes optimal biomass flux $v_i^*(M_E)$, giving predicted phenotype $\widehat y_i(M_E)=\mathbb I[v_i^*(M_E)>\delta]$ under growth threshold $\delta>0$.
Given observed BacDive phenotypes $y\in\{0,1\}^N$, repair quality is measured by phenotype consistency $\Phi_y(M_E)=\frac{1}{N}\sum_{i=1}^N\mathbb I[\widehat y_i(M_E)=y_i]$. FBA therefore serves directly as the functional verifier that maps a repaired GEM to phenotype predictions. 
The finite repair space $\mathcal{F}$ contains all legal reaction edits over candidate universe $\mathcal{U}$.

\textbf{Learning problem.}
Let \(G\) denote the organism-specific protein representation derived from its genome-encoded proteins.
We assume a training distribution \(\mathcal D\) over organisms, where for each training instance \((M_0,G,y)\sim\mathcal D\), \(y\) contains the available wet-lab phenotypes.
Neither a corrected GEM nor a gold reaction edit set is available, and the
observed phenotypes generally do not uniquely identify the underlying repair. We therefore learn a stochastic repair policy \(\pi_\theta(E\mid M_0,G)\) over valid repairs.
Importantly, the phenotypes $y$ provide training supervision through $\Phi_y(M_E)$ only
and are never given as model input. The resulting learning objective is
\begin{equation}
\label{eq:objective}
\max_{\theta}\,
\mathbb E_{(M_0,G,y)\sim\mathcal D}
\mathbb E_{E\sim \pi_\theta(\cdot\mid M_0,G)}
\left[\Phi_y(M_E)\right]
\quad\text{s.t.}\quad
E \in \mathcal{F}.
\end{equation}
At test time, only $(M_0,G)$ is available and the learned policy outputs a
repair $E$ and the corresponding $M_E$.
Because no organism-specific corrected GEM or gold edit set is available,
evaluation uses only held-out phenotype consistency target.

\section{Methodology}
\label{sec:method}

The formulation in Eq.~(\ref{eq:objective}) directly motivates \textbf{policy-gradient RL}\footnote{We frame our learning under group-relative policy optimization (GRPO)~\citep{shao2024deepseekmathpushinglimitsmathematical}, which baselines a critic-free rollout group to avoid critic variance under combinatorial edits~\citep{bello2017neuralcombinatorialoptimizationreinforcement,kool2019attentionlearnsolverouting} while providing the sample pool for quotient estimation.}  \textbf{with explicit multi-solution exploration}:
\textbf{(1)} \textit{Bypassing non-differentiable verifiers and solver bias}: Repair quality is evaluated by non-differentiable FBA-based verifiers, preventing direct backpropagation. While offline mixed-integer linear programming (MILP)~\citep{satish2007optimization,benedict2014likelihood,SCHROEDER2020100783} can provide demonstrations for supervised fine-tuning (SFT), it arbitrarily breaks ties across non-unique feasible sets, biasing SFT toward solver heuristics. We therefore adopt RL to optimize verifier rewards directly beyond offline demonstrations.
\textbf{(2)} \textit{Preventing collapse under partial identification}: Across a combinatorial graph edit space, observed phenotypes only partially identify the metabolic network, creating a landscape with multiple discrete optima where mechanistically distinct repairs achieve identical rewards. A policy collapsed onto any single valid repair is already reward-optimal, leaving standard reward maximization with zero gradient signal to discover alternative mechanisms. Under finite rollout budgets, this wastes samples on redundant variations around a single mode, motivating explicit exploration to preserve diverse scientific hypotheses.

\begin{figure}[t]
    \centering
\includegraphics[width=\linewidth]{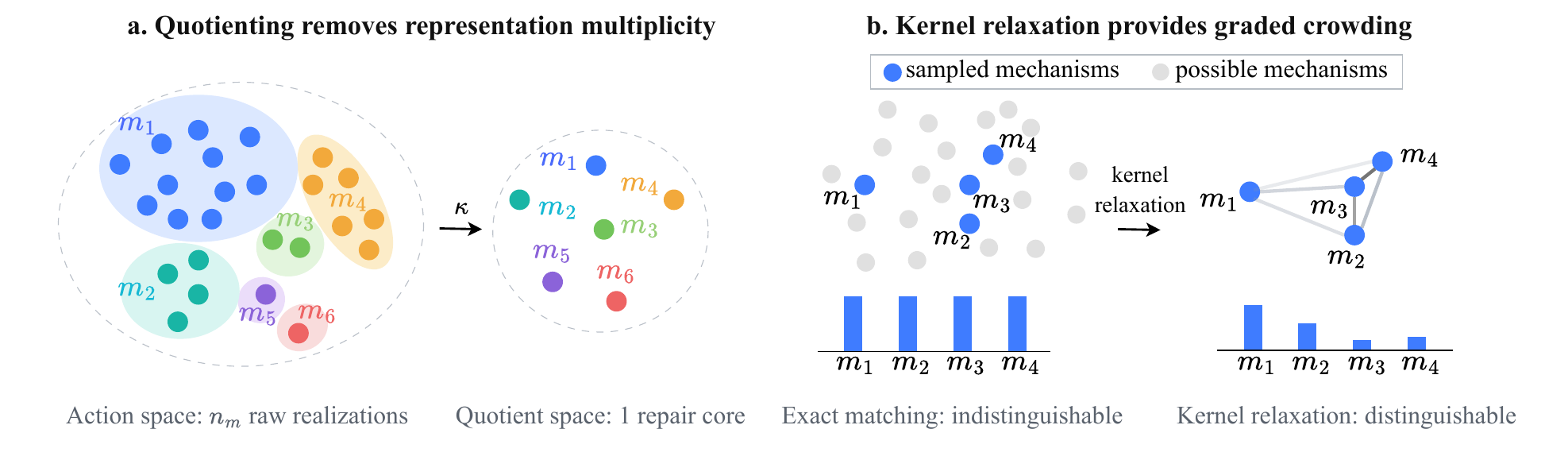}
 \setlength{\abovecaptionskip}{-8pt}
\caption{
\textbf{Theoretical intuition.}
\textbf{(a)} Each core $m$ contains $n_m$ equivalent raw edits contracted by $\kappa$ into a single point, preventing large $n_m$ from distorting exploration.
\textbf{(b)} Bar heights show exploration credit. Exact matching gives identical credit to equal-count cores, whereas kernel similarity favors isolated discoveries ($m_1$) over clustered candidates ($m_2$--$m_4$).
}
    \label{fig:theory}
   \vspace{-0.4cm}
\end{figure}

\subsection{Representation-Invariant Exploration with Quotient Entropy}
\label{sec:exploration}

\textbf{Action entropy rewards representation multiplicity, not mechanistic diversity.}
A natural exploration bonus is the Shannon entropy over raw repair actions $H(\pi_\theta)$, leading to the standard entropy-regularized objective $\mathcal{J}_{\mathrm{act}}(\theta) = \mathbb{E}_{(M_0,G,y)\sim\mathcal{D}} \left[\mathbb{E}_{E\sim\pi_\theta} [\Phi_y(M_E)] + \beta H(\pi_\theta)\right]$, where $\beta > 0$ controls the exploration strength.
The problem is that entropy counts representations, not repair hypotheses. If two equally rewarded hypotheses differ only in that one admits $n$ equivalent action realizations, action entropy gives it an additional preference of $\beta\log n$. Thus, a hypothesis is favored purely by representation redundancy, rather than any distinct biological mechanism. We formalize this multiplicity bias in Appendix~\ref{app:exploration}.

\textbf{Exploring distinct mechanisms requires quotienting equivalent repairs.}
To remove this bias, equivalent repairs should contribute only once to exploration. 
We define a deterministic, verifier-preserving map $\kappa_y$ that removes redundant edits, contracting the $n_m$ equivalent raw realizations into a canonical repair core $m = \kappa_y(E)$ (Fig.~\ref{fig:theory}a).
Let $\mathcal M_y$ denote the set of repair cores, serving as operational representations of underlying mechanisms. Repairs sharing the same core are treated as equivalent hypotheses, and the policy induces the quotient distribution $q=q_{\theta,y}$, where
\(
q(m)=\sum_{E:\kappa_y(E)=m}\pi_\theta(E\mid M_0,G)
\) aggregates probability mass across all equivalent raw realizations.


\begin{theorem}[Within-core invariance characterizes quotient exploration]
\label{thm:quotient_factorization}
An exploration functional $\mathfrak D$ is invariant to within-core
redistribution if and only if it depends on the raw policy $\pi$ only through its induced
core distribution $q_\pi$. Equivalently, there exists a functional
$\overline{\mathfrak D}$ such that
\(
\mathfrak D(\pi)=\overline{\mathfrak D}(q_\pi).
\)
\end{theorem}

Theorem~\ref{thm:quotient_factorization} formalizes why exploration should depend only on the induced distribution over repair cores. The intuition is direct: splitting probability across equivalent realizations inflates action entropy, whereas the aggregated core probability $q(m)$ remains unchanged. 
Evaluating entropy over this aggregated distribution removes multiplicity bias by construction, leading to the quotient-regularized objective:
\(
\mathcal J_{\mathrm{quot}}(\theta)
=
\mathbb E_{(M_0,G,y)\sim\mathcal D}
\left[
\mathbb E_{E\sim\pi_\theta}
\bigl[\Phi_y(M_E)\bigr]
+
\beta H(q)
\right],
\label{eq:quotient_entropy_objective}
\)
where the quotient entropy is $H(q) = -\sum_{m\in\mathcal M_y} q(m)\log q(m)$. Crucially, optimizing $\mathcal J_{\mathrm{quot}}$ weights distinct hypotheses strictly by task utility rather than raw representation multiplicity (Appendix~\ref{app:exploration}).

\subsection{Graded Quotient Estimation with R\'enyi Kernels}
\label{sec:renyi}

\textbf{Exact matching gives coarse finite-sample quotient estimates.}
Because the core probability $q(m)$ is an unobserved fiber mass rather than a closed-form likelihood, the quotient diversity must be estimated from sampled cores.
Exact-frequency estimation, however, provides only binary crowding information: structurally different cores with the same empirical multiplicity receive identical credit.
This limitation becomes extreme in the sparse-collision regime.
With $L$ sampled cores, the expected number of exact pairwise collisions is
$\binom{L}{2}\sum_m q(m)^2$.
When $\binom{L}{2}\sum_m q(m)^2 \ll 1$, a rollout group contains few or no repeated cores, making distinct samples indistinguishable under exact matching (Appendix~\ref{app:renyi}).
Thus, exact-count estimation provides only a coarse finite-sample exploration signal, even before collisions vanish entirely.

\textbf{R\'enyi kernels turn binary identity into graded mechanistic crowding.}
The key limitation is the knife-edge nature of exact matching: it distinguishes identical from non-identical cores, but provides no graded notion of crowding among distinct repairs.
Indeed, the exact pairwise collision probability
\(
\sum_m q(m)^2
\)
is precisely the quadratic R\'enyi concentration
\(
\exp[-H_2(q)]
\)
~\citep{Principe2000,giraldo2014measuresentropydatausing}.
This suggests replacing quotient Shannon entropy with quadratic R\'enyi entropy,
whose exact collision form admits a graded relaxation while preserving the quotient identity.
Representing each repair core as a signed edit set, we replace
the exact-match collision operator with the graded kernel
\(
k(m,m')=\exp(-(1-J(m,m'))/h),
\)
where $J$ is Jaccard similarity and $h>0$ controls smoothness.
As $h\rightarrow0$, $k$ recovers exact matching, whereas finite $h$ assigns progressively larger collision mass to increasingly overlapping repairs.
This kernel relaxation exposes mechanistic crowding before exact duplicates dominate (Figure~\ref{fig:theory}b), discounting exploration credit for clustered repairs while rewarding genuinely distinct bypasses. 
Replacing exact-collision R\'enyi entropy with the kernelized surrogate
\(
H_{2,k}(q)
=
-\log \mathbb E_{m,m'\sim q}[k(m,m')]
\)
gives the crowding-aware objective:
\begin{equation}
\mathcal J_{\mathrm{renyi}}(\theta)
=
\mathbb E_{(M_0,G,y)\sim\mathcal D}
\left[
\mathbb E_{E\sim\pi_\theta}
\bigl[\Phi_y(M_E)\bigr]
+
\beta H_{2,k}(q)
\right].
\label{eq:renyi_entropy_objective}
\end{equation}
Kernelized R\'enyi estimation therefore preserves the quotient notion of identity while replacing coarse binary collisions with a graded finite-sample exploration signal.

\subsection{QuotientPO: Quotient-Space Policy Optimization}
\label{sec:qrl}


QuotientPO operationalizes quotient exploration within GRPO by decomposing verifier feedback and quotient diversity into candidate-level credit.  We optimize a verifier-shaped relaxation where each component provides a principled finite-sample realization of Eq.~(\ref{eq:renyi_entropy_objective}).

\textbf{Verifier-shaped task reward.} For a candidate repair $E$, we define the terminal task reward as
\(
R_y(E)
=
\Phi_y(M_E)-\Phi_y(M_0)
+
2\, \mathbb{I}[\text{all satisfied}].
\)
Since $\Phi_y(M_E)-\Phi_y(M_0) \le 1$, the success weight of $2$ enforces a strict margin ensuring that any fully verified repair ($R_y \ge 2$) strictly outranks any partial improvement ($R_y < 1$). 
This rules out reward-level satisficing on partial repairs while retaining dense shaping rewards during early exploration.

\textbf{Group-relative quotient credit.}
For each organism, we sample $L$ repairs and canonicalize them into repair cores $m_1,\ldots,m_L$.  While the kernelized R\'enyi entropy in Section~\ref{sec:renyi} evaluates diversity collectively across the rollout, policy optimization requires candidate-level credit. A leave-one-out linearization of this group entropy around the empirical mean gives the normalized crowding cost
\begin{equation}
c_i
=
2L
\frac{\sum_{j\neq i} k(m_i,m_j)}
{\sum_{a\neq b} k(m_a,m_b)}
-2,
\qquad
\sum_{i=1}^L c_i=0.
\label{eq:qrl_kernel_credit}
\end{equation}
A positive $c_i$ indicates that a repair core is crowded by similar candidates, while a negative value indicates that it expands quotient-space coverage. The zero-sum form aligns this diversity signal with group-relative optimization. Its derivation is given in Appendix~\ref{app:optimization}.

\textbf{QuotientPO objective.}
 Let $\rho_i(\theta) = \pi_\theta(E_i \mid M_0,G) / \pi_{\theta_{\mathrm{old}}}(E_i \mid M_0,G)$, and let $A_i = (R_y(E_i) - \bar R) / (\sigma_R + \varepsilon)$ denote the group-relative advantage with sample mean $\bar R$ and standard deviation $\sigma_R$.  The complete objective is
\begin{equation}
\mathcal J(\theta)
=
\frac{1}{L}
\sum_{i=1}^L
\min\!\Bigl(
\rho_i(\theta)A_i,
\operatorname{clip}(\rho_i(\theta),1-\epsilon,1+\epsilon)A_i
\Bigr)
-
\frac{\beta}{L}
\sum_{i=1}^L
\rho_i(\theta)c_i
+
\mathcal R_{\mathrm{aux}}(\theta).
\label{eq:qpo_objective}
\end{equation}
The GRPO term optimizes verifier quality, while the quotient term shifts probability away from crowded cores toward distinct repair hypotheses.
The auxiliary term $\mathcal R_{\mathrm{aux}}$ combines QPI credit for newly discovered successful cores and an optional redundancy penalty for edits outside the canonical core. More details are given in Appendix~\ref{app:algorithm}.

\section{Experiments}
\label{sec:experiments}

\begin{table*}[!t]
\centering
\small
\setlength{\tabcolsep}{4pt}
\renewcommand{\arraystretch}{1.5}

\caption{
\textbf{Main results.}
MILP is an oracle with test-time phenotype access.
SFT methods use checkpoints selected on the validation split, while RL methods are evaluated at the final checkpoint after an identical 864-update training budget.
All results are evaluated on the test set, reporting mean $\pm$ standard deviation across three seed runs.
}
\label{tab:main_results}

\begin{tabularx}{\textwidth}{@{}l*{6}{Y}@{}}
\toprule

\textbf{Method}
& \textbf{Greedy}
& \textbf{Success@1}
& \textbf{Success@8}
& \textbf{Success@32}
& \shortstack{\textbf{Improved}\\\textbf{(mean@32)}}
& \shortstack{\textbf{Success}\\\textbf{Cores@32}}
\\

\midrule

\textit{MILP (Oracle)}
& 15.01
& -
& -
& -
& 25.63
& 0.15
\\

\midrule

SFT-Hard
& 1.72$_{\pm\text{0.16}}$
& 0.51$_{\pm\text{0.03}}$
& 1.96$_{\pm\text{0.11}}$
& 4.49$_{\pm\text{0.05}}$
& 1.64$_{\pm\text{0.03}}$
& 0.09$_{\pm\text{0.01}}$
\\

SFT-Soft
& 2.65$_{\pm\text{0.35}}$
& 2.40$_{\pm\text{0.05}}$
& 5.27$_{\pm\text{0.41}}$
& 7.52$_{\pm\text{0.49}}$
& 9.92$_{\pm\text{0.18}}$
& 0.16$_{\pm\text{0.01}}$
\\

\midrule

Vanilla GRPO
& 7.97$_{\pm\text{0.29}}$
& 7.43$_{\pm\text{0.03}}$
& 13.79$_{\pm\text{1.09}}$
& 17.93$_{\pm\text{1.54}}$
& 30.21$_{\pm\text{0.59}}$
& 0.50$_{\pm\text{0.01}}$
\\

MaxEnt
& 7.29$_{\pm\text{1.16}}$
& 5.97$_{\pm\text{0.86}}$
& 12.69$_{\pm\text{1.27}}$
& 17.12$_{\pm\text{1.41}}$
& 29.35$_{\pm\text{2.03}}$
& 0.48$_{\pm\text{0.05}}$
\\

GFlowNet
& 3.20$_{\pm\text{0.23}}$
& 3.23$_{\pm\text{0.15}}$
& 8.35$_{\pm\text{0.38}}$
& 13.10$_{\pm\text{0.17}}$
& 14.91$_{\pm\text{0.49}}$
& 0.34$_{\pm\text{0.01}}$
\\

\midrule

\textbf{QuotientPO}
&\textbf{ 8.47$_{\pm\text{0.25}}$}
& \textbf{7.76$_{\pm\text{0.23}}$}
&\textbf{ 15.30$_{\pm\text{0.30}}$}
& \textbf{20.10$_{\pm\text{0.52}}$}
& \textbf{31.59$_{\pm\text{0.82}}$}
& \textbf{0.57$_{\pm\text{0.02}}$}
\\

\bottomrule
\end{tabularx}
\vspace{-0.2cm}
\end{table*}

\subsection{Experimental Setup}

\textbf{Datasets.} We pair strain-level metabolic phenotypes from BacDive~\citep{10.1093/nar/gkae959} with matched NCBI~\citep{ncbi} proteomes, yielding 11,275 organisms spanning 2,454 genera and 40 phyla. Each proteome is reconstructed under 4 CarveMe~\citep{carveme} metabolic universes in the BiGG~\citep{bigg} namespace. Mapping and filtering phenotype annotations onto 16 core BiGG metabolite families produces 3,978 phenotype-consistent candidate GEMs for synthetic pretraining and 28,842 phenotype-inconsistent GEMs for post-training repair and evaluation. We first partition the post-training pool into train, validation, and test splits with an 85:5:10 ratio. Subsequent phenotype-assay filtering and certified-teacher selection leave 2,504/294 training/validation GEMs for SFT, while the RL split retains 9,196/1,087 GEMs and the final sealed test set contains 2,212 GEMs. All splits are species-disjoint, with 5 held-out genera demanding out-of-distribution generalization rather than taxonomic interpolation. Throughout training and evaluation, BacDive phenotypes are never provided as policy inputs, but are used to generate expert repairs via phenotype-constrained MILP for SFT, define RL rewards, and evaluation (Appendix~\ref{app:data}).

\begin{wrapfigure}{r}{0.54\textwidth}
\vspace{-0.4cm}
    \centering
    \includegraphics[width=\linewidth]{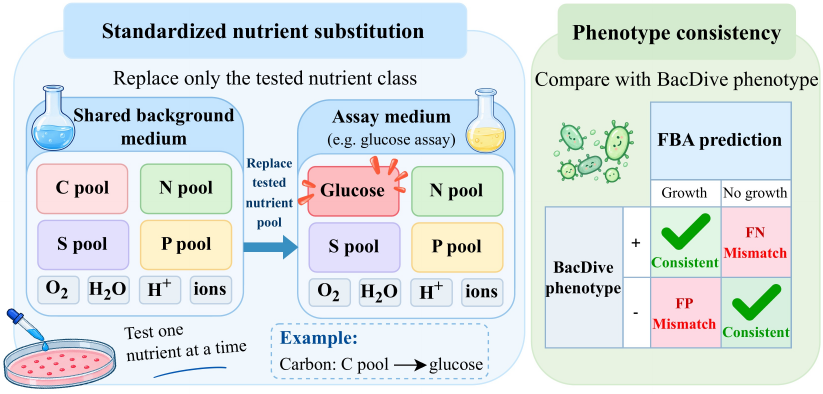}
      \vspace{-0.6cm}
    \caption{\textbf{Verifier design.} }
    \vspace{-0.4cm}
    \label{fig:verfier}
\end{wrapfigure}

\textbf{Verifier.} We define 18 standardized FBA assays across 16 BacDive nutrient families, where each tested nutrient serves as the sole elemental source against an element-depleted background to preclude metabolic bypass (Fig.~\ref{fig:verfier}). Nutrients and oxygen uptake bounds are 10 and 20, respectively, with growth defined by biomass flux $>10^{-3}$. The terminal verifier rewards any repair reproducing the observed phenotypes.

\textbf{Evaluation Metrics.}
We use \textbf{Success@$\boldsymbol{k}$}(\%) as our primary metric, defined as the probability that at least one of $k$ sampled repairs achieves full phenotype recovery ($\Phi_y(M_E)=1$), following multi-sample evaluation in RL and generative modeling~\citep{HongM-RSS-26,chen2021evaluatinglargelanguagemodels,feng2026biomd}.
We report $k \in \{1,8,32\}$ alongside greedy decoding (Greedy), the rate of partial progress $\Phi_y(M_E) > \Phi_y(M_0)$ (Improved), and the macro-averaged number of distinct canonical repair cores recovered among successful repairs within the same 32 samples (Success Cores@32).


\textbf{Baselines.} We compare against 6 baselines across 3 paradigms: (1)~\emph{domain oracle}: MILP with test-time phenotype access; (2)~\emph{standard policy learning}: imitation across probabilistic and discrete-threshold architectures (SFT-Soft/Hard) alongside verifier RL (Vanilla GRPO); and (3)~\emph{diversity-oriented baselines}: action-entropy GRPO (MaxEnt) and GFlowNet~\citep{bengio2021flow} (reward-proportional raw-space sampling without quotient equivalence). Appendix~\ref{app:baselines} details all baselines.


\textbf{Implementation Details.} 
Following prior methods of GEM GNN modeling~\citep{Hasibi2024,huang2024dual}, we train a 92.85M-parameter GNN policy that fuses frozen ESM-2 protein embeddings~\citep{doi:10.1126/science.ade2574} into the global GEM representation, predicting one-shot repair probabilities over a candidate universe of 28,301 reactions. 
For RL training, we initialize from SFT-Soft and train for 3 epochs using AdamW with a fixed learning rate of $2e{-6}$, group size $L=32$, bandwidth $h=0.25$, and the quotient term weight $\beta=0.02$. All FBA simulations use HiGHS~\citep{huangfu2015parallelizingdualrevisedsimplex}. See Appendix~\ref{app:algorithm} and \ref{app:experiment} for more details.

\subsection{Main Results}

\textbf{Quotient exploration improves repair coverage.} Table~\ref{tab:main_results} shows a consistent separation between mechanism-aware exploration and raw-space diversity objectives. QuotientPO achieves the strongest performance across all learned metrics, improving Success@32 from 17.93\% (Vanilla GRPO) to 20.10\% and core discovery from 0.50 to 0.57.
These absolute success rates reflect task difficulty: 
even the phenotype-informed MILP oracle, despite direct test-label access, reaches 15.01\% full repair within its fixed 600\,s search budget.
In contrast, MaxEnt does not improve over vanilla, and GFlowNet underperforms despite explicitly encouraging diversity, confirming that syntactic variety fails to expand mechanistic diversity. The large gap over SFT further shows that supervised imitation alone cannot optimize verifier-defined outcomes.


\textbf{The advantage emerges through test-time scaling.}
Figure~\ref{fig:success_k} shows that quotient exploration improves both repair success and successful-core discovery as the sampling budget increases.
Relative to Vanilla GRPO, the Success@$k$ margin grows from $+0.33$ points at $k=1$ to $+2.17$ at $k=32$, while the gap in Success cores@$k$ expands from $+0.003$ to $+0.068$ cores per GEM.
Thus, the growing deployment advantage reflects genuine coverage of distinct successful repair cores rather than raw scaling alone.
Training dynamics (Fig.~\ref{fig:rl_curves}) mirror this pattern: differences in Success@1 remain modest, whereas QuotientPO separates clearly under Success@8 as training progresses.
Together, these results show that quotient exploration becomes increasingly effective as more samples are allocated to uncover alternative valid repairs.

\begin{figure}[t]
    \centering
    \begin{subfigure}[t]{0.36\linewidth}
        \centering
        \includegraphics[width=\linewidth]{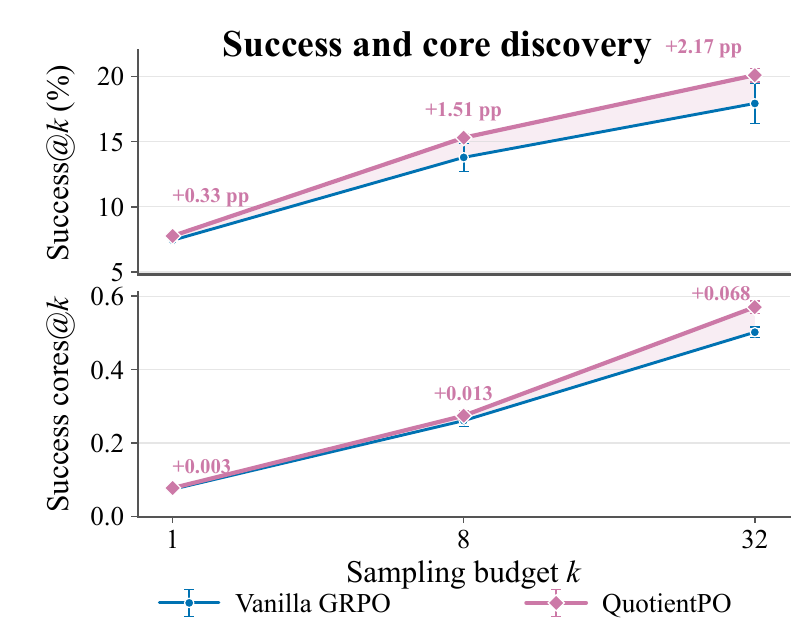}
        \caption{Test-time scaling.}
        \label{fig:success_k}
    \end{subfigure}
     \hfill
      \begin{subfigure}[t]{0.63\linewidth}
        \centering
        \includegraphics[width=\linewidth]{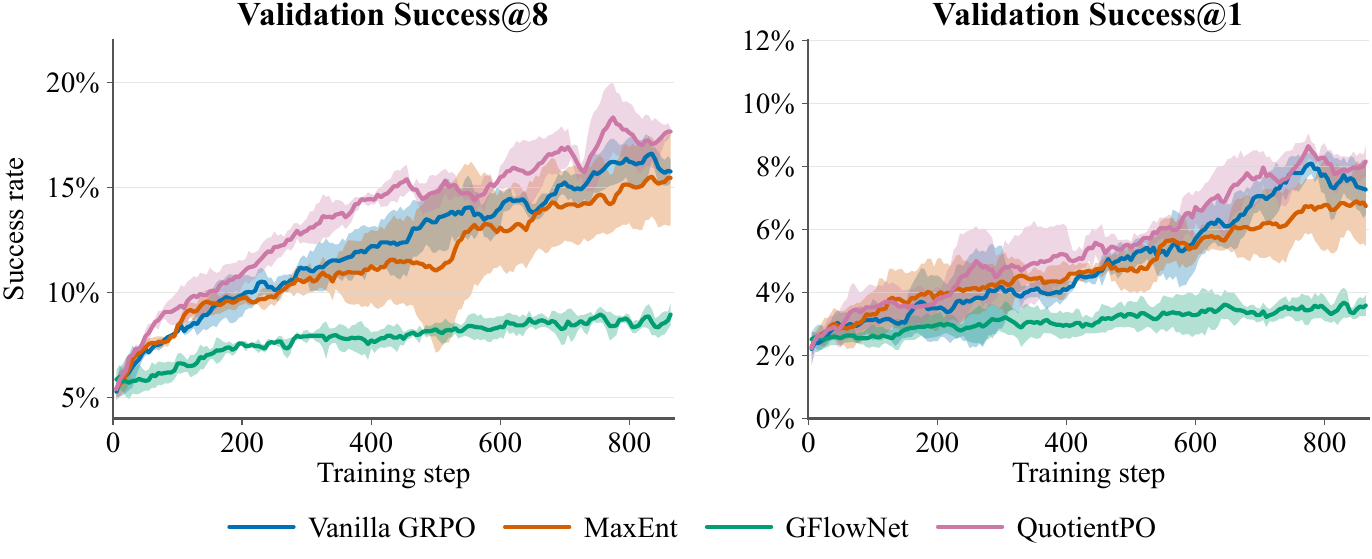}
        \caption{Training dynamics.}
        \label{fig:rl_curves}
    \end{subfigure}

   \caption{
    \textbf{Test-time scaling and training dynamics.}
    \textbf{(a)} Repair success (Success@$k$) and distinct successful-core discovery (Success cores@$k$).
    \textbf{(b)} Validation curves across different RL methods.
}
    \label{fig:rl_analysis}
    \vspace{-0.2cm}
\end{figure}

\subsection{Empirical Analysis of QuotientPO's Mechanism} 

\begin{wrapfigure}{r}{0.42\textwidth}
    \vspace{-0.2cm}
    \centering
    \includegraphics[width=\linewidth]{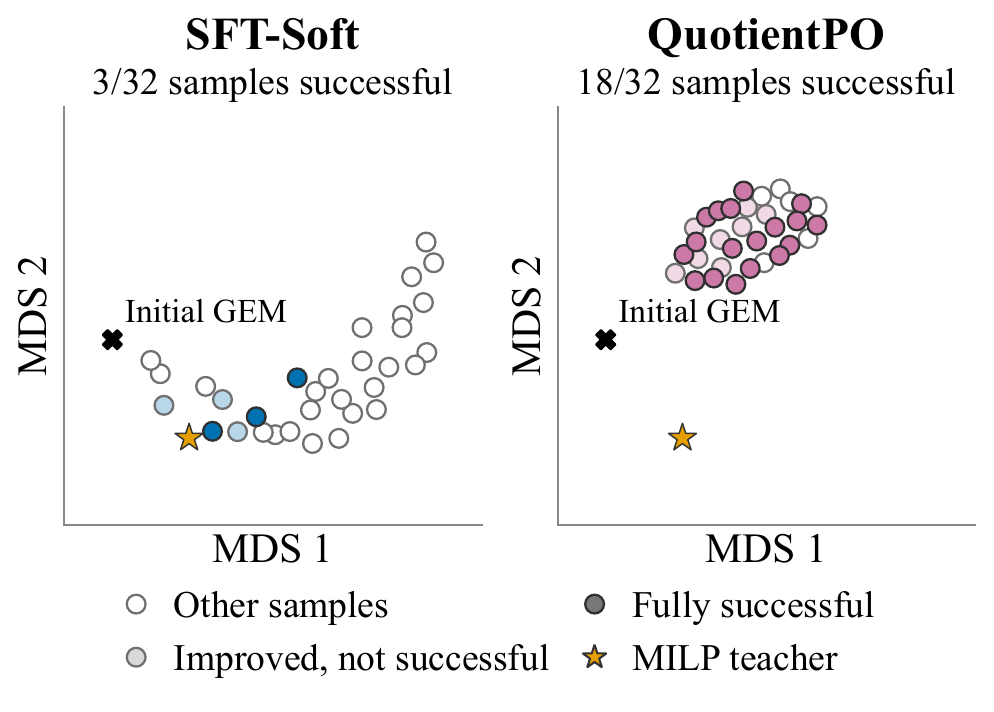}
    \vspace{-0.4cm}
 \caption{\textbf{Landscape of sampled repairs} (\textit{Streptomyces axinellae}).}
    \label{fig:sft-vs-rl}
    \vspace{-0.2cm}
\end{wrapfigure}
\textbf{Verifier feedback moves repair beyond the supervised teacher prior.}
SFT is trained to amortize MILP-generated repairs, but a phenotype-consistent repair is not uniquely identified by its teacher 
edit set. 
Following \citet{tokar2020gsoap}, we project sampled repair sets via metric MDS on pairwise
 signed-Jaccard distances under shared coordinate axes (Fig.~\ref{fig:sft-vs-rl}). While SFT-Soft clusters tightly around the teacher, QuotientPO discovers a well-separated region containing more successful solutions. 
This divergence also holds quantitatively  across the test set: mean distance to the MILP teacher increases from 0.784 for SFT-Soft to 0.980 for QuotientPO (0.780 vs.\ 0.961 among successful repairs).
Thus, the gain from verifier-guided optimization is not explained by more faithful imitation of the teacher, as successful QuotientPO repairs are substantially farther from it.
This supports the partial-identification view of GEM repair, where the teacher provides one admissible correction but verifier feedback reveals alternative successful repairs beyond the supervised mode.


\textbf{Quotient exploration converts diversity into useful mechanism coverage.}
Figure~\ref{fig:why_quotient} exposes a key failure mode: although MaxEnt and GFlowNet
explicitly diversify raw repair sets, over 80\% of their sampled representations collapse after
canonicalization, and neither improves successful-core yield over Vanilla GRPO.
In contrast, QuotientPO applies diversity pressure directly in quotient space, reaching
6.44 distinct cores and 0.57 successful cores per GEM under the same 32-sample budget, compared
with 5.98 and 0.50 for Vanilla GRPO.
To disentangle mechanism coverage from the higher repair success rate itself, we further compare the
two methods on paired GEMs where both produce at least $r$ successful repairs and equalize the
successful-sample budget.
QuotientPO recovers more distinct successful cores at
$r\in\{2,4,8\}$: 1.46 vs.\ 1.43, 2.05 vs.\ 1.96, and 2.78 vs.\ 2.66, respectively
(Appendix~\ref{app:success_controlled}).
Thus, its coverage advantage persists after controlling for the number of successful samples,
rather than arising solely from a higher repair-success frequency.

\begin{figure}[t]
\centering
\begin{subfigure}[t]{0.49\linewidth}
\centering
\includegraphics[width=\linewidth]{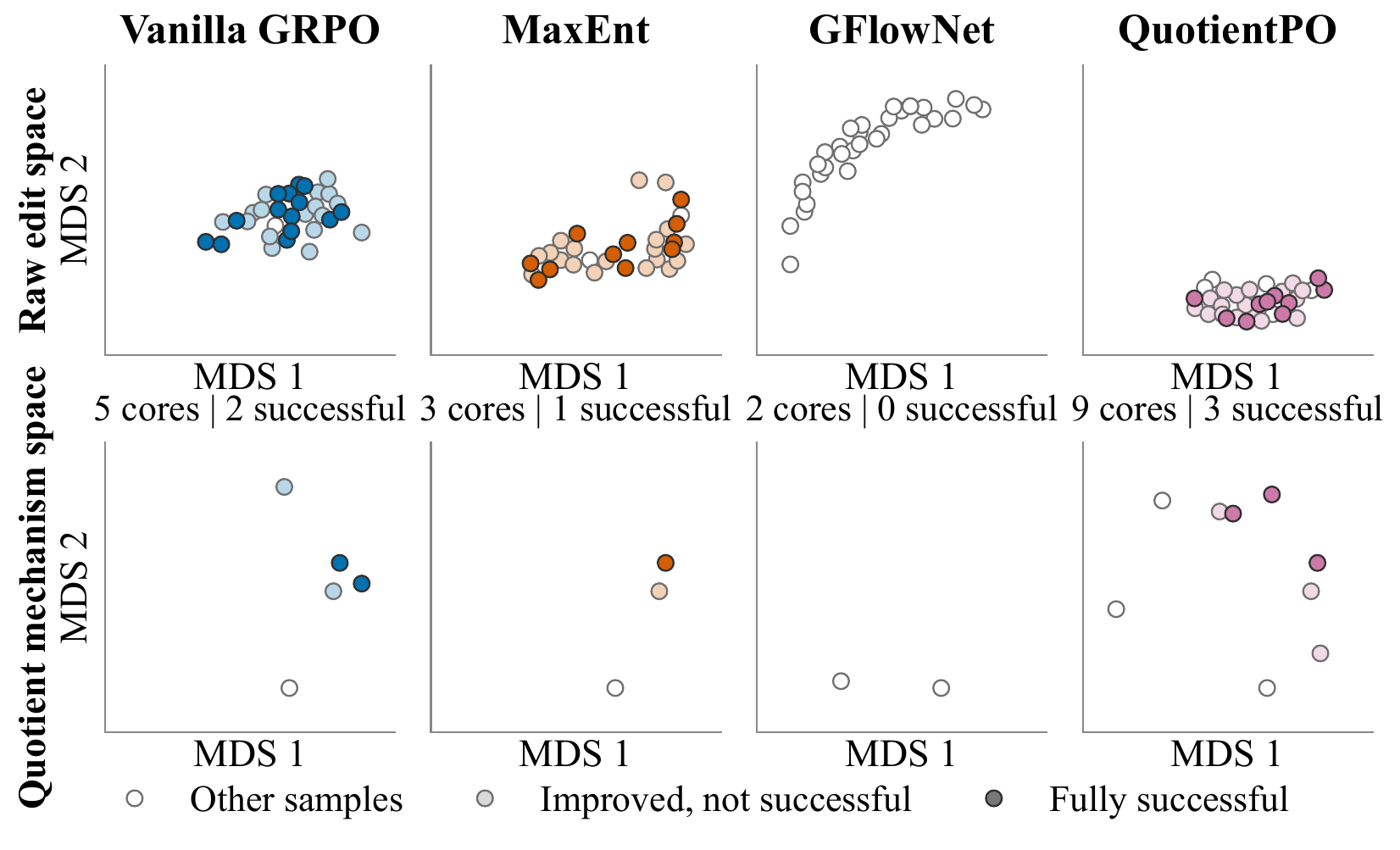}
\caption{Raw diversity does not imply mechanism diversity.}
\label{fig:quotient_space_visualization}
\end{subfigure}
\hfill
\begin{subfigure}[t]{0.49\linewidth}
\centering
\includegraphics[width=\linewidth]{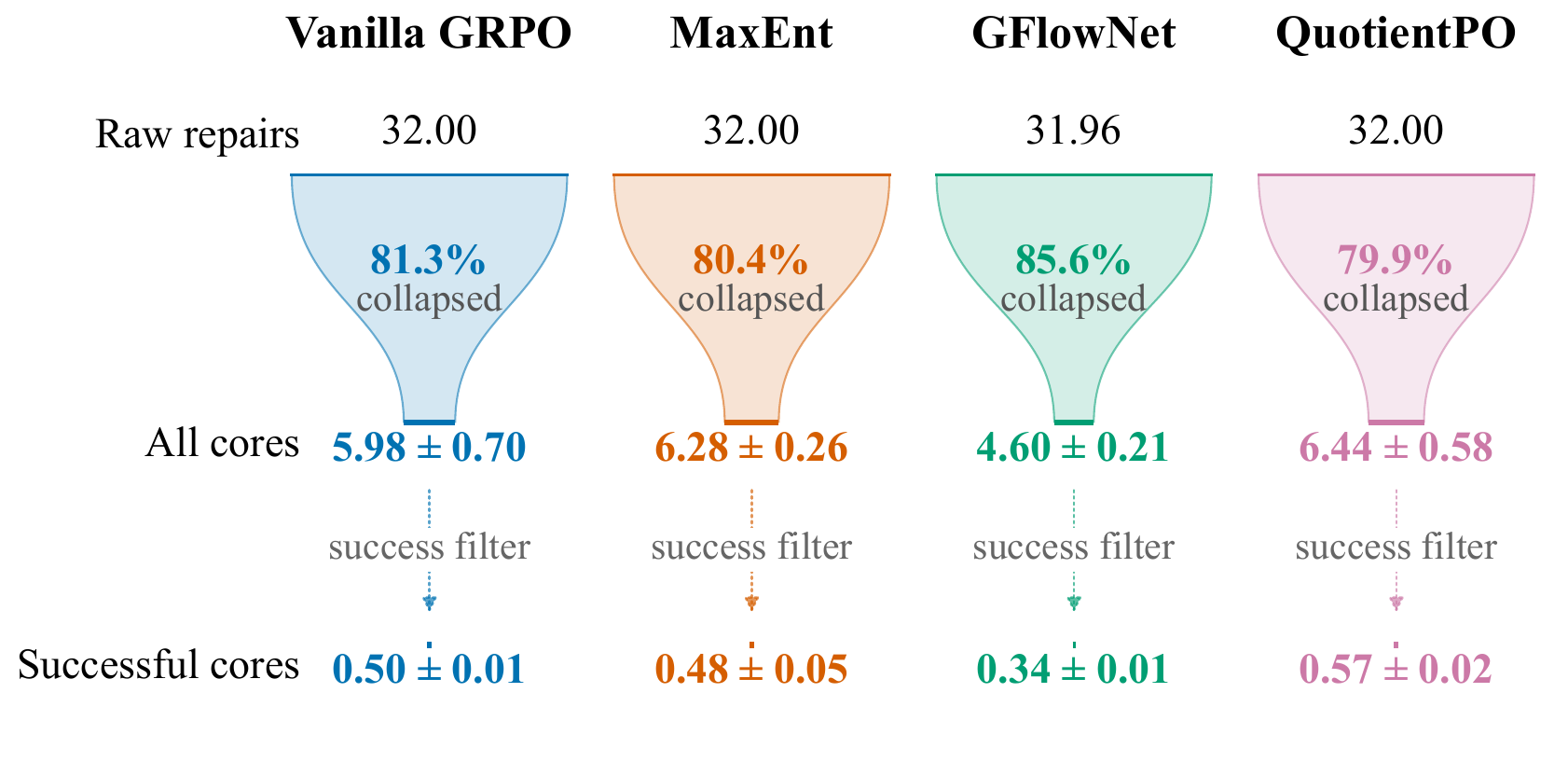}
\caption{Raw-to-core collapse and success filtering.}
\label{fig:raw_to_core}
\end{subfigure}
\caption{\textbf{Exploration across raw and quotient spaces.} \textbf{(a)} MDS projections of sampled repairs before (top) and after (bottom) canonicalization on (\textit{S.~griseoflavus}). \textbf{(b)} Funnel transitions from distinct raw repairs ($k=32$) to verifier-validated cores across the test set.}
\label{fig:why_quotient}
\vspace{-0.4cm}
\end{figure}

\begin{wraptable}{r}{0.6\textwidth}
\centering
\footnotesize
\setlength{\tabcolsep}{0.8pt}
\renewcommand{\arraystretch}{1}
\vspace{-0.4cm}
\caption{
\textbf{Finite-sample credit resolution and successful-core diversity.}
Training diagnostics use all rollouts.
}
\label{tab:kernel_renyi}
\vspace{-0.1cm}
\begin{tabular}{@{}lcccc@{}}
\toprule
&
\multicolumn{2}{c}{\textbf{Training (200→864)}}
&
\multicolumn{2}{c}{\textbf{Successful Cores}}
\\
\cmidrule(lr){2-3}
\cmidrule(lr){4-5}

\textbf{Method}
& \textbf{Exact ties (\%)}
& \textbf{Resolved (\%)}
& \textbf{Collision $\downarrow$}
& \textbf{Mean $d_J$ $\uparrow$}
\\
\midrule

Vanilla
& 8.19 $\rightarrow$ 8.59
& 60.55 $\rightarrow$ 84.86
& 0.5769 & 0.2408 \\

Shannon
& 8.60 $\rightarrow$ 9.72
& 59.23 $\rightarrow$ 80.20
& 0.5687 & 0.2431 \\

Exact R2
& 7.34 $\rightarrow$ 10.23
& 59.71 $\rightarrow$ 78.20
& 0.5285 & 0.2861 \\

Kernel R2
& 8.57 $\rightarrow$ \textbf{7.99}
& 63.81 $\rightarrow$ 72.78
& \textbf{0.5176} & \textbf{0.3008} \\

\bottomrule
\end{tabular}
\vspace{-0.2cm}
\end{wraptable}

\textbf{Kernel relaxation recovers graded diversity signals from coarse finite-sample counts.}
Under Exact R\'enyi-2, the exact-count tie rate rises from 7.34\% to 10.23\% over training (Table~\ref{tab:kernel_renyi}), while 78.2\% of these final ties are resolved by kernel credit. In contrast, Kernel R\'enyi-2 reduces the exact-count tie rate of its rollout groups to 7.99\%, while achieving the lowest successful-core collision (0.518) and largest mean edit Jaccard distance $d_J=1-J$ (0.301). Consistently, it also achieves the best Success@32 and successful-core (Table~\ref{tab:quotient_objective_downstream}). These results show that kernel relaxation provides finer exploration credit precisely where exact multiplicity counts lose resolution.

\begin{wrapfigure}{r}{0.5\textwidth}
    \vspace{-0.4cm}
    \centering
    \includegraphics[width=\linewidth]{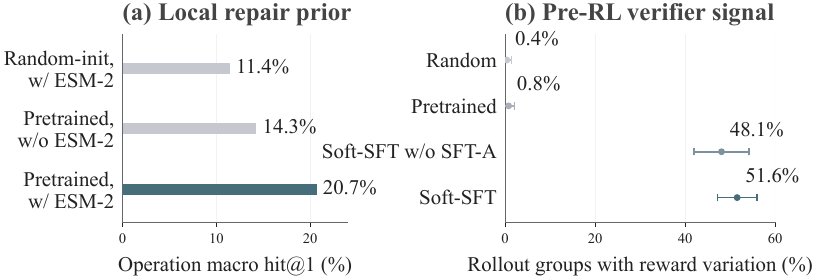}
    \vspace{-0.4cm}
\caption{\textbf{Building an exploration-ready policy.}
\textbf{(a)} Pretraining and ESM-2 enhance the local repair prior.
\textbf{(b)} Set-level SFT unlocks rollout reward variance for RL.
Error bars denote standard deviation across 8 rollout batches.}
    \label{fig:initialization_analysis}
    \vspace{-0.6cm}
\end{wrapfigure}
\textbf{Structured initialization turns sparse verifier feedback into a learnable signal.}
Figure~\ref{fig:initialization_analysis} separates the roles of representation learning and supervised policy formation.
During single-edit imitation (SFT-A), which trains the model on individual MILP teacher operations, operation-level Hit@1 drops from $20.7\%$ to $11.4\%$ without structural pretraining and to $14.3\%$ without
 ESM-2 conditioning,
 showing that both improve local repair prediction.
Yet representation alone provides almost no terminal learning signal: before any RL update, only $0.4\%$ and $0.8\%$ of rollout groups exhibit within-group reward variation under random and pretrained policies.
Set-level supervision changes this regime sharply, increasing the rate to $48.1\%$ without SFT-A and to $51.6\%$ with the complete supervised pipeline.
Thus, while pretraining and genomic conditioning provide a local repair prior, set-level supervision is crucial to unlock group reward variance for RL.

\subsection{Robustness and Ablation Analysis}
\label{sec:ablation}



\textbf{Robustness across RL objectives and policy architectures.}
As shown in Table~\ref{tab:ablations}(a,b), QuotientPO consistently outperforms baselines across RL formulations and architectures. Under RLOO~\citep{ahmadian-etal-2024-back}, it reaches 17.21 Success@32 (S@32, vs.\ 15.75 Vanilla), while under GRPO, it improves both Soft (20.10 vs.\ 17.93) and Hard (6.65 vs.\ 5.92) architectures.
These results show that the benefit of quotient-space exploration is not tied to a specific RL formulation or  architecture.



\textbf{Design ablation and sensitivity.} 
QPI and the redundancy penalty yield smaller mean improvements, while the quotient term accounts for the dominant performance gain: removing it incurs the steepest drop in S@32 (from 20.10 to 17.54; Table~\ref{tab:ablations}c). 
Critically, with the same kernel and training budget, quotient-space exploration improves both Success@32 (20.10 vs. 18.16) and successful-core discovery (0.57 vs. 0.52) over its raw-space counterpart.
Sensitivity analysis (Table~\ref{tab:ablations}d) shows broad robustness to the diversity weight $\beta$, but sensitivity to bandwidth $h$: a larger $h$ over-smooths distant cores and blunts crowding selectivity, while leaving quotient identity unchanged.

\begin{table*}[t]
\centering
\footnotesize

\caption{
\textbf{Robustness and ablation studies.} S@32 denotes Success@32 (\%).
}
\label{tab:ablations}
\vspace{-0.2cm}
\begin{minipage}[t]{0.2\textwidth}
\centering
\textbf{(a) RLOO Base}
\setlength{\tabcolsep}{1.2pt}        
\renewcommand{\arraystretch}{1.15}
\begin{tabular}{@{}lcc@{}}
\toprule
\textbf{Method}
& \textbf{S@32}
\\
\midrule

Vanilla
& 15.75  \\

MaxEnt
& 14.29  \\

GFlowNet
& 11.57 \\

\textbf{QuotientPO}
& \textbf{17.21}  \\

\bottomrule
\end{tabular}
\end{minipage}
\hfill
%
\begin{minipage}[t]{0.25\textwidth}
\centering
\textbf{(b) Policy Architecture}
\setlength{\tabcolsep}{1.2pt}        
\renewcommand{\arraystretch}{1.05}
\begin{tabular}{@{}llcc@{}}
\toprule
\textbf{Arch.}
& \textbf{RL}
& \textbf{S@32} \\
\midrule

\multirow{2}{*}{Soft}
& Vanilla
& 17.93  \\

& QuotientPO
& 20.10  \\

\midrule

\multirow{2}{*}{Hard}
& Vanilla
& 5.92 \\

& QuotientPO
& 6.65 \\

\bottomrule
\end{tabular}
\end{minipage}
\hfill
%
\begin{minipage}[t]{0.23\textwidth}
\centering
\textbf{(c) Design Ablation}
\setlength{\tabcolsep}{1.2pt}        
\renewcommand{\arraystretch}{0.96}
\begin{tabular}{@{}lcc@{}}
\toprule
\textbf{Variant}
& \textbf{S@32}
 \\
\midrule

Default
& 20.10  \\
Raw-space kernel 
& 18.16 \\
w/o Quotient term
& 17.54 \\

w/o QPI term
& 19.62 \\

w/o Red. penalty
& 19.94  \\
\bottomrule
\end{tabular}
\end{minipage}
\hfill
%
\begin{minipage}[t]{0.27\textwidth}
\centering
\textbf{(d) Sensitivity }
\setlength{\tabcolsep}{2pt}        
\renewcommand{\arraystretch}{0.6}

\begin{tabular}{@{}llcc@{}}
\toprule
\textbf{Param.}
& \textbf{Setting}
& \textbf{S@32} \\
\midrule

--
& Default
& 20.10   \\

\midrule

\multirow{2}{*}{$h (0.25)$}
& 0.125
& 19.41  \\

& 0.5
& 18.57  \\

\midrule

\multirow{2}{*}{$\beta (0.02)$}
& 0.01
& 19.05 \\

& 0.04
& 20.00 \\

\bottomrule
\end{tabular}
\end{minipage}
\vspace{-0.4cm}
\end{table*}

\section{Case Study}

Figure~\ref{fig:case_study} grounds the mechanistic failure modes in concrete biological repairs.
On \textit{S.~mashuensis} (Fig.~\ref{fig:case1}), Vanilla GRPO collapses entirely onto boundary entry blockade (\texttt{NO3abc}, \texttt{NO3tex}), whereas QuotientPO discovers an orthogonal intracellular mechanism: jointly knocking out the NADH- and menaquinol-dependent reductases (\texttt{NITR}, \texttt{NTR3B}).
On \textit{S.~palmae} (Fig.~\ref{fig:case2}), syntactically different Vanilla proposals repeatedly retain the same extracellular bypass (\texttt{SUCRe}), so surrounding edits vary without resolving the sucrose phenotype, whereas QuotientPO recovers clean successful cores. Counterfactual replays confirm that toggling \texttt{SUCRe} directly flips repair outcomes.
These cases expose two failures of raw exploration: viable mechanisms can be missed, while surface variation around the same outcome-determining core can be mistaken for useful diversity. QuotientPO resolves both by directing exploration toward functionally distinct mechanisms.

\begin{figure}[h]
\centering
\vspace{-0.4cm}
\begin{subfigure}[t]{0.498\linewidth}
\centering
\includegraphics[width=\linewidth, trim=16 0 0 0,
    clip]{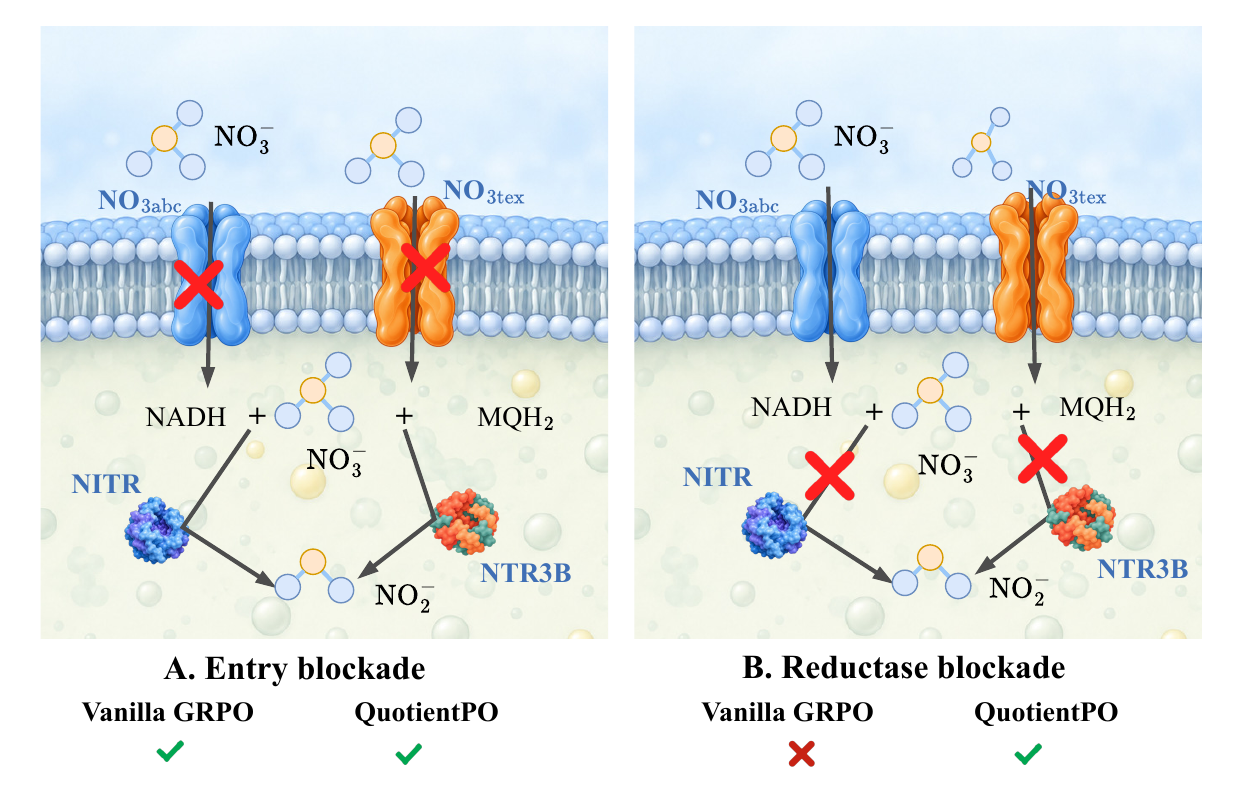}
\vspace{-0.4cm}
\caption{\textit{S. mashuensis.} Nitrate utilization.}
\label{fig:case1}
\end{subfigure}
\hspace{-0.02\linewidth}
\begin{subfigure}[t]{0.498\linewidth}
\centering
\includegraphics[width=\linewidth,trim=8 0 16 0,
    clip]{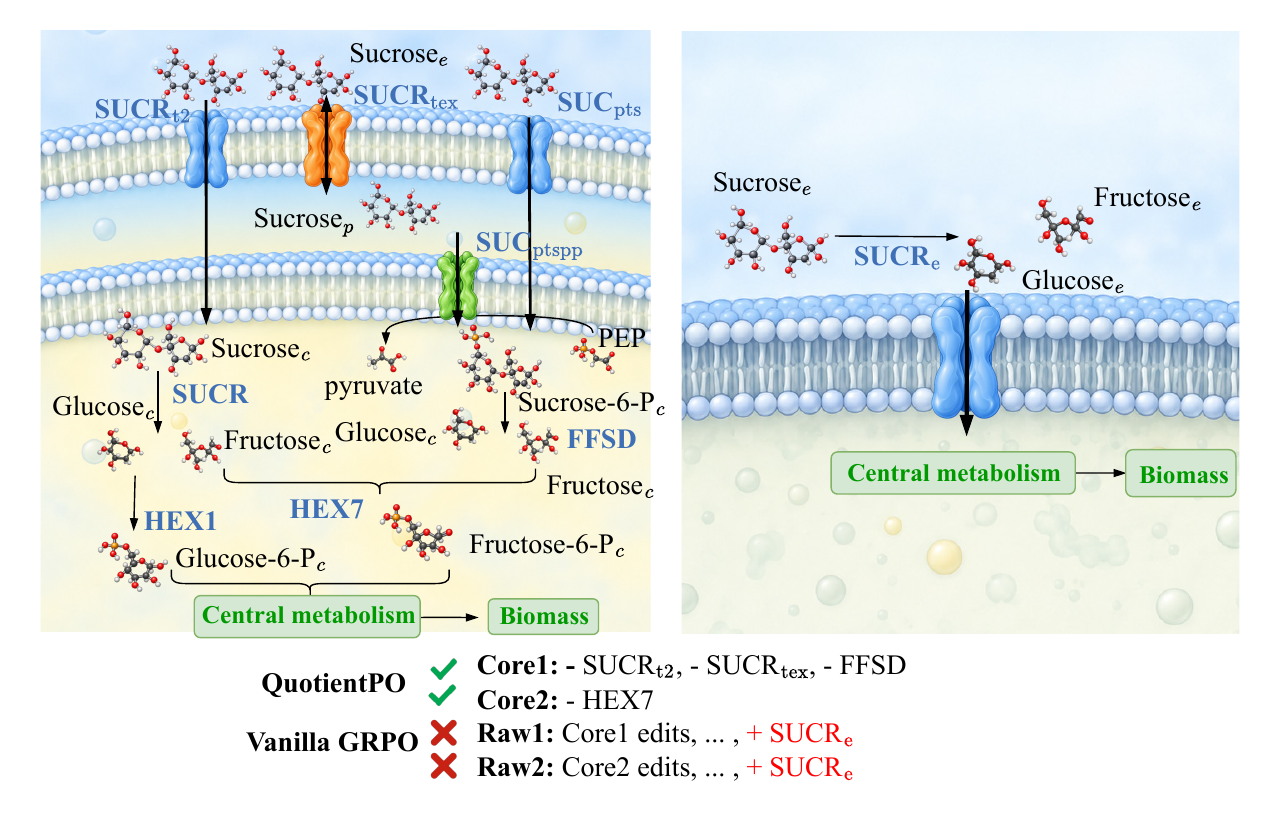}
\vspace{-0.4cm}
\caption{\textit{S. palmae.} Sucrose utilization.}
\label{fig:case2}
\end{subfigure}
\vspace{-0.2cm}
\caption{\textbf{Case studies of mechanistic repair.}}
\label{fig:case_study}
\vspace{-0.2cm}
\end{figure}

\newpage
\section{Conclusion}
Repairing GEMs from indirect functional feedback faces a structural mismatch: combinatorial candidate edits can collapse to the same underlying repair mechanism, causing standard exploration to mistake syntactic multiplicity for scientific diversity.
We address this mismatch by changing the unit of exploration from raw reaction-level actions to invariant repair mechanisms.
By canonicalizing verifier-preserving edits, QuotientPO explores directly in quotient space, with a kernelized Rényi relaxation replacing coarse exact-match statistics with graded finite-sample crowding signals.
This formulation removes representation-multiplicity bias while empirically improving repair success and the coverage of distinct valid mechanisms.
More broadly, our results suggest that when scientific solutions admit equivalent representations, effective exploration should be defined over equivalence classes rather than surface representations (limitations and future directions in Appendix~\ref{app:limitations}).

\section*{AI Use Statement}
We used generative AI tools to assist with writing refinement and code generation. For writing, these tools were used to improve the clarity, grammar, conciseness, and presentation of author-written text. For code, generative AI was used to assist with generating and revising portions of the implementation. All AI-assisted code was manually reviewed, tested, and verified by the authors before use. All AI-assisted writing was likewise reviewed and edited by the authors. The authors take full responsibility for the correctness and final content of this work.

\section*{Ethics statement}
This work uses publicly available microbial genomic and phenotype data from NCBI~\citep{ncbi} and BacDive~\citep{10.1093/nar/gkae959}, together with established computational resources for genome-scale metabolic modeling (CarveMe~\citep{carveme}, BiGG~\citep{bigg}, etc.), in accordance with the terms and policies of their respective sources. All experiments are computational and involve no human participants, personally identifiable information, animal experiments, or experimental manipulation of living organisms. Our method operates only on in silico metabolic models and is intended for scientific modeling and computational research. While improved metabolic models may support downstream applications in metabolic engineering or synthetic biology, such uses should follow appropriate institutional biosafety procedures, regulatory requirements, and responsible-use practices.

\section*{Reproducibility Statement}
We provide detailed materials to support reproduction of our results. The complete QuotientPO algorithm and optimization procedure are described in Appendix~\ref{app:algorithm}, with the core implementation and usage instructions included in the supplemental materials. The benchmark construction and data-processing pipeline, based on publicly available databases, are documented in Appendix~\ref{app:data}; the benchmark construction pipeline and evaluation code are also provided in the supplemental materials. Model architecture and training details for all stages are specified in Appendix~\ref{app:experiment}, and the corresponding code is included in the supplemental materials. Implementation details for all baselines are provided in Appendix~\ref{app:baselines}, with their core implementations likewise included in the supplemental materials.

\bibliography{iclr2027_conference}
\bibliographystyle{iclr2027_conference}

\newpage
\appendix
\section{Algorithm Details of QuotientPO}
\label{app:algorithm}

This section gives the implementation details of QuotientPO.
For compactness, let $x=(M_0,G,y)$ denote a training instance, while
emphasizing that the repair policy $\pi_\theta(E\mid M_0,G)$ conditions
only on the initial GEM and genomic representation. The phenotype labels
$y$ are used exclusively by the verifier.
QuotientPO fine-tunes the supervised one-shot repair policy using
four signals: verifier reward, kernelized quotient crowding (Quotient term), finite-panel
successful-core discovery credit (QPI term), and a redundancy penalty outside the
canonical repair core.

\subsection{One-shot repair policy}
Let $\mathcal A(M_0)$ denote the legal signed reaction edits for an initial
GEM, where $(\mathrm{ADD},r)$ and $(\mathrm{REMOVE},r)$ are treated as
distinct actions.
The policy predicts the complete unordered repair set in one shot.
For each legal edit $a\in\mathcal A(M_0)$, the actor produces a logit
$z_{\theta,a}$ and an operation-specific threshold
$\tau_{\theta,o(a)}$, giving the inclusion probability
\(
p_{\theta,a}
=
\operatorname{sigmoid}
\left(
z_{\theta,a}-\tau_{\theta,o(a)}
\right).
\)
The complete edit-set likelihood is the product of the corresponding
Bernoulli decisions,
\[
\log \pi_\theta(E\mid M_0,G)
=
\sum_{a\in\mathcal A(M_0)}
\left[
\mathbb I[a\in E]\log p_{\theta,a}
+
\mathbb I[a\notin E]\log(1-p_{\theta,a})
\right].
\label{eq:app_set_likelihood}
\]

\begin{algorithm}[h]
\caption{Verifier-preserving canonicalization $\kappa_y$}
\label{alg:canonicalize}
\small
\begin{algorithmic}[1]
\Require Repair set $E$, phenotype targets $y$,
         deterministic verifier $\mathrm{FBA}_y$,
         maximum parallelism $B \geq 1$
\Ensure Canonical repair core $m=\kappa_y(E)$

\State $\boldsymbol{\sigma}^{\star} \gets \mathrm{FBA}_y(E)$
\State $C \gets \operatorname{sort}_{(\mathrm{operation},\mathrm{reaction})}(E)$

\While{$|C|>0$}
    \State $(c_1,\ldots,c_n) \gets C$
    \State $\mathrm{removed} \gets \mathrm{false}$
    \State $i \gets 1$
    \State $w \gets 1$

    \While{$i \leq n$}
        \State $r \gets \min(n,i+w-1)$

        \For{$j=i,\ldots,r$}
            \State $S_j \gets C\setminus\{c_j\}$
            \State $h_j \gets
                \Call{Submit}{\mathrm{FBA}_y(S_j)=\boldsymbol{\sigma}^{\star}}$
            \Comment{Asynchronous verifier check}
        \EndFor

        \For{$j=i,\ldots,r$}
            \State $a_j \gets \Call{Await}{h_j}$
            \Comment{Consume results in the fixed order}

            \If{$a_j=\mathrm{true}$}
                \State $C \gets C\setminus\{c_j\}$
                \State $\mathrm{removed} \gets \mathrm{true}$
                \State \Call{CancelOrDiscard}{$h_{j+1:r}$}
                \State \textbf{break}
            \EndIf
        \EndFor

        \If{$\mathrm{removed}$}
            \State \textbf{break}
            \Comment{Restart after the committed deletion}
        \EndIf

        \State $i \gets r+1$
        \State $w \gets \min(2w,B)$
        \Comment{Expand only after the entire batch fails}
    \EndWhile

    \If{\textbf{not} $\mathrm{removed}$}
        \State \textbf{break}
        \Comment{No verifier-preserving singleton deletion remains}
    \EndIf
\EndWhile

\State \Return $C$
\end{algorithmic}
\end{algorithm}

\subsection{Verifier-Preserving Canonicalization}
\label{app:canonicalization}

For a sampled repair $E$, let $\boldsymbol{\sigma}_y(E) = (\mathbb I[\widehat y_i(M_E)=y_i])_{i=1}^N$ denote the complete phenotype-satisfaction vector across all $N$ conditions. The quotient map $\kappa_y$ (Section~\ref{sec:exploration}) is implemented as a deterministic, verifier-preserving reduction (Algorithm~\ref{alg:canonicalize}) that systematically removes redundant edits while preserving $\boldsymbol{\sigma}_y(E)$. The resulting deletion-irreducible \emph{repair core} $m=\kappa_y(E)\subseteq E$ defines the canonical functional representation of $E$ under the verifier and induces the redundancy cost $|E|-|\kappa_y(E)|$. Because $\kappa_y$ is deterministic, every repair is assigned to exactly one canonical core, yielding a well-defined quotient partition over the repair space. 

\subsection{Kernelized quotient crowding (Quotient term)}

For each organism, we sample a rollout group $E_{1:L}\sim\pi_{\theta_{\mathrm{old}}}(\cdot\mid M_0,G)$ and obtain their canonical cores $m_i=\kappa_y(E_i)$.
All cores participate in quotient exploration, irrespective of whether they achieve complete phenotype recovery. As in Section~\ref{sec:renyi}, signed repair cores are compared via the kernel
\[
k(m_i,m_j)
=
\exp\left(
-\frac{1-J(m_i,m_j)}{h}
\right),
\qquad
J(m_i,m_j)
=
\frac{|m_i\cap m_j|}{|m_i\cup m_j|},
\label{eq:app_kernel}
\]
with $J(\varnothing,\varnothing)=1$, and with ADD and REMOVE edits treated as distinct signed elements.

For $L\ge2$, the empirical off-diagonal kernel concentration is
\[
\widehat I
=
\frac{1}{L(L-1)}
\sum_{a\neq b}
k(m_a,m_b).
\]
The candidate-level crowding cost is directly given by Eq.~\eqref{eq:qrl_kernel_credit}:
\[
c_i
=
2L
\frac{
\sum_{j\neq i}
k(m_i,m_j)
}{
\sum_{a\neq b}
k(m_a,m_b)
}
-2.
\]
For numerical robustness, the implemented score is clipped:
\(
\widetilde c_i
=
\operatorname{clip}(c_i,-8,8).
\label{eq:app_crowding_clip}
\)
The analysis in Section~\ref{sec:qrl} characterizes the unclipped crowding score $c_i$;
implementation clipping is used only for finite-sample numerical stability.
A positive crowding cost suppresses probability assigned to locally redundant cores, while a negative cost favors isolated cores.

\subsection{Finite-panel successful-core discovery credit (QPI term)}
\label{app:qpi}

Kernelized R\'enyi credit encourages geometric diversity over the entire space of repair cores, but finite rollout groups provide sparse signals regarding which distinct cores are functionally useful.
We therefore supplement it with a finite-panel \emph{quotient policy iteration} (QPI) credit that explicitly rewards successful-core discovery.

For the reported setting, each rollout group of $L=32$ samples is deterministically repartitioned $R=4$ times into panels of size $K=8$.
For a panel $P$, let $S_P = \sum_{j\in P} s_y(E_j)$ denote the number of successful samples, and $N_P(m) = \sum_{j\in P} \mathbb I[m_j=m]$ denote the multiplicity of core $m$ within the panel. Each training organism maintains a bounded archive $\mathcal H_x$ of previously discovered successful cores.
Let $\mathcal H_x^{-}$ denote the archive before processing the current rollout group, and let
\(
\nu_i = \mathbb I[s_y(E_i)=1,\, m_i\notin\mathcal H_x^{-}]
\label{eq:app_novel}
\)
indicate that sample $i$ discovers a previously unseen successful core.
The multiplicity of novel samples realizing core $m$ in panel $P$ is $N_P^{\mathrm{new}}(m) = \sum_{j\in P} \nu_j\,\mathbb I[m_j=m]$.

For sample $i$ in panel $P$, the credit decomposes into three terms:
\begin{align*}
a_i(P)
&=
\begin{cases}
s_y(E_i)/S_P, & S_P>0,\\
0, & S_P=0,
\end{cases}
&&\text{successful coverage},
\\
b_i(P)
&=
\frac{1}{K} \mathbb I[N_P(m_i)=1],
&&\text{within-panel core diversity},
\\
n_i(P)
&=
\begin{cases}
\displaystyle \frac{1}{K\,N_P^{\mathrm{new}}(m_i)}, & \nu_i=1,\\[2mm]
0, & \nu_i=0,
\end{cases}
&&\text{archive novelty}.
\end{align*}
The first term shares credit among successful candidates in the panel; the second rewards cores that occur uniquely within the current panel; and the third rewards successful cores not previously archived.  Averaging over the $R$ deterministic repartitions gives the detached QPI credit
\begin{equation*}
u_i
=
\frac{1}{R}
\sum_{r=1}^{R}
\left[
a_i(P_i^{(r)})
+
\alpha_t
\left(
b_i(P_i^{(r)})
+
n_i(P_i^{(r)})
\right)
\right],
\label{eq:app_qpi_credit}
\end{equation*}
where $P_i^{(r)}$ is the panel containing sample $i$ in repartition $r$.
Newly discovered successful cores are subsequently inserted into $\mathcal H_x$ (bounded to eight distinct cores per organism).
The archive is used strictly to evaluate novelty and prioritize under-explored instances; archived repairs are never replayed as actions, and no distillation loss is applied.

\begin{algorithm}[t]
\caption{QuotientPO: quotient-space policy optimization}
\label{alg:quotientpo}
\small
\begin{algorithmic}[1]
\Require Training set $\mathcal D$;
         initialized policy $\pi_{\theta_0}$;
         verifier $\text{FBA}$;
         rollout size $L$;
         panel size $K$
\Ensure Trained policy $\pi_{\theta_T}$

\State $\theta\gets\theta_0$
\State Initialize successful-core archives $\{\mathcal H_x\}_{x\in\mathcal D}$

\For{$t=1,\ldots,T$}
    \State Select a training batch, including archive-guided context replay
    \State $\theta_{\mathrm{old}}\gets\theta$
    \State Set coefficients
           $(\beta_t,\lambda_t,\alpha_t,\eta_t)$

    \For{each training instance $x=(M_0,G,y)$ in the batch}

        \State Sample
        $E_{1:L}\sim
        \pi_{\theta_{\mathrm{old}}}(\cdot\mid M_0,G)$

        \For{$i=1,\ldots,L$}
            \State Verify $M_{E_i}$ to obtain
                   $\text{FBA}_y(E_i)$,
                   $\boldsymbol{\sigma}_y(E_i)$,
                   and $\Phi_y(M_{E_i})$
            \State Compute terminal reward $R_y(E_i)$
            \State $m_i\gets
                   \Call{Canonicalize}{E_i,y,\text{FBA}_y}$
                   \Comment{Algorithm~\ref{alg:canonicalize}}
            \State Compute redundancy $(|E|-|\kappa_y(E)|)$
        \EndFor

        \State Compute group-relative advantages $A_{1:L}$

        \State Compute kernel crowding costs
               $\widetilde c_{1:L}$
               from the valid cores

        \State Compute finite-panel QPI credits
               $u_{1:L}$ against the pre-update archive
               $\mathcal H_x^{-}$

        \State Insert newly discovered successful cores into $\mathcal H_x$
    \EndFor

    \State Rescore all sampled edit sets under $\pi_\theta$
           and compute policy ratios $\rho_i(\theta)$

    \State Form $\mathcal J_t(\theta)$

    \State Take one gradient update on $\theta$
\EndFor

\State \textbf{return} $\pi_{\theta_T}$
\end{algorithmic}
\end{algorithm}

\subsection{Complete QuotientPO objective}

Within each rollout group, the verifier reward is normalized using the GRPO
baseline,
\(
A_i
=
\frac{
R_y(E_i)-\overline R
}{
\sigma_R+\varepsilon
},
\)
with $A_i=0$ when the group reward variance vanishes.
For a sampled repair, the policy ratio is
\(
\rho_i(\theta)
=
\frac{
\pi_\theta(E_i\mid M_0,G)
}{
\pi_{\theta_{\mathrm{old}}}(E_i\mid M_0,G)
}
\). The implemented maximization objective is
\begin{align*}
\mathcal J_t(\theta)
={}&
\frac{1}{L}
\sum_{i=1}^{L}
\min\Bigl(
\rho_i(\theta)A_i,\,
\operatorname{clip}
(\rho_i(\theta),1-\epsilon,1+\epsilon)A_i
\Bigr)
\nonumber\\
&-
\frac{\beta_t}{L}
\sum_{i=1}^{L}
\rho_i(\theta)\,
\operatorname{sg}(\widetilde c_i)
\nonumber\\
&+
\frac{\lambda_t K}{L}
\sum_{i=1}^{L}
\rho_i(\theta)\,
\operatorname{sg}(u_i)
-
\frac{\eta_t}{L}
\sum_{i=1}^{L}
\rho_i(\theta)\,
\operatorname{sg}(|E|-|\kappa_y(E)|),
\label{eq:app_full_objective}
\end{align*}

where $\operatorname{sg}(\cdot)$ denotes stop-gradient.
The first term is verifier-guided GRPO.
The second implements the kernelized quotient-diversity objective by shifting
probability away from crowded repair cores.
The third is the finite-panel QPI term that promotes successful-core
discovery and retention.
The final term suppresses surface edits removed by canonicalization.
All verifier outcomes, advantages, core identities, crowding costs, QPI
credits, and redundancy costs are detached from the actor graph.
Gradients therefore flow only through the complete edit-set likelihood
$\pi_\theta(E\mid M_0,G)$.
PPO clipping is applied to the GRPO task term, whereas the quotient and
auxiliary terms use the unclipped likelihood ratio. The full algorithm process is Algorithm~\ref{alg:quotientpo}.

In all experiments, the hyperparameters balance verifier reward optimization against mechanism-level diversity: we use a rollout group size of $L=32$ with PPO clipping threshold $\epsilon=0.2$. For exploration, we set the kernel crowding weight to $\beta=0.02$ to penalize mechanism collapse, the finite-panel QPI discovery weight to $\lambda=0.10$ with panel size $K=8$ and internal novelty scale $\alpha=0.05$, and the canonical redundancy penalty to $\eta=0.001$ to gently suppress extraneous surface edits without impeding phenotype recovery.

\section{Theoretical Details}
\label{app:theory}

\subsection{Formulation of GEM Modeling}
\label{app:formulation}
A genome-scale metabolic model (GEM) is a directed, weighted bipartite graph
\(M=(\mathcal C,\mathcal R,S)\), where metabolites \(\mathcal C\) and reactions
\(\mathcal R\) form the two node types and
\(S\in\mathbb R^{|\mathcal C|\times|\mathcal R|}\) encodes their stoichiometric
edges. For metabolite \(c\) and reaction \(r\), \(S_{cr}<0\) denotes consumption
(\(c\!\to\!r\)), \(S_{cr}>0\) denotes production (\(r\!\to\!c\)), and
\(|S_{cr}|\) gives the stoichiometric coefficient. Reaction addition or deletion
therefore corresponds to adding or removing a reaction node and its associated
column of \(S\). Flux balance analysis (FBA) assigns reaction fluxes \(v\) and
computes
\[
v^*(M)=\max_v v_{\mathrm{bio}}
\qquad
\text{s.t.}
\qquad
Sv=0,\quad \ell\le v\le u,
\]
where \(Sv=0\) enforces steady-state mass balance and \([\ell,u]\) specifies
reaction directionality and the reference environment. Wet-lab phenotypes
provide organism-specific functional constraints through \(\widehat y_z(M)\),
and \(v^*(M)>\delta\) enforces a non-negligible biomass-supporting feasible
flux.

\subsection{Theory of Representation-Invariant Exploration}
\label{app:exploration}
\textbf{Notation and setup.}
Given an initial draft metabolic network $M_0$, genomic context $G$, and experimental phenotype target $y$, the repair policy generates a raw repair candidate $E \in \mathcal{F}$ (a set of reaction additions or deletions). 
The deterministic verifier-preserving reduction
$\kappa_y:\mathcal F\to\mathcal F$
maps each raw repair $E$ to an operational repair core
$m=\kappa_y(E)$.
We denote by $\mathcal{M}_y = \kappa_y(\mathcal{F}) \subseteq \mathcal{F}$ the quotient mechanism space of all such canonical repairs. Throughout this subsection, we fix an arbitrary training instance $(M_0, G, y)$ and suppress this conditioning from the raw policy $\pi_\theta$ and the induced quotient policy $q_{\theta,y}$ to streamline notation, writing $\pi(E)$ and $q(m)$ respectively. For each mechanism $m \in \mathcal{M}_y$, we define:
\begin{itemize}[leftmargin=*,noitemsep,topsep=2pt]
    \item the \emph{repair fiber} $\Omega_m = \{E \in \mathcal{F} : \kappa_y(E) = m\}$,  
which collects all syntactically distinct reaction edits mapped to the same
canonical repair mechanism $m$.
    \item the \emph{fiber multiplicity} $n_m = |\Omega_m|$, quantifying the syntactic degeneracy (the number of redundant representations) of mechanism $m$.
    \item the \emph{quotient policy marginal} $q(m) = \sum_{E \in \Omega_m} \pi(E)$, representing the aggregate probability assigned by the policy to deploying mechanism $m$.
\end{itemize}
Because $\kappa_y$ preserves the complete phenotype response, any raw edit $E \in \Omega_m$ achieves an identical post-repair phenotype evaluation $\Phi_y(M_E) = \Phi_y(M_m)$. Consequently, the mechanism-level class utility
\(
U_y(m) = \Phi_y(M_E), E \in \Omega_m,
\)
is well defined and invariant to the choice of raw representative within $\Omega_m$. Since the candidate edit space $\mathcal{F}$ is formed by reaction additions and deletions over a finite reaction database, it is finite. Consequently, the quotient mechanism space $\mathcal{M}_y$ and every repair fiber $\Omega_m$ are finite. We denote by $\Delta(\mathcal{F})$ and $\Delta(\mathcal{M}_y)$ the probability simplices over candidate edits and repair mechanisms, respectively, and assume an exploration coefficient $\beta > 0$ throughout.

\begin{lemma}[Entropy decomposition over repair fibers]
\label{lem:entropy_decomposition}
For any probability distribution $\pi \in \Delta(\mathcal{F})$, define the conditional distribution $\pi(E\mid m) = \pi(E) / q(m)$ for $E \in \Omega_m$ whenever $q(m) > 0$. Then
\begin{equation}
H_{\mathrm{act}}(\pi) = H(q) + \sum_{m \in \mathcal{M}_y} q(m) H \! \left(\pi(\cdot \mid m)\right),
\label{eq:app_entropy_decomposition}
\end{equation}
where zero-mass fibers contribute zero. Moreover,
\begin{equation}
H\!\left(\pi(\cdot\mid m)\right)
\leq
\log n_m,
\label{eq:app_conditional_entropy_bound}
\end{equation}
with equality, for $q(m)>0$, if and only if
$\pi(\cdot\mid m)$ is uniform on $\Omega_m$.
\end{lemma}

\begin{proof}
Because
$\mathcal F=\bigsqcup_{m\in\mathcal M_y}\Omega_m$,
for every $E\in\Omega_m$ with $q(m)>0$ we may write
$\pi(E)=q(m)\pi(E\mid m)$.
Therefore,
\begin{align*}
H_{\mathrm{act}}(\pi)
&=
-\sum_{m\in\mathcal M_y}
\sum_{E\in\Omega_m}
q(m)\pi(E\mid m)
\log\!\bigl(q(m)\pi(E\mid m)\bigr)
\nonumber\\
&=
-\sum_m q(m)\log q(m)
   \sum_{E\in\Omega_m}\pi(E\mid m)
\nonumber
-\sum_m q(m)
\sum_{E\in\Omega_m}
\pi(E\mid m)\log\pi(E\mid m)
\nonumber\\
&=
H(q)
+
\sum_m q(m)
H\!\left(\pi(\cdot\mid m)\right),
\end{align*}
which proves Eq.~(\ref{eq:app_entropy_decomposition}).

For the second claim, let $u_m$ denote the uniform distribution on
$\Omega_m$, so $u_m(E)=1/n_m$.
Non-negativity of the Kullback--Leibler divergence gives
\begin{align}
0
\leq
D_{\mathrm{KL}}
\!\left(
\pi(\cdot\mid m)\,\|\,u_m
\right)
&=
\sum_{E\in\Omega_m}
\pi(E\mid m)
\log
\frac{\pi(E\mid m)}{1/n_m}
\nonumber
=
-H\!\left(\pi(\cdot\mid m)\right)
+\log n_m.
\end{align}
Hence
$H(\pi(\cdot\mid m))\leq\log n_m$.
Equality in the KL inequality holds if and only if
$\pi(\cdot\mid m)=u_m$, proving the equality condition.
\end{proof}

\begin{proposition}[Action-space multiplicity bias]
\label{prop:action_multiplicity_bias}
For a fixed training instance, let $\{\Omega_m\}_{m\in\mathcal M_y}$
partition the feasible repair space into verifier-equivalent classes, with
common utility $U_y(m)=\Phi_y(M_E)$ for $E\in\Omega_m$ and multiplicity
$n_m=|\Omega_m|$.
Optimizing $\mathcal J_{\mathrm{act}}$ over the unrestricted probability
simplex induces
\(
q_{\mathrm{act}}^\star(m)
\propto
n_m\exp\!\left(U_y(m)/\beta\right).
\)
Equivalently, action entropy introduces an implicit class-level utility
$\beta\log n_m$ in addition to biological utility $U_y(m)$.
\end{proposition}
\begin{proof}
For a fixed raw repair distribution $\pi$, phenotype utility depends only on
its quotient marginal:
\begin{align}
\mathbb E_{E\sim\pi}\!\left[\Phi_y(M_E)\right]
=
\sum_{m\in\mathcal M_y}
\sum_{E\in\Omega_m}
\pi(E)U_y(m) =
\sum_{m\in\mathcal M_y}
q(m)U_y(m).
\label{eq:app_expected_utility}
\end{align}
Combining Eq.~\eqref{eq:app_expected_utility} with
Lemma~\ref{lem:entropy_decomposition}, the fixed-instance action-entropy
objective can be written as
\begin{align}
\mathcal J_{\mathrm{act}}(\pi)
&=
\sum_m q(m)U_y(m)
+
\beta H(q)
+
\beta
\sum_m
q(m)H\!\left(\pi(\cdot\mid m)\right).
\label{eq:app_action_objective_decomp}
\end{align}

Now fix an arbitrary quotient marginal $q$.
The first two terms in
Eq.~\eqref{eq:app_action_objective_decomp}
are independent of the conditional distributions inside the fibers.
By Eq.~\eqref{eq:app_conditional_entropy_bound}, the final term is maximized
independently in every positive-mass fiber by
\[
\pi^\star(E\mid m)
=
\frac{1}{n_m},
\qquad E\in\Omega_m.
\]
Therefore optimization over the raw simplex $\Delta(\mathcal F)$ is
equivalent to
\begin{equation} 
\max_{\pi \in \Delta(\mathcal{F})} \mathcal{J}_{\mathrm{act}}(\pi) = \max_{q \in \Delta(\mathcal{M}_y)} \left\{ \sum_m q(m)U_y(m) + \beta H(q) + \beta \sum_m q(m)\log n_m \right\}.
\label{eq:app_reduced_action_objective}
\end{equation}

Define
\[
a_m
=
n_m\exp\!\left(\frac{U_y(m)}{\beta}\right),
\qquad
Z_{\mathrm{act}}
=
\sum_{m\in\mathcal M_y}a_m,
\]
and
\[
q_{\mathrm{act}}^\star(m)
=
\frac{a_m}{Z_{\mathrm{act}}}.
\]
Since every fiber is nonempty, $n_m\geq1$, and all utilities are finite,
$q_{\mathrm{act}}^\star(m)>0$ for every $m$.
Using
$\log a_m=\log n_m+U_y(m)/\beta$,
the objective in Eq.~\eqref{eq:app_reduced_action_objective} becomes
\begin{align}
&\sum_m q(m)U_y(m)
+
\beta\sum_mq(m)\log n_m
-
\beta\sum_mq(m)\log q(m)
\nonumber\\
&\qquad=
\beta\sum_mq(m)
\log\frac{a_m}{q(m)}
\nonumber\\
&\qquad=
\beta\log Z_{\mathrm{act}}
-
\beta
D_{\mathrm{KL}}
\!\left(
q\,\|\,q_{\mathrm{act}}^\star
\right).
\label{eq:app_action_kl}
\end{align}
Because KL divergence is nonnegative and vanishes only when its two
arguments coincide, Eq.~\eqref{eq:app_action_kl} proves that the unique
quotient marginal maximizing the objective is
\begin{equation}
q_{\mathrm{act}}^\star(m)
=
\frac{
n_m\exp(U_y(m)/\beta)
}{
\sum_{m'\in\mathcal M_y}
n_{m'}\exp(U_y(m')/\beta)
}.
\end{equation}

Furthermore, because this optimum assigns positive mass to every fiber,
Lemma~\ref{lem:entropy_decomposition} requires the unique optimal
within-fiber conditional to be uniform.
Thus the corresponding raw-policy optimum satisfies
\[
\pi^\star(E)
=
\frac{q_{\mathrm{act}}^\star(m)}{n_m}
=
\frac{
\exp(U_y(m)/\beta)
}{
Z_{\mathrm{act}}
},
\qquad
E\in\Omega_m.
\]
Hence the multiplicity factor arises only when raw action probabilities are
aggregated at the class level: each representative receives the same
Boltzmann weight for a fixed utility, and a class containing $n_m$
representatives accumulates $n_m$ such contributions.
Equivalently, action entropy contributes the implicit class utility
$\beta\log n_m$ claimed in the proposition.
\end{proof}

\begin{lemma}[Surjectivity of the quotient marginalization map]
\label{lem:surjectivity}
Let
$T:\Delta(\mathcal F)\to\Delta(\mathcal M_y)$
be defined by
\[
[T\pi](m)=\sum_{E\in\Omega_m}\pi(E),
\]
where
$\Omega_m=\{E\in\mathcal F:\kappa_y(E)=m\}$.
Then $T$ is surjective.
\end{lemma}

\begin{proof}
Because
$\mathcal M_y=\kappa_y(\mathcal F)$,
every $m\in\mathcal M_y$ has at least one representative
$r(m)\in\mathcal F$ satisfying
$\kappa_y(r(m))=m$.
For any $q\in\Delta(\mathcal M_y)$, define
\[
\pi_q(E)=
\begin{cases}
q(m), & E=r(m)\ \text{for some }m\in\mathcal M_y,\\
0, & \text{otherwise}.
\end{cases}
\]
Since distinct fibers are disjoint, the selected representatives are distinct and
$\pi_q\in\Delta(\mathcal F)$.
Moreover,
\[
[T\pi_q](m)
=
\sum_{E\in\Omega_m}\pi_q(E)
=
\pi_q(r(m))
=
q(m).
\]
Hence $T\pi_q=q$, proving that $T$ is surjective.
\end{proof}

\begin{remark}[Biological and probabilistic intuition]
In probabilistic terms, the operator $T$ simply formalizes marginalization
over functionally equivalent repairs: it aggregates the probabilities of all
distinct reaction edit sets mapped to the same canonical repair mechanism,
yielding the total probability $q(m)$ of deploying mechanism $m$.
Biologically, surjectivity ensures that the quotient mechanism space $\Delta(\mathcal{M}_y)$ has no ``ghost states": any target distribution over metabolic repair mechanisms is physically realizable by an underlying reaction-level policy.
\end{remark}

\begin{proof}[Proof of Theorem~\ref{thm:quotient_factorization}]
Let $T : \Delta(\mathcal{F}) \to \Delta(\mathcal{M}_y)$ be the quotient marginalization operator from Lemma~\ref{lem:surjectivity}.

\textbf{\emph{Sufficiency.}}
Suppose there exists
$\overline{\mathfrak D}:\Delta(\mathcal M_y)\to\mathbb R$
such that
\(
\mathfrak D(\pi)
=
\overline{\mathfrak D}(T\pi)
\text{, for all }\pi\in\Delta(\mathcal F).
\)
If two raw policies $\pi,\pi'$ induce the same quotient marginal,
$T\pi=T\pi'$, then
\[
\mathfrak D(\pi)
=
\overline{\mathfrak D}(T\pi)
=
\overline{\mathfrak D}(T\pi')
=
\mathfrak D(\pi').
\]
Hence $\mathfrak D$ is invariant to all probability redistributions that
leave the quotient law unchanged.

\textbf{\emph{Necessity.}}
Conversely, assume that
\[
T\pi=T\pi'
\quad\Longrightarrow\quad
\mathfrak D(\pi)=\mathfrak D(\pi')
\]
for all $\pi,\pi'\in\Delta(\mathcal F)$.
By Lemma~\ref{lem:surjectivity}, for every
$q\in\Delta(\mathcal M_y)$ there exists at least one lift
$\pi_q\in\Delta(\mathcal F)$ satisfying $T\pi_q=q$.
Define
\begin{equation}
\overline{\mathfrak D}(q)
=
\mathfrak D(\pi_q).
\label{eq:app_factor_definition}
\end{equation}
We first verify that this definition is independent of the chosen lift.
If $\pi_q$ and $\widetilde\pi_q$ both satisfy
$T\pi_q=T\widetilde\pi_q=q$, then the assumed invariance gives
\(
\mathfrak D(\pi_q)
=
\mathfrak D(\widetilde\pi_q).
\)
Thus Eq.~\eqref{eq:app_factor_definition} defines a well-defined functional
on the entire quotient simplex.

Now take an arbitrary
$\pi\in\Delta(\mathcal F)$ and set $q=T\pi$.
Both $\pi$ and the selected lift $\pi_q$ induce the same quotient law.
Therefore,
\[
\mathfrak D(\pi)
=
\mathfrak D(\pi_q)
=
\overline{\mathfrak D}(q)
=
\overline{\mathfrak D}(T\pi),
\]
which proves the required factorization.

It remains to prove uniqueness.
Suppose two functionals
$\overline{\mathfrak D}_1$ and
$\overline{\mathfrak D}_2$ satisfy
\(
\mathfrak D
=
\overline{\mathfrak D}_1\circ T
=
\overline{\mathfrak D}_2\circ T.
\)
For any $q\in\Delta(\mathcal M_y)$,
surjectivity of $T$ provides a lift $\pi_q$ with $T\pi_q=q$.
Hence
\[
\overline{\mathfrak D}_1(q)
=
\overline{\mathfrak D}_1(T\pi_q)
=
\mathfrak D(\pi_q)
=
\overline{\mathfrak D}_2(T\pi_q)
=
\overline{\mathfrak D}_2(q).
\]
Therefore
$\overline{\mathfrak D}_1=\overline{\mathfrak D}_2$
on $\Delta(\mathcal M_y)$.
This establishes both existence and uniqueness.
\end{proof}

\begin{remark}[Preservation under a representative lift]
Beyond existence and uniqueness, the quotient factorization preserves basic analytical properties of the exploration functional on the probability simplex.
Fix, for each $m \in \mathcal{M}_y$, an arbitrary representative
$r(m) \in \Omega_m$, so that $\kappa_y(r(m)) = m$, as in Lemma~B.3.
This representative selection induces a lift
$L_r : \Delta(\mathcal{M}_y) \to \Delta(\mathcal{F})$ defined by
\[
[L_r q](E)
=
\begin{cases}
q(m), & E = r(m) \text{ for some } m \in \mathcal{M}_y,\\
0, & \text{otherwise}.
\end{cases}
\]
By construction, $T L_r q = q$ for every $q \in \Delta(\mathcal{M}_y)$.
Moreover, $L_r$ is linear:
\[
L_r\!\left(\alpha q_1 + (1-\alpha)q_2\right)
=
\alpha L_r q_1 + (1-\alpha)L_r q_2,
\qquad \alpha \in [0,1].
\]
Therefore, for the quotient functional $\bar{\mathcal{D}}$ characterized by
Theorem~\ref{thm:quotient_factorization},
\[
\bar{\mathcal{D}}(q)
=
\mathcal{D}(L_r q),
\]
and continuity and convexity or concavity of $\mathcal{D}$ restrict directly
to $\bar{\mathcal{D}}$ through this linear lift.
\end{remark}


\begin{corollary}[Quotient Shannon entropy removes the multiplicity factor]
\label{cor:quotient_shannon_optimum}
For a fixed training instance, unrestricted maximization of
\(
\sum_{m\in\mathcal M_y}q(m)U_y(m)
+
\beta H(q)
\)
over $q\in\Delta(\mathcal M_y)$ has the unique solution
\begin{equation}
q_{\mathrm{quot}}^\star(m)
=
\frac{
\exp(U_y(m)/\beta)
}{
\sum_{m'\in\mathcal M_y}
\exp(U_y(m')/\beta)
}.
\label{eq:app_quotient_optimum}
\end{equation}
\end{corollary}

\begin{proof}
Define
\[
Z_{\mathrm{quot}}
=
\sum_{m\in\mathcal M_y}
\exp\!\left(\frac{U_y(m)}{\beta}\right),
\qquad
q_{\mathrm{quot}}^\star(m)
=
\frac{
\exp(U_y(m)/\beta)
}{
Z_{\mathrm{quot}}
}.
\]
For arbitrary $q\in\Delta(\mathcal M_y)$,
\begin{align}
\sum_mq(m)U_y(m)+\beta H(q)
&=
\beta
\sum_mq(m)
\log
\frac{\exp(U_y(m)/\beta)}{q(m)}
\nonumber\\
&=
\beta\log Z_{\mathrm{quot}}
-
\beta
D_{\mathrm{KL}}
\!\left(
q\,\|\,q_{\mathrm{quot}}^\star
\right).
\label{eq:app_quotient_kl}
\end{align}
The KL divergence is nonnegative and equals zero if and only if
$q=q_{\mathrm{quot}}^\star$.
Therefore Eq.~\eqref{eq:app_quotient_optimum} is the unique maximizer.
\end{proof}

Comparing this expression with the action-space optimum,
\(
q_{\mathrm{act}}^\star(m)
\propto
n_m\exp(U_y(m)/\beta),
\)
shows that quotient Shannon entropy removes the representation-volume factor
$n_m$ exactly.

\textbf{Scope of the results.}
Lemma~\ref{lem:entropy_decomposition} applies to any distribution on the finite feasible repair space under the partition induced by $\kappa_y$. The Gibbs solutions characterize unrestricted simplex optima and serve as variational references for the parameterized policy.
The factorization theorem holds for any fixed verifier-preserving $\kappa_y$, establishing rigorous representation invariance directly over the operational mechanism classes induced by the verifier across diverse biological hypotheses.

\subsection{Collision Sparsity and Kernel R\'enyi Estimation}
\label{app:renyi}

\textbf{Notation and setup.} This subsection provides the formal results supporting
\textbf{Sec.~\ref{sec:renyi}}, adopting the quotient setup from
Appendix~\ref{app:exploration}.
Unless stated otherwise, $m_1,\ldots,m_L$ are drawn IID from the quotient
distribution $q$ on the finite space $\mathcal M_y$.
We write
\[
C(q)
=
\sum_{m\in\mathcal M_y}q(m)^2
=
\exp[-H_2(q)]
\]
for the collision concentration, where $H_2$ denotes categorical
R\'enyi entropy of order two.

We first characterize the collision-free regime as an extreme case of the broader finite-sample coarseness of exact-count statistics.
\begin{proposition}[Sparse-collision barrier]
\label{prop:collision}
For $L$ independently sampled repair cores $m_1,\ldots,m_L \sim q_{\theta,y}$, the pairwise collision count $N_{\mathrm{col}}=\sum_{i<j}\mathbf 1[m_i=m_j]$ satisfies $\mathbb P(N_{\mathrm{col}}>0) \leq \mathbb E[N_{\mathrm{col}}] = \binom{L}{2} \sum_{m\in\mathcal M_y} q_{\theta,y}(m)^2$. Consequently, along any sequence of instances and policies for which $L^2 \sum_{m \in \mathcal M_y} q_{\theta,y}(m)^2 \to 0$, the collision probability vanishes ($\mathbb{P}(N_{\mathrm{col}} > 0) \to 0$), leaving all sampled candidates with identical zero exploration credit.
\end{proposition}

\begin{proof}
Writing $N_{\mathrm{col}} = \sum_{1 \le i < j \le L} \mathbf{1}[m_i = m_j]$, independence gives
\[
\mathbb P(m_i=m_j)
=
\sum_{m\in\mathcal M_y}q(m)^2
=
C(q).
\]
Hence, by linearity of expectation,
\[
\mathbb E[N_{\mathrm{col}}]
=
\binom{L}{2}C(q).
\]
Since $N_{\mathrm{col}}$ is nonnegative and integer valued, Markov's
inequality yields
\[
\mathbb P(N_{\mathrm{col}}>0)
\leq
\min\!\left\{
1,\,
\binom{L}{2}C(q)
\right\}.
\]

More formally, along any sequence of quotient laws $q_n$ and rollout sizes
$L_n$ satisfying
\[
L_n^2 C(q_n)\longrightarrow0,
\]
the preceding bound implies
$\mathbb P(N_{\mathrm{col}}>0)\to0$.

If the group is collision-free, its empirical distribution assigns mass
$1/L$ to each sampled core, so both its plug-in Shannon entropy and
categorical R\'enyi-2 entropy equal $\log L$.
After removing any candidate $i$, the corresponding entropy satisfies
\[
H_{-i}=\log(L-1)
\qquad
\text{for all }i.
\]
Since this holds for every $i$, the mean across the group is $\bar{H}_{-} = \frac{1}{L} \sum_{j=1}^L H_{-j} = \log(L-1)$, yielding a leave-one-out entropy contrast
\begin{equation}
\bar{H}_{-} - H_{-i} = \log(L-1) - \log(L-1) = 0
\end{equation}
identically for all $i \in \{1, \ldots, L\}$. The same argument applies to any symmetric equality-only statistic whose leave-one-out value depends solely on collision patterns. Hence, under collision-free rollout groups, exact-match leave-one-out evaluation provides zero exploration credit to all sampled candidates.
\end{proof}

We now formalize the kernel relaxation used to replace exact collisions.
Each repair core is identified with its canonical finite set of signed
reaction edits, so equality of cores coincides with equality of their
signed-edit sets.
Define
\[
J(m,m')
=
\frac{|m\cap m'|}{|m\cup m'|},
\qquad
J(\varnothing,\varnothing)=1,
\]
and, for fixed bandwidth $h>0$,
\[
k_y(m,m')
=
\exp\!\left(
-\frac{1-J(m,m')}{h}
\right).
\]
The kernel rule is shared across instances and the subscript $y$ indicates its
restriction to the instance-specific quotient space $\mathcal M_y$.

\begin{proposition}[Validity and consistency of the signed-Jaccard relaxation]
\label{prop:app_jaccard_kernel}
The kernel $k_y$ is positive semidefinite and satisfies
\[
e^{-1/h} \le k_y(m,m') \le 1, \qquad k_y(m,m) = 1.
\]
Moreover, on the finite quotient space $\mathcal{M}_y$, $k_y(m, m')$ converges uniformly to the exact-match kernel $\mathbf{1}[m = m']$ as $h \to 0$. 

Define the population kernel concentration and its associated kernelized quadratic R\'enyi functional by
\begin{equation}
V_k(q) = \mathbb{E}_{m,m' \sim q}[k_y(m,m')], \qquad H_{2,k}(q) = -\log V_k(q).
\end{equation}
Then $C(q) \le V_k(q) \le 1$, $0 \le H_{2,k}(q) \le H_2(q)$, and as $h \to 0$,
\[
V_k(q) \longrightarrow C(q), \qquad H_{2,k}(q) \longrightarrow H_2(q).
\]
\end{proposition}

\begin{proof}
We first establish positive semidefiniteness of the Jaccard similarity.
Let $\mathcal U_{\pm}$ denote the finite universe of signed reaction edits,
and let $\zeta$ be a uniformly random permutation of $\mathcal U_{\pm}$.
For each nonempty core $m$, let $b_\zeta(m)$ be its minimum element under
$\zeta$, assigning the empty core a distinguished sentinel symbol.

For any finite collection $m_1,\ldots,m_n$, the matrix
\[
G_{ab}^{\zeta}
=
\mathbf 1[b_\zeta(m_a)=b_\zeta(m_b)]
\]
is a Gram matrix of one-hot vectors and is therefore positive semidefinite.
For two nonempty cores, the standard minwise-hashing identity gives
\[
\mathbb E_\zeta[G_{ab}^{\zeta}]
=
\frac{|m_a\cap m_b|}{|m_a\cup m_b|}
=
J(m_a,m_b),
\]
with the sentinel convention giving the stated values involving the empty
core~\citep{666900,NIPS2011_0a1bf96b}.
Hence the Jaccard Gram matrix $Q$ is positive semidefinite.

For the same collection, the Gram matrix $K$ of $k_y$ can be expressed via the Hadamard exponential as
\[
K = e^{-1/h} \exp^\odot(Q/h) = e^{-1/h} \sum_{r=0}^\infty \frac{Q^{\odot r}}{h^r r!}.
\]
By the Schur product theorem, every Hadamard power $Q^{\odot r}$ is positive semidefinite. Thus every finite partial sum is positive semidefinite; since the positive-semidefinite cone is closed and the series converges in finite dimension, $K \succeq 0$. The bounds follow directly from $0 \le J \le 1$.

For the small-bandwidth limit, the case $|\mathcal{M}_y| = 1$ is immediate. If $|\mathcal{M}_y| > 1$, finiteness of $\mathcal{M}_y$ and the canonical signed-edit representation guarantee
\[
\gamma_y = \min_{m \neq m'} \bigl(1 - J(m, m')\bigr) > 0.
\]
Therefore,
\begin{equation}
\sup_{m \neq m'} k_y(m, m') \le e^{-\gamma_y/h} \longrightarrow 0 \quad \text{as } h \to 0, \label{eq:uniform_dirac_limit}
\end{equation}
while $k_y(m, m) = 1$ for all $h > 0$, proving uniform convergence to the exact-match kernel.

Finally, expanding $V_k(q)$ into diagonal and off-diagonal terms gives
\[
V_k(q) = \sum_m q(m)^2 + \sum_{m \neq m'} q(m)q(m') k_y(m, m') = C(q) + \sum_{m \neq m'} q(m)q(m') k_y(m, m') \ge C(q),
\]
while $V_k(q) \le 1$. Taking negative logarithms yields $0 \le H_{2,k}(q) \le H_2(q)$.

Note that the total off-diagonal probability mass satisfies
\[
\sum_{m \neq m'} q(m)q(m') = \left(\sum_m q(m)\right)^2 - \sum_m q(m)^2 = 1 - C(q).
\]
Consequently, factoring out the off-diagonal supremum yields
\begin{align}
0 \le V_k(q) - C(q) 
&= \sum_{m \neq m'} q(m)q(m') k_y(m, m') \notag \\
&\le \left( \sup_{m \neq m'} k_y(m, m') \right) \sum_{m \neq m'} q(m)q(m') \notag \\
&= (1 - C(q)) \sup_{m \neq m'} k_y(m, m') \notag \\
&\le \sup_{m \neq m'} k_y(m, m') \le e^{-\gamma_y/h}.
\end{align}
As $h \to 0$, this upper bound vanishes by Eq.~\eqref{eq:uniform_dirac_limit}, so $V_k(q) \to C(q)$. Because $\mathcal{M}_y$ is finite and nonempty, Cauchy--Schwarz ensures $C(q) = \sum_m q(m)^2 \ge 1/|\mathcal{M}_y| > 0$; thus continuity of $-\log$ on $(0, 1]$ implies $H_{2,k}(q) \to H_2(q)$.
\end{proof}


\begin{theorem}[Finite-group estimation and consistency]
\label{thm:finite_group}
For $L\ge 2$, define
\[
\widehat V_L
=
\frac{1}{L(L-1)}
\sum_{i\neq j}
k_y(m_i,m_j),
\qquad
\widehat H_{2,k}
=
-\log \widehat V_L,
\qquad
J_L(q)
=
\mathbb E[\widehat H_{2,k}].
\]
Then $\widehat V_L$ is an unbiased estimator of the population kernel concentration,
\[
\mathbb E[\widehat V_L]=V_k(q),
\]
and, for every fixed bandwidth $h>0$,
\[
0
\le
J_L(q)-H_{2,k}(q)
\le
B_L(h),
\qquad
B_L(h)
=
\frac{(e^{1/h}-1)^2}{4L}.
\]
Hence, for fixed $h$, $J_L(q)$ converges uniformly over quotient laws to
$H_{2,k}(q)$ at rate $O(L^{-1})$.
\end{theorem}

\begin{proof}
The statistic $\widehat{V}_L$ is the order-two U-statistic
\[
\widehat{V}_L = \frac{2}{L(L-1)} \sum_{i<j} k_y(m_i, m_j),
\]
so independence immediately gives $\mathbb{E}[\widehat{V}_L] = V_k(q)$.
Since $-\log$ is convex, Jensen's inequality implies
\[
J_L(q) = \mathbb{E}[-\log \widehat{V}_L] \ge -\log V_k(q) = H_{2,k}(q).
\]
Let $X = \widehat{V}_L$, $V = V_k(q)$, and recall that $X, V \in [e^{-1/h}, 1]$. For $f(v) = -\log v$, $f''(v) = 1/v^2$. By Taylor's theorem, for some $\xi$ between $X$ and $V$,
\[
f(X) = f(V) + f'(V)(X - V) + \frac{f''(\xi)}{2}(X - V)^2.
\]
Since $\xi \in [e^{-1/h}, 1]$, we have $f''(\xi) \le e^{2/h}$, so
\[
0 \le f(X) - f(V) - f'(V)(X - V) \le \frac{e^{2/h}}{2} (X - V)^2.
\]
Taking expectations and using $\mathbb{E}[X] = V$ yields
\begin{equation}
0 \le J_L(q) - H_{2,k}(q) \le \frac{e^{2/h}}{2} \operatorname{Var}(\widehat{V}_L). \label{eq:app_taylor_remainder}
\end{equation}
To bound the variance, let $g(m, m') = k_y(m, m')$, $g_1(m) = \mathbb{E}_{m'}[g(m, m')] - V$, and define the degenerate second-order projection $g_2(m, m') = g(m, m') - V - g_1(m) - g_1(m')$. Let $\zeta_1 = \mathbb{E}[g_1(m)^2]$, $\zeta_2 = \mathbb{E}[g_2(m, m')^2]$. The Hoeffding decomposition for an order-two U-statistic gives
\[
\operatorname{Var}(\widehat{V}_L) = \frac{4}{L} \zeta_1 + \frac{2}{L(L-1)} \zeta_2,
\]
while orthogonality of the projections implies $2\zeta_1 + \zeta_2 = \operatorname{Var}(g(m, m'))$. Therefore, for $L \ge 2$,
\begin{equation}
\operatorname{Var}(\widehat{V}_L) \le \frac{2}{L}(2\zeta_1 + \zeta_2) = \frac{2}{L} \operatorname{Var}(k_y(m, m')) \le \frac{(1 - e^{-1/h})^2}{2L}, \label{eq:app_u_stat_var}
\end{equation}
where the last inequality follows from Popoviciu's inequality since $k_y(m, m') \in [e^{-1/h}, 1]$. Substituting Eq.~\eqref{eq:app_u_stat_var} into Eq.~\eqref{eq:app_taylor_remainder} gives
\[
J_L(q) - H_{2,k}(q) \le \frac{e^{2/h} (1 - e^{-1/h})^2}{4L} = \frac{(e^{1/h} - 1)^2}{4L}.
\]
Therefore, for every fixed $h>0$,
\[
\sup_q
\left|
J_L(q)-H_{2,k}(q)
\right|
=
O(L^{-1}),
\]
which establishes uniform consistency over quotient laws.
\end{proof}


\textbf{Scope of the result.}
Theorem~\ref{thm:finite_group} establishes uniform $O(L^{-1})$ consistency for fixed bandwidth $h$.
The distribution-free bound is a conservative worst-case guarantee and is not intended as a tight characterization of finite rollout groups.
The limit $H_{2,k}(q)\to H_2(q)$ is a population result, while signed-Jaccard similarity quantifies structural overlap rather than biological equivalence.

\subsection{QuotientPO Optimization Theory}
\label{app:optimization}

This subsection establishes the theoretical role of the quotient-diversity
term in Eq.~\eqref{eq:qpo_objective}.
We deliberately separate this analysis from the verifier-shaped reward,
GRPO clipping, and auxiliary discovery and retention mechanisms, which are
optimization choices detailed in Appendix~\ref{app:algorithm}.
Our goal here is twofold: \textbf{to characterize the candidate-level crowding cost
in Eq.~\eqref{eq:qrl_kernel_credit}, and to relate its policy update to the
population kernel R\'enyi functional introduced in
Appendix~\ref{app:renyi}.}

\textbf{Notation and setup.} In this subsection, fix a training instance and suppress its conditioning. For a rollout group $m_{1:L}$, recall the group kernel concentration $\widehat{V}_L = \frac{1}{L(L-1)} \sum_{a \neq b} k_y(m_a, m_b)$ from Theorem~\ref{thm:finite_group}, and define the candidate-level average kernel crowding
\begin{equation}
\bar{k}_i = \frac{1}{L-1} \sum_{j \neq i} k_y(m_i, m_j), \qquad \text{so that} \quad \frac{1}{L}\sum_{i=1}^L \bar{k}_i = \widehat{V}_L.
\end{equation}

\begin{proposition}[Candidate crowding from leave-one-out kernel entropy]
\label{prop:app_loo_crowding}
For $L \ge 3$, let $\widehat{H}_{-i} = -\log \widehat{V}_{-i}$ be the kernel R\'enyi entropy of the group after removing candidate $i$, where
\[
\widehat{V}_{-i} = \frac{L \widehat{V}_L - 2\bar{k}_i}{L-2}.
\]
Then the leave-one-out contrast $\Delta_i = \widehat{H}_{2,k} - \widehat{H}_{-i}$ satisfies
\begin{equation}
\Delta_i = \log \left( 1 - \frac{2(\bar{k}_i - \widehat{V}_L)}{(L-2)\widehat{V}_L} \right). \label{eq:app_exact_loo_contrast}
\end{equation}
As $L \to \infty$, the first-order Taylor expansion around the group mean $\widehat{V}_L$ yields
\begin{equation}
\Delta_i = - \frac{c_i}{L-2} + \mathcal{O}\left( \frac{1}{L^2} \right), \label{eq:app_linearized_loo}
\end{equation}
where $c_i = \frac{2 \bar{k}_i}{\widehat{V}_L} - 2$ is the crowding cost in Eq.~\eqref{eq:qrl_kernel_credit}. Moreover,
\[\sum_{i=1}^L c_i = 0, \quad c_i > 0 \iff \bar{k}_i > \widehat{V}_L ,\quad c_i < 0 \iff \bar{k}_i < \widehat{V}_L.
\]
\end{proposition}

\begin{proof}
By the kernel symmetry $k_y(m, m') = k_y(m', m)$, removing candidate $i$ deletes $2\sum_{j \neq i} k_y(m_i, m_j) = 2(L-1)\bar{k}_i$ off-diagonal mass from the total group mass $L(L-1)\widehat{V}_L$. The remaining $(L-1)(L-2)$ ordered pairs have average concentration
\begin{equation}
\widehat{V}_{-i} = \frac{L(L-1)\widehat{V}_L - 2(L-1)\bar{k}_i}{(L-1)(L-2)} = \frac{L\widehat{V}_L - 2\bar{k}_i}{L-2} = \widehat{V}_L - \frac{2(\bar{k}_i - \widehat{V}_L)}{L-2}. \label{eq:app_v_minus_i_derivation}
\end{equation}
By Proposition~\ref{prop:app_jaccard_kernel}, $k_y(m, m') \ge e^{-1/h} > 0$ holds uniformly, guaranteeing that $\widehat{V}_{-i} \ge e^{-1/h} > 0$ and the logarithm is strictly well-defined. Dividing Eq.~\eqref{eq:app_v_minus_i_derivation} by $\widehat{V}_L$ and evaluating the leave-one-out entropy contrast gives
\begin{align}
\Delta_i 
= \widehat{H}_{2,k} - \widehat{H}_{-i} \notag 
= -\log \widehat{V}_L - (-\log \widehat{V}_{-i}) \notag 
= \log \left( \frac{\widehat{V}_{-i}}{\widehat{V}_L} \right) \notag 
= \log \left( 1 - \frac{2(\bar{k}_i - \widehat{V}_L)}{(L-2)\widehat{V}_L} \right),
\end{align}
which establishes Eq.~\eqref{eq:app_exact_loo_contrast}.

Next, let $x_i = \frac{2(\bar{k}_i - \widehat{V}_L)}{(L-2)\widehat{V}_L}$. Since $|\bar{k}_i - \widehat{V}_L| \le 1 - e^{-1/h}$ and $\widehat{V}_L \ge e^{-1/h}$, we have $x_i = \mathcal{O}(L^{-1})$. Applying the Taylor expansion $\log(1 - x) = -x + \mathcal{O}(x^2)$ around $x=0$ yields
\begin{align*}
\Delta_i &= -x_i + \mathcal{O}(x_i^2) = - \frac{2(\bar{k}_i - \widehat{V}_L)}{(L-2)\widehat{V}_L} + \mathcal{O}(L^{-2}) = - \frac{c_i}{L-2} + \mathcal{O}(L^{-2}),
\end{align*}
which establishes Eq.~\eqref{eq:app_linearized_loo} and shows that the candidate crowding cost $c_i = \frac{2\bar{k}_i}{\widehat{V}_L} - 2$ governs the leading leave-one-out entropy rate.

Finally, using the definition $\sum_{i=1}^L \bar{k}_i = L\widehat{V}_L$ from Eq.~\eqref{eq:qrl_kernel_credit}, the group sum of the costs satisfies
\begin{equation}
\sum_{i=1}^L c_i = \frac{2}{\widehat{V}_L} \sum_{i=1}^L \bar{k}_i - 2L = \frac{2}{\widehat{V}_L}(L\widehat{V}_L) - 2L = 0.
\end{equation}
Rewriting the cost as $c_i = \frac{2}{\widehat{V}_L}(\bar{k}_i - \widehat{V}_L)$ with strictly positive normalizer $\widehat{V}_L > 0$, the signs satisfy $c_i > 0 \iff \bar{k}_i > \widehat{V}_L$ and $c_i < 0 \iff \bar{k}_i < \widehat{V}_L$ directly by affine monotonicity.
\end{proof}

The preceding result gives a finite-group interpretation of the implemented
cost.
We next show that the same quantity converges to the score coefficient of
the population kernel R\'enyi objective.
Define the population kernel density of a repair core $m$ by
\[
d_q(m)
=
\mathbb E_{m'\sim q}[k_y(m,m')],
\]
so that
\[
V_k(q)
=
\mathbb E_{m\sim q}[d_q(m)].
\]
Define the corresponding population crowding coefficient
\begin{equation}
c_q^k(m)
=
2\frac{d_q(m)}{V_k(q)}-2.
\label{eq:app_population_crowding}
\end{equation}
\begin{theorem}[Population score identity and local consistency]
\label{thm:app_qpo_kernel_gradient}
Assume that $\pi_\theta$ is differentiable with locally fixed support and integrable score $\mathbb E_{E\sim\pi_\theta}[\|g_\theta(E)\|] < \infty$, and that $\kappa_y$ and $k_y$ remain fixed during the policy update. Let $g_\theta(E) = \nabla_\theta \log \pi_\theta(E \mid M_0,G)$ and let $q_{\theta,y}$ denote the distribution of repair cores $m = \kappa_y(E)$ induced by $E \sim \pi_\theta$ (abbreviated as $q$). Then:
\begin{enumerate}[leftmargin=*,label=\rm(\roman*)]
    \item \textbf{Score Identity:} The population kernel R\'enyi functional satisfies
    \begin{equation}
    \nabla_\theta H_{2,k}(q_{\theta,y}) = - \mathbb E_{E\sim \pi_\theta}\left[c_{q_{\theta,y}}^k(\kappa_y(E))\,g_\theta(E)\right].
    \label{eq:app_population_kernel_gradient}
    \end{equation}
    \item \textbf{Local Consistency:} For the empirical crowding cost $c_i$ in Eq.~\eqref{eq:app_linearized_loo}, we have $|c_i - c_q^k(m_i)| \to 0$ in probability as $L \to \infty$. Consequently, for the detached quotient-diversity surrogate $\mathcal J_{k,L}(\theta) = -\frac{\beta}{L} \sum_{i=1}^L \rho_i(\theta) \mathrm{sg}(c_i)$ in Eq.~\eqref{eq:qpo_objective}, its expected gradient at the behavior policy satisfies
    \begin{equation}
    \lim_{L\to\infty} \mathbb E\left[\nabla_\theta \mathcal J_{k,L}(\theta)\big|_{\theta=\theta_{\mathrm{old}}}\right] = \beta \nabla_\theta H_{2,k}(q_{\theta_{\mathrm{old}},y}).
    \label{eq:app_kernel_gradient_consistency}
    \end{equation}
\end{enumerate}
Thus the QuotientPO kernel term is a locally consistent score surrogate for population quotient diversity as the rollout group grows.
\end{theorem}

\begin{proof}
Using the deterministic core map $m = \kappa_y(E)$, the population concentration can be written as
\begin{equation*}
V_k(q_{\theta,y}) = \mathbb E_{E,E'\sim \pi_\theta}\left[k_y(\kappa_y(E), \kappa_y(E'))\right].
\end{equation*}
Differentiating the two IID policy factors gives
\begin{equation*}
\nabla_\theta V_k(q_{\theta,y}) = 2 \, \mathbb E_{E\sim \pi_\theta}\left[d_q(\kappa_y(E))\,g_\theta(E)\right].
\end{equation*}
Since $H_{2,k}(q) = -\log V_k(q)$, we have
\begin{equation*}
\nabla_\theta H_{2,k}(q_{\theta,y}) = - \frac{2}{V_k(q_{\theta,y})} \mathbb E_{E\sim \pi_\theta}\left[d_q(\kappa_y(E))\,g_\theta(E)\right].
\end{equation*}
The score identity $\mathbb E_{E\sim \pi_\theta}[g_\theta(E)] = 0$ allows baseline subtraction without altering the expectation, yielding
\begin{equation*}
\nabla_\theta H_{2,k}(q_{\theta,y}) = - \mathbb E_{E\sim \pi_\theta}\left[\left(2\frac{d_q(\kappa_y(E))}{V_k(q_{\theta,y})} - 2\right) g_\theta(E)\right] = - \mathbb E_{E\sim \pi_\theta}\left[c_q^k(\kappa_y(E))\,g_\theta(E)\right],
\end{equation*}
which proves Eq.~\eqref{eq:app_population_kernel_gradient}.

Next, we establish the local consistency of the empirical surrogate. Conditional on $m_i$, the remaining $L-1$ cores are IID draws from $q$. By the weak law of large numbers, $\bar{k}_i \to d_q(m_i)$ in probability as $L \to \infty$. In addition, the bounded order-two U-statistic satisfies $\widehat{V}_L \to V_k(q)$ in probability. Because $V_k(q) \ge e^{-1/h} > 0$ is bounded away from zero, the continuous mapping theorem implies that $|c_i - c_q^k(m_i)| \to 0$ in probability as $L \to \infty$.

Finally, evaluating the surrogate gradient at the behavior policy gives $\nabla_\theta \rho_i(\theta)|_{\theta=\theta_{\mathrm{old}}} = g_{\theta_{\mathrm{old}}}(E_i)$. By the exchangeability of the IID rollout sample $(E_1, \dots, E_L)$,
\begin{equation*}
\mathbb E\left[\nabla_\theta \mathcal J_{k,L}(\theta)\big|_{\theta=\theta_{\mathrm{old}}}\right] = -\beta \mathbb E\left[\frac{1}{L}\sum_{i=1}^L c_i g_{\theta_{\mathrm{old}}}(E_i)\right] = -\beta \mathbb E\left[c_1^k g_{\theta_{\mathrm{old}}}(E_1)\right].
\end{equation*}
Since $k_y \in [e^{-1/h}, 1]$, the statistics $\bar{k}_1$ and $\widehat{V}_L$ both lie in $[e^{-1/h}, 1]$, so $c_1^k$ is uniformly bounded in $L$ by $2(e^{1/h}-1)$. Combined with score integrability $\mathbb E[\|g_{\theta_{\mathrm{old}}}(E_1)\|] < \infty$, the dominated convergence theorem implies
\begin{equation*}
\lim_{L\to\infty} \mathbb E\left[c_1^k g_{\theta_{\mathrm{old}}}(E_1)\right] = \mathbb E\left[c_q^k(\kappa_y(E_1))\,g_{\theta_{\mathrm{old}}}(E_1)\right].
\end{equation*}
Substituting this limit back into the surrogate expectation and applying Eq.~\eqref{eq:app_population_kernel_gradient} directly yields Eq.~\eqref{eq:app_kernel_gradient_consistency}.
\end{proof}

\textbf{Relation to deployment-time success and mechanism coverage.}
Because $k$ is already used for the similarity kernel, let $K$ denote the
number of independently generated repairs in the evaluation metric
success@$K$.
For a fixed evaluation instance $x$, let
$s_x(E)\in\{0,1\}$ denote the evaluation success predicate and define the
single-draw success probability
\[
Z_x(\theta)
=
\mathbb P_{E\sim\pi_\theta(\cdot\mid x)}
[s_x(E)=1].
\]

\begin{proposition}[Successful quotient mass determines Success@$K$]
\label{prop:app_successk}
For $K\geq1$ IID repairs sampled from $\pi_\theta(\cdot\mid x)$, the
probability that at least one repair succeeds is
\begin{equation}
P_K(x;\theta)
=
1-
\bigl(1-Z_x(\theta)\bigr)^K.
\label{eq:app_successk}
\end{equation}
Therefore $P_K(x;\theta)$ is monotonically increasing in
$Z_x(\theta)$, and strictly increasing whenever $Z_x(\theta)<1$.
If the evaluation success predicate is invariant within repair cores, i.e.,
there exists $s_x^{\mathrm{core}}$ such that
\(
s_x(E)
=
s_x^{\mathrm{core}}(\kappa_y(E)),
\)
then
\begin{equation}
Z_x(\theta)
=
\sum_{m\in\mathcal M_y}
q_{\theta,y}(m)\,
s_x^{\mathrm{core}}(m).
\label{eq:app_successful_core_mass}
\end{equation}
Thus success@$K$ is determined by the total quotient probability mass assigned
to successful repair cores.
\end{proposition}

\begin{proof}
A single repair fails with probability $1-Z_x(\theta)$.
Under IID sampling, the probability that all $K$ repairs fail is therefore
\[
\bigl(1-Z_x(\theta)\bigr)^K.
\]
Taking the complement proves Eq.~\eqref{eq:app_successk}.
Moreover,
\[
\frac{\partial P_K(x;\theta)}
{\partial Z_x}
=
K(1-Z_x)^{K-1},
\]
which is positive for $Z_x<1$.

Under core invariance of the success predicate,
\begin{align*}
Z_x(\theta)
&=
\sum_E
\pi_\theta(E\mid x)
s_x^{\mathrm{core}}(\kappa_y(E))\\
&=
\sum_{m\in\mathcal M_y}
s_x^{\mathrm{core}}(m)
\sum_{E:\kappa_y(E)=m}
\pi_\theta(E\mid x)\\
&=
\sum_{m\in\mathcal M_y}
q_{\theta,y}(m)
s_x^{\mathrm{core}}(m),
\end{align*}
which proves Eq.~\eqref{eq:app_successful_core_mass}.
\end{proof}

\begin{proposition}[Expected distinct successful-core coverage]
\label{prop:successful_core_coverage}
Let
\[
\mathcal M_x^{\mathrm{succ}}
=
\left\{
m\in\mathcal M_y:
s_x^{\mathrm{core}}(m)=1
\right\}
\]
denote the successful repair cores for instance $x$.
For $K$ IID repairs $E_1,\ldots,E_K\sim\pi_\theta(\cdot\mid x)$,
let
\[
D_K
=
\left|
\left\{
\kappa_y(E_t):
1\leq t\leq K,\;
s_x(E_t)=1
\right\}
\right|
\]
be the number of distinct successful cores discovered within the
$K$ samples. Then
\[
\mathbb E[D_K]
=
\sum_{m\in\mathcal M_x^{\mathrm{succ}}}
\left[
1-
\left(1-q_{\theta,y}(m)\right)^K
\right].
\]
Thus, unlike Success@K, distinct successful-core coverage depends
not only on the total successful mass
\[
Z_x(\theta)
=
\sum_{m\in\mathcal M_x^{\mathrm{succ}}}
q_{\theta,y}(m),
\]
but also on how that mass is distributed across successful cores.
\end{proposition}

\begin{proof}
For each successful core $m\in\mathcal M_x^{\mathrm{succ}}$,
define
\[
I_m
=
\mathbf 1
\left[
m\in
\{\kappa_y(E_1),\ldots,\kappa_y(E_K)\}
\right].
\]
Then
\[
D_K
=
\sum_{m\in\mathcal M_x^{\mathrm{succ}}} I_m.
\]
Because each sampled repair induces core $m$ with probability
$q_{\theta,y}(m)$, independence gives
\[
\Pr(I_m=0)
=
\left(1-q_{\theta,y}(m)\right)^K.
\]
Hence
\[
\mathbb E[I_m]
=
1-\left(1-q_{\theta,y}(m)\right)^K.
\]
Applying linearity of expectation yields
\[
\mathbb E[D_K]
=
\sum_{m\in\mathcal M_x^{\mathrm{succ}}}
\left[
1-\left(1-q_{\theta,y}(m)\right)^K
\right].
\]
\end{proof}

\begin{corollary}[Successful-mass equalization improves mechanism coverage]
Fix an instance $x$, a sampling budget $K>1$, and a set of $S$ successful repair cores.
For a fixed total successful mass
\[
\sum_{m\in\mathcal{M}^{\mathrm{succ}}_x} q_{\theta,y}(m)=Z,
\]
the expected distinct successful-core coverage
\[
\mathbb{E}[D_K]
=
\sum_{m\in\mathcal{M}^{\mathrm{succ}}_x}
\left[
1-\left(1-q_{\theta,y}(m)\right)^K
\right]
\]
is Schur-concave in the successful-core probability vector.
Consequently, any mean-preserving redistribution that equalizes probability mass across successful cores weakly increases expected coverage.
The maximum is attained at the uniform allocation
\[
q_{\theta,y}(m)=\frac{Z}{S},
\]
yielding
\[
\mathbb{E}[D_K]
=
S\left[
1-\left(1-\frac{Z}{S}\right)^K
\right].
\]
Meanwhile, Success@$K$ remains
\[
1-(1-Z)^K,
\]
and therefore depends only on the total successful mass $Z$.
\end{corollary}

\begin{proof}
For $K>1$, define
\[
f_K(p)=1-(1-p)^K.
\]
Since
\[
f_K''(p)
=
-K(K-1)(1-p)^{K-2}
\le 0,
\]
$f_K$ is concave on $[0,1]$.
Therefore,
\[
\mathbb{E}[D_K]
=
\sum_{m\in\mathcal{M}^{\mathrm{succ}}_x}
f_K\!\left(q_{\theta,y}(m)\right)
\]
is a symmetric Schur-concave function of the successful-core probability vector.
Hence, any mean-preserving equalization of this vector weakly increases expected coverage.
Under the fixed-sum constraint
\[
\sum_{m\in\mathcal{M}^{\mathrm{succ}}_x}
q_{\theta,y}(m)=Z,
\]
the maximum is attained at the uniform allocation
\[
q_{\theta,y}(m)=\frac{Z}{S}.
\]
Substituting this value gives
\[
\mathbb{E}[D_K]
=
S\left[
1-\left(1-\frac{Z}{S}\right)^K
\right].
\]
By Proposition~B.12, Success@$K$ is
\[
1-(1-Z)^K,
\]
which depends only on the total successful mass $Z$.
\end{proof}

\paragraph{Interpretation and scope.}
Propositions~\ref{prop:app_successk}--\ref{prop:successful_core_coverage} reveal two complementary dimensions of deployment-time performance.
Success@$K$ is governed solely by the total probability mass $Z_x$ assigned to successful repair cores,
whereas distinct-core coverage additionally rewards how broadly this mass is distributed across successful mechanisms.
At fixed $Z_x$, equalizing probability across successful cores monotonically improves expected mechanism coverage,
with uniform allocation achieving the optimum.
Thus, successful-mass concentration and successful-mass allocation control different aspects of test-time scaling:
increasing $Z_x$ raises the probability of obtaining a valid repair,
while distributing that mass across distinct successful cores increases the number of mechanisms recoverable under the same sampling budget.

This characterization motivates quotient-space exploration at the mechanism level.
By shifting probability away from crowded regions of the repair-core distribution,
the kernelized quotient objective promotes broader allocation across candidate mechanisms,
while verifier-guided optimization concentrates probability on functionally successful regions.
Their combination therefore aligns naturally with the two quantities identified above:
successful probability mass and successful-core coverage.
Empirically, QuotientPO improves both,
yielding higher Success@$K$ and greater distinct successful-core recovery under the same sampling budget.

\section{Dataset and Benchmark Construction}
\label{app:data}

\subsection{Construction Pipeline}  

\textbf{Construction of initial GEM and proteome matching.} BacDive phenotype records were linked to NCBI~\citep{ncbi} proteomes via prefix/suffix-normalized INSDC assembly accessions. Retaining matches with available \texttt{protein.faa(.gz)} files yielded 12,303 proteome inputs across 11,238 unique assemblies. We constructed candidate GEMs under four CarveMe settings: default, standard Gram-positive, standard Gram-negative, and a custom archaeal universe (augmenting CarveMe with archaeal biomass and ATP maintenance), producing 36,855 SBML Level~3/FBC2 GEMs as uncurated starting hypotheses. Concurrently, raw BacDive records were filtered for unambiguous binary phenotypes and mapped to BiGG identifiers. Retaining organisms with at least one mapped phenotype label produced 11,275 organisms and 37,864 organism--model entries (as multiple BacDive strains may share an assembly). This cohort was filtered and bifurcated into two downstream pipelines: 3,978 phenotype-consistent GEMs under the selected assays were identified as candidates for self-supervised pre-training, 28,842 canonical error-bearing GEMs formed the post-training cohort. Records sharing the same normalized accession were assigned to the same split to prevent genome-level leakage.

\begin{table*}[h]
\centering
\small
\setlength{\tabcolsep}{8pt}
\renewcommand{\arraystretch}{1.1}
\caption{
\textbf{Nutrient-growth assay coverage.}
}
\label{tab:bacdive_families}
\begin{tabularx}{0.86\textwidth}{
@{}>{\raggedright\arraybackslash}p{0.22\textwidth}
>{\raggedright\arraybackslash}X@{}
}
\toprule
\textbf{Assay category} & \textbf{BacDive families} \\
\midrule
Carbon source
& glucose, maltose, sucrose, glycerol, acetate, lactate,
pyruvate, succinate, citrate, glutamate \\
Nitrogen source
& ammonium, urea, nitrate, nitrite \\
Sulfur source
& sulfate \\
Phosphorus source
& phosphate \\
\bottomrule
\end{tabularx}
\end{table*}


\textbf{Construction of BacDive targets.} Binary BacDive nutrient-utilization phenotypes were mapped to BiGG metabolites across 16 nutrient families (Table~\ref{tab:bacdive_families}), with stereospecific assays for lactate and glutamate yielding 18 total assay configurations. In each assay, the tested nutrient $s$ serves as the sole elemental source for its macro-element class (C, N, S, or P) against an element-depleted minimal background to preclude metabolic bypass, with flux bounds set to 10 for the tested substrate, 20 for oxygen, and 1,000 for non-limiting element-free inorganics and water. To ensure consistent boundary conditions across draft GEMs, native exchanges were replaced by standardized boundary reactions (excluded from the editable action space), and all unsupplemented exchange uptakes were strictly disabled. Growth is predicted when optimal biomass flux exceeds $10^{-3}$. For multi-assay labels, positive predictions require growth in at least one assay, while negative predictions require absence across all. Requiring at least one decidable target and one initial phenotype error yields 21,615 post-training eligible GEMs across all partitions.

\textbf{Certified repair supervision.} Teacher repair sets were generated offline via mixed-integer optimization (Appendix~\ref{app:milp}) and verified by FBA under identical assay conditions. We distinguish \emph{complete repairs} (correcting all phenotype errors) from \emph{partial repairs} (strictly reducing net errors while preserving initially correct phenotypes). Partial repairs provide action-level supervision for SFT-A (masking actions that temporarily degrade phenotype accuracy), whereas complete repairs supervise set-level prediction in SFT-B. Training states are constructed from unordered subsets of verified repairs to eliminate arbitrary edit-order assumptions, retaining multiple valid solutions with per-GEM weight normalization. This supervision covers 1,426 GEMs (10,837 states) for SFT-A and 1,078 GEMs (40,151 rows) for SFT-B after augmentation. Crucially, phenotype labels, error diagnostics, and teacher edits are never exposed as policy inputs.


\begin{figure*}[t]
    \centering
    \includegraphics[width=\textwidth]{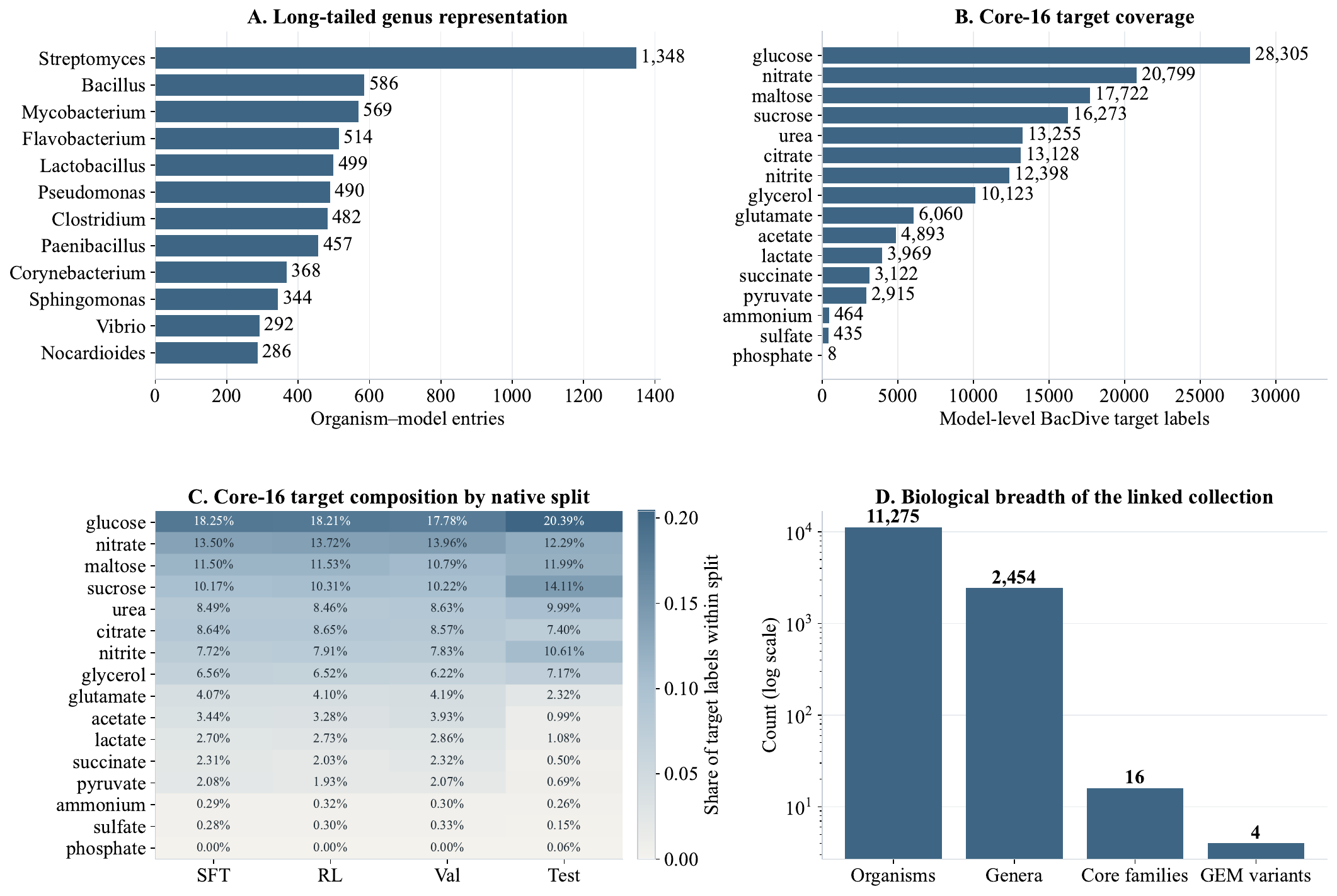}
    \caption{
    \textbf{Dataset characterization and split diagnostics.}
    }
    \label{fig:dataset_characterization}
    \vspace{-0.2cm}
\end{figure*}

\textbf{Data partitioning and leakage prevention.}
Dataset partitions were fixed before post-training eligibility filtering.
Entries connected by normalized organism identity, proteome identity, GEM provenance, or model checksum were merged by transitive closure into 8,568 indivisible identity components.
A genus-held-out test set contains \textit{Bacillus}, \textit{Clostridium}, \textit{Lactobacillus}, \textit{Mycobacterium}, and \textit{Streptomyces}.
The remaining identity components were deterministically assigned to the SFT, RL, and validation pools:
\(
28{,}842
=
12{,}223_{\mathrm{SFT}}
+
12{,}226_{\mathrm{RL}}
+
1{,}445_{\mathrm{validation}}
+
2{,}948_{\mathrm{test}}.
\)
After phenotype-assay filtering and certified-teacher selection, SFT-A retained 1,426/177 training/validation GEMs and SFT-B retained 1,078/117.
The RL split retained 9,196/1,087 training/validation GEMs, while the final repair test set contained 2,212 GEMs.
Eligible models in the SFT pool without certified supervision were left unused rather than reassigned to RL.
All augmented states inherit the partition of their source GEM.
Beyond the component-level partitioning of the post-training cohort, the 3,978 clean GEMs allocated to self-supervised pre-training were strictly quarantined from the outset.
Pre-training explicitly excludes all five held-out evaluation genera and shares zero normalized INSDC accessions or proteomes with the test split.
Auditing confirmed zero biological identity overlap across the entire lifecycle (Pre-training, SFT, RL, Validation and Test), preventing models from acquiring prior familiarity with test-split metabolic networks.

\subsection{Dataset Characterization and Quality Control}
\label{app:dataset_characterization}

Figure~\ref{fig:dataset_characterization} summarizes the biological breadth and split-level properties of the post-training collection. The linked dataset is taxonomically long-tailed, with 11,275 model-linked organisms spanning 2,454 genera and four GEM reconstruction variants (Fig.~\ref{fig:dataset_characterization}A,D). The core-16 supervision covers 153,869 model-level BacDive target groups across the 28,842 retained GEMs (Fig.~\ref{fig:dataset_characterization}B). These counts measure phenotype-target coverage rather than repair errors: the same organism contributes separately through different GEM variants, and distinct compound-level annotations remain independent targets. Across the native SFT, RL, and validation pools,  major phenotype partitions are broadly consistent, while the sealed genus-held-out test set exhibits the expected distribution shifts induced by taxonomic holdout (Fig.~\ref{fig:dataset_characterization}C). Together, these diagnostics establish broad biological coverage, explicit bidirectional phenotype supervision, and a leakage-controlled evaluation shift.

\subsection{Benchmark Validation}
\label{app:benchmark_validation}

The sealed benchmark (2,212 GEMs) evaluates generalization across biological taxa, reconstruction templates, and phenotype error modes (Figure~\ref{fig:benchmark_validation}). Taxonomically, it spans five held-out genera: \textit{Streptomyces} (926), \textit{Bacillus} (443), \textit{Lactobacillus} (403), \textit{Mycobacterium} (266), and \textit{Clostridium} (174) (Fig.~\ref{fig:benchmark_validation}A). It balances four draft reconstruction templates: 697 default, 517 archaeal, 507 Gram-negative, and 491 Gram-positive models (Fig.~\ref{fig:benchmark_validation}B). Across the 16 core metabolites, initial errors display dominant false negatives alongside non-trivial false positives (Fig.~\ref{fig:benchmark_validation}C), requiring policies to coordinate reaction additions and deletions rather than relying on monotonic gap-filling.

\begin{figure*}[h]
    \centering
    \includegraphics[width=\textwidth]{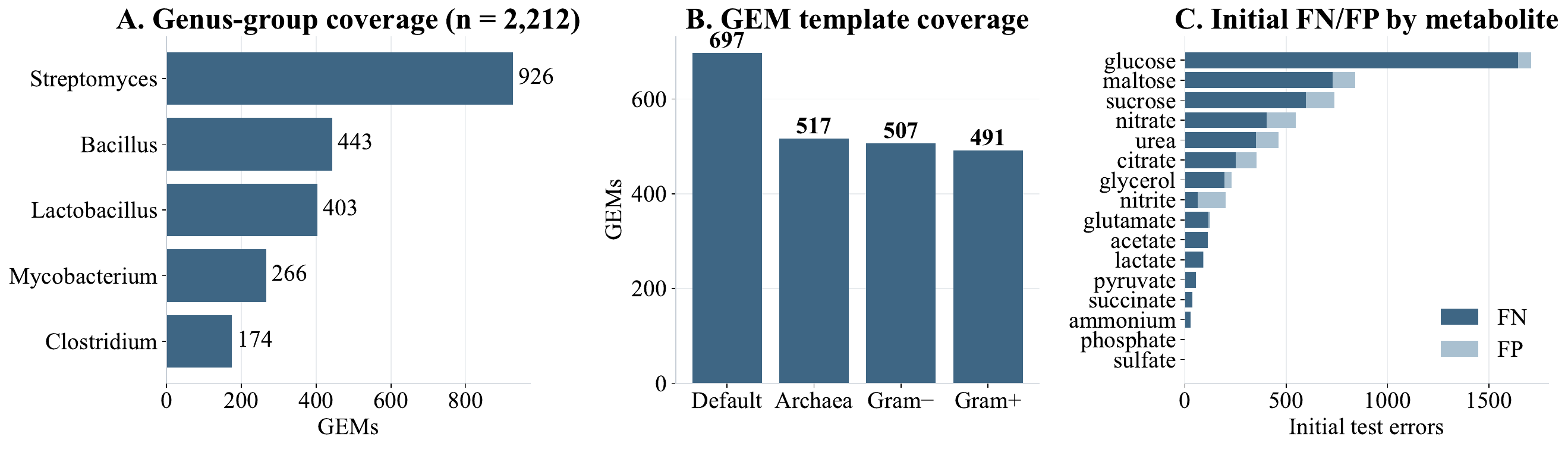}
    \caption{
    \textbf{Benchmark validation.}
    }
    \label{fig:benchmark_validation}
    \vspace{-0.5cm}
\end{figure*}

\section{Baseline Details}
\label{app:baselines}

This section details the baselines that require bespoke algorithmic reformulation to operate over metabolic networks:
the phenotype-informed MILP oracle (\S\ref{app:milp}) and the object-level GFlowNet (\S\ref{app:gflownet}),
both adapted from external paradigms to phenotype-supervised GEM repair. The remaining baselines build directly upon our neural architecture and training pipeline: imitation models (\textit{SFT-Soft} and \textit{SFT-Hard}) are described alongside architectural details in Appendix~\ref{app:sft}, while reinforcement learning policies (Vanilla GRPO, MaxEnt, and QuotientPO) are detailed in Appendix~\ref{app:algorithm} and \ref{app:rl}.

\subsection{MILP}
\label{app:milp}

\textbf{Mathematical formulation and phenotype-guided separation.} We construct a phenotype-informed mixed-integer linear programming (MILP) oracle that couples discrete reaction edits with continuous intracellular fluxes. For an initial GEM $M_0$ with reaction set $\mathcal R_0$ and a candidate reaction universe $\mathcal U$, a binary decision vector $x \in \{0,1\}^{|\mathcal U|}$ specifies toggled reactions, defining the repair edit set $E(x) = (A(x), D(x))$ with additions $A(x) = \{r \in \mathcal U \setminus \mathcal R_0 : x_r = 1\}$ and deletions $D(x) = \{r \in \mathcal R_0 \cap \mathcal U : x_r = 1\}$, such that $\|x\|_1 = |E(x)|$. For each nutrient assay $s$, biomass flux is evaluated via flux balance analysis (FBA), formulated as a continuous linear programming (LP) problem:
\[
v^*_s(x) = \max_v \left\{c^\top v : Sv = 0, \; \ell_s(x) \le v \le u_s(x)\right\}.
\]
Given target phenotype labels $y_i \in \{0, 1\}$ and mapped assay sets $\mathcal A_i$, the predicted phenotype is $\widehat y_i(x) = \mathbf 1 [\max_{s \in \mathcal A_i} v^*_s(x) > 10^{-3}]$, yielding the total phenotype error $\mathrm{Err}(x) = \sum_i \mathbf 1 [\widehat y_i(x) \neq y_i]$. The oracle accesses test-time labels to lexicographically minimize phenotype errors before edit cardinality: $\min_x (\mathrm{Err}(x), \|x\|_1)$. We implement this hierarchy via binary violation indicators $z_i \in \{0, 1\}$:
\[
\min_{x,z} \quad \sum_{r \in \mathcal U} x_r + (|\mathcal U| + 1) \sum_i z_i,
\]
where $z_i = 0$ requires $\widehat y_i(x) = y_i$. Because $\|x\|_1 \le |\mathcal U|$, the weight $(|\mathcal U| + 1)$ guarantees that resolving any phenotype strictly dominates toggling candidate reactions. 
Rather than embedding continuous fluxes into a large big-$M$ formulation, the discrete master
optimizes $(x,z)$ and generates phenotype-guided cuts through FBA separation. At iteration $k$,
a violated phenotype induces
\[
D_H(x,x^{(k)}) + z_i \geq 1,
\qquad
D_H(x,x^{(k)}) =
\sum_{r\in H}\left|x_r-x_r^{(k)}\right|,
\]
where $H\subseteq U$ concentrates the separation on phenotype-relevant reactions identified from
LP dual information for failed positive phenotypes or active flux support for failed negative phenotypes.
These phenotype-guided restricted-support cuts concentrate the finite-budget search on functionally
relevant edit configurations, while every retained candidate is independently certified by complete
FBA evaluation under the original assay conditions.

\textbf{Solver implementation and verification pipeline.}
Each GEM receives a 600-second search budget using double-precision HiGHS
~\citep{huangfu2015parallelizingdualrevisedsimplex} with one thread per GEM, LP feasibility tolerance $1e{-8}$,
MIP integrality tolerance $1e{-9}$, and zero requested MIP gap.
The search allocates 120 seconds to greedy subset repair (up to 1,000 separation iterations)
and the remainder to lexicographic search (up to 4,000 iterations).
Each candidate is independently evaluated by the complete runtime FBA verifier, which determines
its phenotype error under the original assay conditions.
Across the 2,212 sealed test GEMs, we retain the FBA-verified candidate with the best
lexicographic objective $(\mathrm{Err}(x),\|x\|_1)$ encountered within the search budget,
with unimproved cases retaining $M_0$.
This phenotype-informed oracle therefore provides a strong label-access baseline for evaluating
amortized repair policies.

\subsection{GFlowNet}
\label{app:gflownet}

We adapt GFlowNet~\citep{bengio2021flow} to one-shot GEM repair as an
object-level diversity baseline.
It is to test whether distributing probability across high-reward
raw repair sets, without identifying equivalent repairs, is sufficient
to obtain the multi-solution benefits targeted by QuotientPO.

\textbf{Object-level GFlowNet objective.}
Let $\pi_0$ denote the frozen SFT policy used to initialize RL.
For a complete repair set $E$, we define the reward-tilted target
distribution
\begin{equation}
p^\star(E\mid M_0,G,y)
\propto
\pi_0(E\mid M_0,G)^\gamma
\exp\!\left(R_y(E)/\tau\right),
\label{eq:gfn_target}
\end{equation}
where $\tau$ controls reward temperature and $\gamma$ controls the strength
of the SFT base distribution.
The target therefore favors high-reward repairs while retaining a structured
prior over plausible edits.
Crucially, its support is defined over raw edit sets: two repairs
remain distinct GFlowNet objects whenever their edit sets differ, even if
they canonicalize to the same repair core.
This provides a direct object-space counterpart to the quotient-space
exploration of QuotientPO.

\textbf{One-shot adaptation.}
Unlike standard trajectory-based GFlowNets, our policy predicts the entire unordered repair set in a single forward pass via independent Bernoulli decoders, requiring no intermediate states, action orderings, or learned backward policies. Flow consistency under target~\eqref{eq:gfn_target} then requires policy probabilities to match target weights up to a partition constant. For a rollout group of $L$ candidates, defining the log flow residual $z_i = R_y(E_i)/\tau + \gamma\log\pi_0(E_i) - \log\pi_\theta(E_i)$ and eliminating the unknown instance offset yields the log-partition variance (LPV) loss,
\[
\mathcal L_{\mathrm{LPV}} = \frac{1}{L}\sum_{i=1}^{L}(z_i-\bar z)^2,
\]
which enforces reward-proportional probabilities without explicit partition-function estimation.

\textbf{SFT anchor and task adaptation.}
A uniform base distribution over complete edit sets concentrates probability on unrealistically large cardinalities (averaging half of all available actions), causing unanchored exploration ($\gamma=0$) to rapidly degenerate into thousands of spurious ADD edits (Table~\ref{tab:gfn_anchor}). The SFT anchor in Eq.~\eqref{eq:gfn_target} resolves this combinatorial mismatch by tempering the target distribution with the structured prior $\pi_0^\gamma$ while keeping policy initialization matched. Comparing completed candidates across $\gamma\in\{0.4,0.6,1.0\}$ under identical training recipes, the intermediate anchor $\gamma=0.6$ achieves the highest validation Success@8 ($9.38\%$) while maintaining sparse, plausible repairs ($10.98$ ADD edits); we therefore fix $\gamma=0.6$ before running replicated seeds without consulting test data. In contrast, an identity-uniform diagnostic that matched expected edit counts but removed reaction-specific SFT preferences collapsed toward nearly empty repairs ($0.065$ edits, $0.276\%$ validation success), confirming that GFlowNet requires biological reaction-level preferences rather than mere cardinality constraints.

\begin{table}[h]
\centering
\small
\setlength{\tabcolsep}{5pt}
\renewcommand{\arraystretch}{1.1}
\caption{
\textbf{SFT-anchor selection for the GFlowNet baseline.}
Anchored variants use the same final training recipe and complete all
864 updates.
The unanchored run is incomplete and is included only to show object-space
cardinality degeneration.
Success values are validation Success@8 (\%).
}
\label{tab:gfn_anchor}
\begin{tabular}{@{}lrrrr@{}}
\toprule
Anchor $\gamma$
& Best
& Best step
& Last
& Last ADD \\
\midrule
$0$ (no anchor)
& 7.91
& 15
& 5.80
& 7367.57 \\
$0.4$
& 8.46
& 745
& 7.64
& 23.05 \\
$\mathbf{0.6}$
& $\mathbf{9.38}$
& 864
& $\mathbf{9.38}$
& 10.98 \\
$1.0$
& 8.00
& 715
& 7.27
& 6.05 \\
\bottomrule
\end{tabular}
\end{table}

\textbf{Implementation details.}
To isolate object-level exploration under matched task optimization, we train the baseline using the joint objective $\mathcal L_{\mathrm{GFlowNet}} = \mathcal L_{\mathrm{GRPO}} + \lambda_{\mathrm{LPV}}\mathcal L_{\mathrm{LPV}}$, combining standard verifier-guided GRPO on group-normalized advantages with LPV distribution matching on raw rewards. GFlowNet uses no quotient-space mechanisms during training, with repair cores computed strictly post hoc for evaluation. The baseline shares the exact policy backbone, candidate action masks, FBA verifier, and frozen SFT initialization ($\pi_0$) with all other RL methods. We use rollout size $L=32$, anchor strength $\gamma=0.6$, reward temperature $\tau=1$, LPV weight $\lambda_{\mathrm{LPV}}=0.02$, and PPO clipping threshold $0.2$. Each update batches 32 organisms with 32 fresh on-policy samples for a fixed budget of 864 updates (matching three epochs) via AdamW (learning rate $2e{-6}$, weight decay $1e{-4}$, global gradient clipping at 1.0). Across seeds 42, 43, and 44, test metrics are reported from the final checkpoint matching Table~\ref{tab:main_results}.

\section{Experimental Details}
\label{app:experiment}

Our training follows a three-stage \emph{Pretrain--SFT--RL} pipeline. Since GEM repair is evaluated only through non-differentiable terminal feedback, direct RL from random initialization is ineffective in the large and sparse reaction space. We therefore first pretrain the model to learn GEM structure and reaction compatibility, then use phenotype-supervised SFT to learn a repair prior, and finally optimize complete edit sets with verifier-grounded RL. These stages correspond to \emph{structural literacy}, \emph{repair prior learning}, and \emph{joint consequence optimization}.

\subsection{Base model architecture.} 
\label{app:base_model}
All post-training methods use the same pretrained backbone (Fig.~\ref{fig:model_arch}). Each GEM is represented as a metabolite--reaction bipartite graph with 19-dimensional node features and 3-dimensional stoichiometric edge features, and encoded by a 12-layer edge-aware GINE~\citep{hu2020strategiespretraininggraphneural} with hidden dimension 768, residual connections, LayerNorm, and dropout. Reaction and metabolite nodes use type embeddings together with shared biochemical identity embeddings, while separate attention pooling over the two node types produces a global GEM representation. Candidate ADD reactions are encoded by a universal reaction encoder over 28,301 reactions, which combines a shared reaction-ID embedding, pooled substrate/product metabolite embeddings, and reaction-level structural features into the same 768-dimensional latent space. 
Organism-level proteomic context is extracted by a frozen ESM2-t12-35M~\footnote{\url{https://huggingface.co/facebook/esm2_t12_35M_UR50D}} 
encoder, whose 480-dimensional protein embeddings are deterministically
reduced to 256 dimensions, attention-pooled over at most 512 proteins, and
mapped to 768 dimensions by a learnable gated MLP before fusion with the
global GEM representation.
Based on this shared representation, separate ADD and REMOVE heads score absent and existing reactions, respectively: the ADD head combines graph--reaction matching, pairwise interactions, and local metabolite compatibility, whereas the REMOVE head scores interactions between the global GEM state and reaction-node embeddings. During post-training, the ESM-2 backbone remains frozen, while the GINE encoder, universal reaction encoder, proteome conditioner, and ADD/REMOVE scoring heads are jointly fine-tuned.

\begin{figure*}[h]
    \centering
    \includegraphics[width=\textwidth]{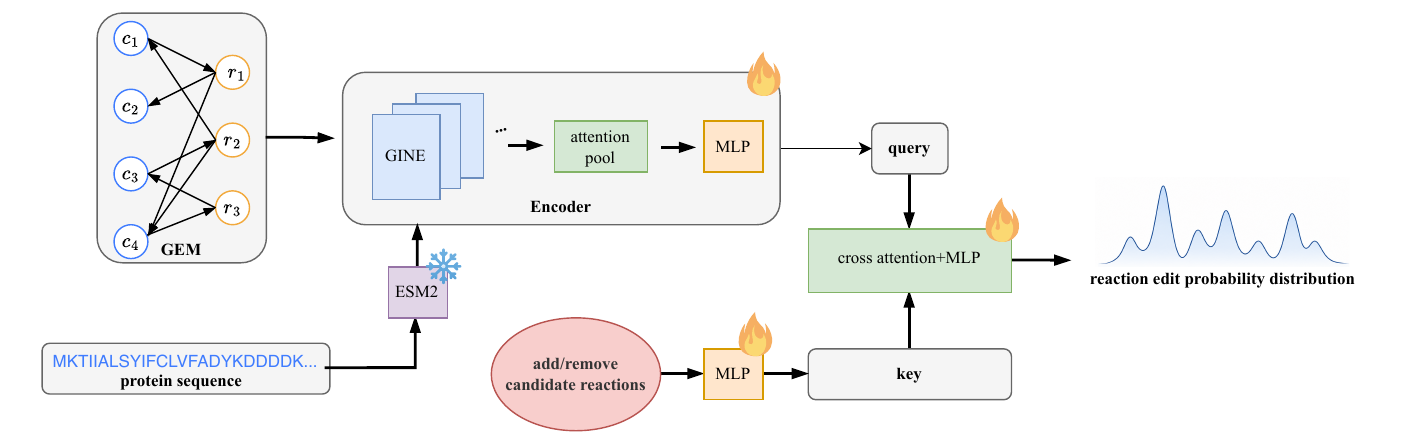}
    \caption{
    \textbf{Base model architecture.}
    }
    \label{fig:model_arch}
\end{figure*}

\subsection{Pretraining Details}
\label{app:pretraining_details}

\textbf{Dataset construction.}
Of the 3,978 candidate GEMs, 143 were excluded by the pretraining leakage-control criteria, leaving 3,806 training GEMs and 29 held-out validation GEMs.
We pretrain the GEM backbone through single-reaction denoising on these 3,806 clean GEMs.
Each training example contains exactly one corruption: either a valid reaction is removed and recovered by the ADD head, or an absent reaction is inserted and identified by the REMOVE head.
The two tasks are sampled with equal probability.
To prevent synthetic shortcuts, injected reactions must be absent from the original GEM, avoid duplicate stoichiometric signatures, and connect metabolites already represented in the graph; artifact-prone metadata features are masked during pretraining.
Each epoch contains 125,000 examples, comprising 100,000 dynamically generated corruptions and 25,000 fixed atomic examples, for a total of five million training exposures over 40 epochs.
This objective trains reaction-level compatibility without imposing a synthetic prior over downstream repair-set cardinality or co-occurrence.

\textbf{Model training and evaluation.}
Each pre-training example contains exactly one corrupted reaction, corresponding to either an ADD or REMOVE target.
For action type $a\in\{\mathrm{ADD},\mathrm{REMOVE}\}$, let $\mathcal C_i^a$ denote the legal candidate set for example $i$ and $y_i^a$ its target reaction.
ADD uses up to 4,096 sampled absent reactions, always including the positive target and frequency-based hard negatives, whereas REMOVE scores all reactions currently present in the corrupted GEM.
Given logits $s_{ir}^a$, the reaction-identification loss is
\(
\mathcal L_{\mathrm{CE}}^a
=
-\frac{1}{|\mathcal B_a|}
\sum_{i\in\mathcal B_a}
\log
\frac{\exp(s_{iy_i^a}^a)}
{\sum_{r\in\mathcal C_i^a}\exp(s_{ir}^a)}.
\)
We further emphasize difficult confounders using the eight highest-scoring incorrect candidates,
$\mathcal H_i^a=\operatorname{Top8}\{s_{ir}^a:r\in\mathcal C_i^a\setminus\{y_i^a\}\}$,
with pairwise ranking loss
\(
\mathcal L_{\mathrm{rank}}^a
=
\frac{1}{|\mathcal B_a|}
\sum_{i\in\mathcal B_a}
\frac{1}{|\mathcal H_i^a|}
\sum_{r\in\mathcal H_i^a}
\operatorname{softplus}
\left(s_{ir}^a-s_{iy_i^a}^a\right).
\)
The complete objective is
\[
\mathcal L_{\mathrm{pretrain}}
=
\mathcal L_{\mathrm{CE}}^{\mathrm{ADD}}
+
\mathcal L_{\mathrm{CE}}^{\mathrm{REMOVE}}
+
0.2
\left(
\mathcal L_{\mathrm{rank}}^{\mathrm{ADD}}
+
\mathcal L_{\mathrm{rank}}^{\mathrm{REMOVE}}
\right).
\]
Models are trained with AdamW using a learning rate of $1e{-4}$, weight decay $1e{-5}$, global batch size 384 across eight GPUs\footnote{All experiments across pretraining, SFT, RL, and evaluation in this work are conducted on NVIDIA RTX PRO 6000 Blackwell GPUs.}, BF16 mixed precision, and gradient clipping at 1.0.
Validation uses 580 fixed corruptions generated from 29 held-out GEMs, including multi-edit corruptions, and scores the full reaction vocabulary. Checkpoint selection uses the macro-average of ADD and REMOVE recall on the operation-specific single-target subsets.

\begin{figure}[h]
    \centering
    \includegraphics[width=0.6\linewidth]{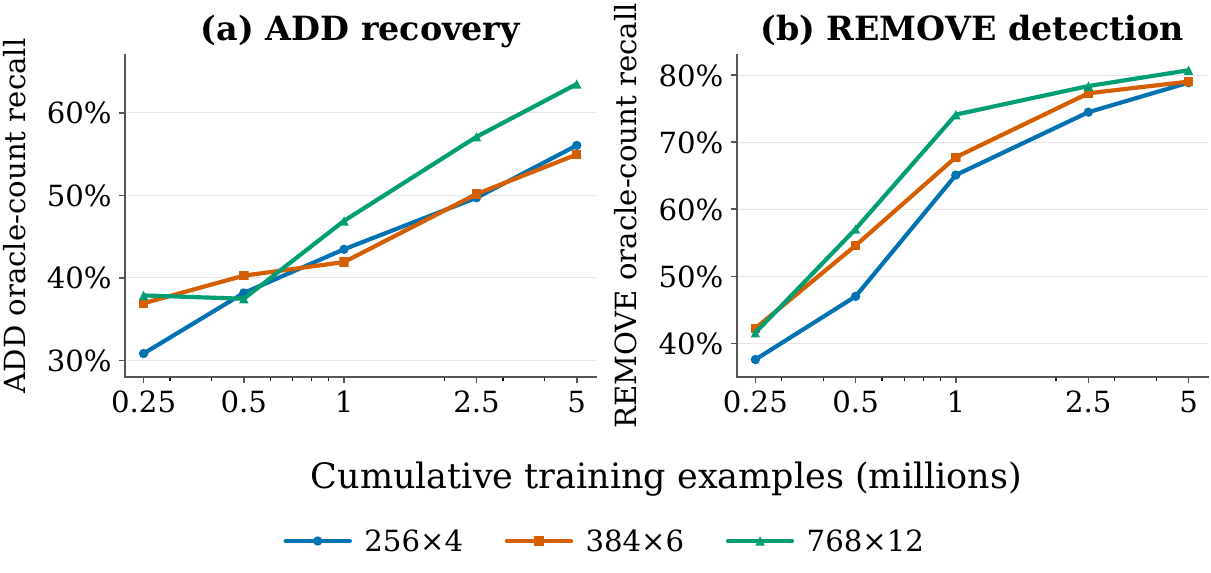
    }
    \caption{
    \textbf{Scaling behavior of GEM pretraining.}
    }
    \label{fig:pretrain_scaling}
\end{figure}

\textbf{Scaling of model size and data size.}
We study scaling along both model capacity and cumulative pretraining exposure.
Specifically, we compare 256(hidden size)$\times$4(layer), 384$\times$6, and 768$\times$12 GINE backbones containing 11.16M, 18.98M, and 58.45M parameters, respectively, at 0.25M, 0.5M, 1M, 2.5M, and 5M training examples.
Scaling is evaluated using oracle-count edit recall on 580 fixed validation examples generated from 29 held-out GEMs. Unlike the single-reaction corruptions used for training, validation examples may contain multiple target edits. For each example and operation, reactions are ranked by model score and the top $k_i$ predictions are retained, where $k_i$ equals the number of ground-truth targets for that operation. Recall is then micro-averaged over target reactions separately for ADD and REMOVE. This metric isolates reaction-ranking quality from threshold calibration and cardinality prediction. As shown in Figure~\ref{fig:pretrain_scaling}, both additional data and model capacity improve recovery, with the largest gains appearing for ADD. At five million examples, the $768\times12$ model reaches 63.49\% ADD recall and 80.73\% REMOVE recall. We therefore use the $768\times12$ checkpoint as the shared pretrained backbone for all downstream experiments.

\subsection{Supervised Fine-tuning Details}
\label{app:sft}
Supervised fine-tuning bridges single-reaction pretraining and one-shot RL by establishing a directed biochemical repair prior. Because directly predicting complete repairs conflates reaction prioritization with combinatorial set sizing under underdetermined phenotypes, we divide SFT into two sequential stages. SFT-A learns local reaction rankings via multi-positive partial-label supervision, while SFT-B calibrates state-conditioned thresholds to decode complete edit sets in one shot, stably initializing downstream policy optimization.

\subsubsection{SFT-A: Partial-Label Reaction-Edit Ranking}
\label{app:sfta}

\textbf{Algorithm design.}
SFT-A bridges pretrained metabolic representations and downstream complete-set decoding by learning which reaction edits are promising from partially edited states. Because phenotype supervision is partially identifying and admits multiple valid repair paths, treating a single teacher edit as unique penalizes alternative valid choices, while treating unlisted reactions as negatives overstates supervisory completeness. We therefore formulate SFT-A as multi-positive partial-label ranking. Evaluated by FBA verification on partially edited states, candidate edits are categorized into strictly improving actions (immediate phenotype gain, forming set $P^+$), required neutral actions (zero immediate gain but necessary for multi-step completions, forming set $P^0$), and regressive actions (degrading agreement, masked from gradients).
Let $C^{-}$ denote the regressive actions and define the effective candidate set
$\widetilde C = C \setminus C^{-}$; all losses below are evaluated only over
$\widetilde C$.
By construction, each SFT-A state contains at least one certified positive
($P\neq\varnothing$). When $P^+=\varnothing$, the strict-tier term is omitted
($L_{\mathrm{strict}}=0$), and when $H=\varnothing$, $L_{\mathrm{tail}}=0$.
Improving and neutral edits jointly form the certified positive set $P = P^+ \cup P^0$. For candidate logits $z$ over legal candidates $\mathcal C$, the complete training objective combines partial-label ranking, strict-tier enhancement, hard-tail margin separation against the top unlabeled candidates $\mathcal H \subseteq \mathcal C \setminus P$ ($|\mathcal H| \le 32$), and a logit regularizer:
\begin{align*}
\mathcal L_{\mathrm{SFT\text{-}A}} &= \mathcal L_{\mathrm{union}} + 0.15\,\mathcal L_{\mathrm{strict}} + 0.15\,\mathcal L_{\mathrm{tail}} + 10^{-5}\mathcal L_{\mathrm{logit}}, \label{eq:sfta_total} \\
\text{where} \quad 
\mathcal L_{\mathrm{union}} &= \operatorname{LSE}_{b \in  \widetilde C} z_b - \operatorname{LSE}_{a \in P} z_a, \nonumber \\
\mathcal L_{\mathrm{strict}} &= \operatorname{LSE}_{b \in \widetilde C} z_b - \operatorname{LSE}_{a \in P^+} z_a, \nonumber \\
\mathcal L_{\mathrm{tail}} &= \frac{1}{|\mathcal H|} \sum_{b \in \mathcal H} \operatorname{softplus}\bigl(z_b - \max_{a \in P} z_a + 0.5\bigr). \nonumber
\end{align*}
Here, $\operatorname{LSE}_{i \in \mathcal S} z_i = \log \sum_{i \in \mathcal S} \exp(z_i)$ denotes the LogSumExp operator, and $\mathcal L_{\mathrm{logit}} = \frac{1}{2|\widetilde C|} \sum_{b \in \widetilde C} z_b^2$ regularizes finite logits to prevent probability saturation. This multi-positive objective equips downstream policy learning with a directional biochemical prior toward functionally viable reactions while strictly preserving alternative valid repair pathways.

\textbf{Engineering implementation.}
The reaction-edit actor initializes from our pretraining checkpoint using a 12-layer, 768-dimensional GINE encoder over bipartite reaction--metabolite graphs, containing 59.3M trainable parameters out of 92.7M total. 
Organism context is supplied by frozen ESM2 representations (Fig.\ref{fig:model_arch}):
up to 512 uniformly sampled proteins per proteome are represented by fixed
256-dimensional truncations of the normalized 480-dimensional ESM2 embeddings,
aggregated by trainable attention, and mapped into the 768-dimensional graph
embedding via a zero-initialized gated MLP residual, preserving pretrained
reaction logits at initialization without learning explicit protein-to-reaction
alignments.
For scalable ADD training, each local batch constructs a shared candidate panel of nominal size 2,048 that retains all supervised batch positives and draws up to 75\% of remaining slots from cached topology-based hard candidates (256 per graph) alongside vocabulary negatives; validation scores the complete ADD vocabulary without candidate sampling, while REMOVE scoring evaluates only present reaction nodes. Training uses 10,837 partial states weighted inversely by each organism's state count ($w_s = 1/n_x$) to balance supervisory influence. Optimization spans 8 GPUs with local batch size 4 (effective global batch size 32) for up to five epochs using AdamW (learning rate $3\times 10^{-5}$, weight decay $10^{-4}$, gradient clipping at 1.0) under BF16 autocast with FP32 loss reductions. We evaluate checkpoints on 410 development states and select checkpoints lexicographically by operation-macro, model-equal hit rate (Hit@1). SFT-A strictly outputs reaction candidate rankings, leaving decision thresholds, edit cardinality, and complete repair decoding to the downstream SFT-Soft stage.

\textbf{Key results.}
Table~\ref{tab:sfta-direct} summarizes reaction-retrieval performance. The full Pretrained + ESM2 configuration achieves the strongest accuracy and likelihood fit, lifting operation-macro Hit@1 from $0.016$ on pretrained initialization to $0.207$ while attaining the lowest development loss ($3.248$). Removing pretraining degrades Hit@1 to $0.114$ and inflates loss to $4.524$, while removing ESM2 drops Hit@1 to $0.143$ with loss $3.922$. This advantage is concentrated in ADD retrieval, reaching $0.264$ Hit@1 against $0.068$ without pretraining and $0.125$ without ESM2. Both metabolic pretraining and proteome context thus prove essential for prioritizing viable reactions and minimizing partial-label ranking loss.

\begin{table}[t]
\centering
\small
\setlength{\tabcolsep}{5pt}
\caption{
\textbf{Direct SFT-A development results.}
Hit rates are operation-macro, model-equal retrieval scores on certified
partial labels (seed 42).
}
\label{tab:sfta-direct}
\begin{tabular}{@{}lccc@{}}
\toprule
Configuration & Hit@1$\uparrow$ & dev-select loss$\downarrow$  \\
\midrule
Pretrained + ESM2
& \textbf{0.207} & \textbf{3.248}  \\
Random + ESM2
& 0.114 & 4.524  \\
Pretrained, no ESM2
& 0.143 & 3.922  \\
\bottomrule
\end{tabular}
\end{table}

\subsubsection{SFT-B: Complete Set-level Reaction-Edit Supervision}
\label{app:sftb}

\textbf{Algorithm design.}
SFT-B converts the single action ranking actor learned in SFT-A into a one shot predictor of the complete remaining repair set. Suppressing state and operation indices, let $\mathcal C$ be the legal candidate set, $Y\subseteq\mathcal C$ the target edits from one certified completion, $N=\mathcal C\setminus Y$ the non target candidates. Unlike SFT-A, SFT-B treats each teacher completion as a complete binary target. Alternative valid completions remain separate training examples rather than being merged. Both variants share a class balanced binary cross entropy loss and a hard negative ranking loss. Let $v$ denote the candidate score vector, and let $H_h\subseteq N$ be the top $h$ negatives. With $\operatorname{softplus}(u)=\log(1+e^u)$,
\begin{align*}
&\mathcal L_{\mathrm{BCE}}(v)
=
\frac{1}{2}
\left[
\frac{1}{|Y|}\sum_{a\in Y}\operatorname{softplus}(-v_a)
+
\frac{1}{|N|}\sum_{b\in N}\operatorname{softplus}(v_b)
\right],\\
&\mathcal L_{\mathrm{rank}}(v)
=
\frac{1}{|Y||H_h|}
\sum_{a\in Y}\sum_{b\in H_h}
\operatorname{softplus}(v_b-v_a).
\end{align*}
The balanced BCE prevents the large negative vocabulary from dominating the objective, while the ranking loss separates target edits from competitive negatives. When one class is empty, BCE uses only the available class.
Likewise, boundary losses are evaluated only when both their positive and
negative aggregates are nonempty; otherwise the corresponding boundary term is
set to zero.
Ranking is applied only when both classes are present. Hard-SFT evaluates these losses on raw logits $z$, while Soft-SFT evaluates them on threshold adjusted scores $u$.

\emph{Hard-SFT} factorizes set prediction into reaction ranking and discrete cardinality prediction. The actor outputs reaction logits $z$ and a categorical count distribution $q(k)$. Decoding predicts $\widehat k = \arg\max_k q(k)$ and returns $\widehat Y = \operatorname{TopK}(z, \widehat k)$. The count head supports up to $K_{\max}$ edits ($13$ for ADD, $1$ for REMOVE), trained with cross-entropy $\mathcal L_{\mathrm{count}} = -\log q(m)$ and normalized expected error $\mathcal L_{\mathrm{dist}} = |\sum_k kq(k) - m| / K_{\max}$. Hard-SFT also enforces separation around the decision boundary at temperature $t=0.25$ via $\mathcal L_{\mathrm{bdry}} = t\operatorname{softplus}((b^- - b^+ + 1)/t)$, where $b^+ = -t\log(|Y|^{-1}\sum_{a\in Y}e^{-z_a/t})$ and $b^- = t\log(|H_{32}|^{-1}\sum_{b\in H_{32}}e^{z_b/t})$ smoothly aggregate target and top-32 negative logits. The complete objective is
$$
\mathcal L_{\mathrm{Hard}}
=
\mathcal L_{\mathrm{BCE}}(z)
+
0.40\mathcal L_{\mathrm{rank}}(z)
+
0.20\mathcal L_{\mathrm{bdry}}
+
0.25\mathcal L_{\mathrm{count}}
+
0.50\mathcal L_{\mathrm{dist}}.
$$
Hard-SFT therefore learns reaction identity and repair size through separate outputs. This gives explicit cardinality control, but count errors directly cause under or over editing and the finite support imposes a fixed upper bound on repair size.

\emph{Soft-SFT} replaces the discrete count head with a state conditioned inclusion threshold. The actor outputs reaction logits $z_a$, while an auxiliary head predicts an operation specific threshold $\tau$. The adjusted score is $u_a=z_a-\tau$, giving inclusion probability $p_a=\sigma(u_a)$ and deterministic prediction $\widehat Y=\{a\in\mathcal C:u_a\ge0\}$. Repair size therefore emerges from the candidate scores and the learned threshold rather than from a separate count prediction. Soft-SFT applies the shared BCE and ranking losses to $u$, with $H_8$ used for ranking. Its expected cardinality is $\bar k=\sum_{a\in\mathcal C}p_a$. Using smooth $L_1$ loss $\rho$, cardinality is supervised by $\mathcal L_{\mathrm{logcount}}=\rho(\log(1+\bar k)-\log(1+m))$ and $\mathcal L_{\mathrm{relative}}=\rho((\bar k-m)/(m+1))$. The former stabilizes supervision across different set sizes, while the latter remains sensitive to small targets. Soft-SFT also directly trains the zero inclusion boundary using
\(
\mathcal L_{\mathrm{bdry}}
=
\frac{1}{2}
[
|Y|^{-1}
\sum_{a\in Y}
\operatorname{softplus}(0.5-u_a)
+
|H_{64}|^{-1}
\sum_{b\in H_{64}}
\operatorname{softplus}(u_b+0.5)
].
\)
The complete objective is
$$
\mathcal L_{\mathrm{Soft}}
=
\mathcal L_{\mathrm{BCE}}(u)
+
0.20\mathcal L_{\mathrm{rank}}(u)
+
0.15\mathcal L_{\mathrm{bdry}}
+
0.50\mathcal L_{\mathrm{logcount}}
+
0.02\mathcal L_{\mathrm{relative}}.
$$
By coupling reaction scoring and repair sizing within continuous scores $u$, Soft-SFT removes the discrete cardinality limit and directly produces the Bernoulli inclusion probabilities for downstream one-shot RL.

\textbf{Engineering implementation.}
Both variants initialize their shared 12-layer, 768-dimensional GNN actor, ADD/REMOVE scoring heads, and frozen ESM2 conditioning from the identical selected SFT-A checkpoint across seeds 42--44, randomly initializing only the newly added count or threshold module. Training views are matched to their decoding assumptions across 1,078 shared GEMs: Hard-SFT trains on 5,874 early-subset states with ADD-loss reweighting, whereas Soft-SFT trains on 40,151 all-subset states. Hard-SFT restricts training to early-subset states ($\vert{}S\vert{}\le 2$) to prevent its categorical count head from developing a small-cardinality prior collapse induced by late-stage combinatorial subsets, aligning the training cardinality distribution with initial-state deployment. Conversely, Soft-SFT leverages all-subset states to train its continuous threshold across diverse residual repair scales. Operation-specific Soft thresholds are initialized by maximizing edit F1 on initial training examples and recalibrated per-epoch on a fixed training subset without consulting evaluation splits. Both methods score the full ADD vocabulary under per-GEM legality masks, restrict REMOVE scoring to present reactions, and evaluate or deploy directly from original unedited GEMs. Optimization runs on eight GPUs with effective batch size 128 using AdamW (weight decay $1e{-4}$, gradient clipping at 1.0) with layer-wise learning rates of $5e{-6}$ for the backbone, $2e{-5}$ for actor heads, and $1e{-4}$ for the count or threshold module. 

\textbf{Key results.}
For downstream evaluation, checkpoints across seeds are selected on the development split by joint edit F1, with final held-out test evaluations presented in Table~\ref{tab:main_results}. Crucially, a matched-batch diagnostic before any RL update confirms that Soft-SFT provides a far more informative initialization for policy optimization, yielding nonzero within-group reward variance for $51.6\%\pm4.9\%$ of sampled organism groups compared with $32.0\%\pm5.1\%$ for Hard-SFT. This verifies that the adaptive-threshold policy exposes a richer verifier signal to kick-start policy exploration.

\begin{table}[t]
\centering
\small
\setlength{\tabcolsep}{4.5pt}
\renewcommand{\arraystretch}{1.08}
\caption{
\textbf{Shared RL configuration.}
}
\label{tab:rl-settings}
\begin{tabularx}{\linewidth}{@{}l>{\raggedright\arraybackslash}X@{}}
\toprule
Parameter & Setting \\
\midrule
Actor initialization & Soft-SFT \\
Total / trainable parameters & 92.85M / 59.46M \\
Training / validation GEMs & 9,196 / 1,087 \\
RL seeds & 42, 43, 44 \\
Training devices & 8 GPU ranks \\
GEMs per update & 32 \\
Rollouts per GEM & 32 \\
Training budget & 3 nominal epochs, 864 updates \\
Actor updates per rollout batch & 1 \\
Optimizer & AdamW \\
Learning rate & $2e{-6}$ \\
Weight decay & $1e{-4}$ \\
Gradient norm clipping & 1.0 \\
GRPO clipping & $\epsilon=0.2$ \\
Advantage stabilizer & $1e{-6}$ \\
Bernoulli sampling temperature & 1.0 \\
FBA growth threshold & $1e{-3}$ \\
Validation cadence & Every 5 updates and at update 864 \\
\bottomrule
\end{tabularx}
\end{table}

\subsection{Reinforcement Learning Details}
\label{app:rl}
\textbf{Shared RL training setup.}
Alongside the GFlowNet (Appendix~\ref{app:gflownet}), Vanilla GRPO, MaxEnt, and QuotientPO (Appendix~\ref{app:algorithm}) share a unified one-shot policy optimization framework with detailed configurations summarized in Table~\ref{tab:rl-settings}. Initialized from the selected Soft-SFT checkpoint, the policy samples a complete repair set in a single forward pass via independent Bernoulli decisions over legal edits using fine-tuned inclusion thresholds, dispensing with value critics, reference KL penalties, or discrete cardinality heads. Vanilla GRPO optimizes this verifier-guided objective directly, serving as the common backbone for exploration ablations.

\textbf{Action-entropy baseline.}
MaxEnt augments GRPO with the exact Shannon entropy of the factorized Bernoulli policy, summed analytically in nats over all legal ADD and REMOVE decisions without vocabulary normalization, Monte Carlo estimation, or additional verifier overhead. Because summing unnormalized entropy across a large reaction vocabulary risks simultaneously elevating many individually unlikely actions and causing severe repair cardinality explosion, its coefficient requires explicit calibration. In development experiments, coefficients of $1e{-3}$ and $5e^{-4}$ triggered severe ADD cardinality inflation, whereas $1e{-4}$ remained stable but yielded insufficient exploration. We therefore fixed the intermediate coefficient of $2e{-4}$ across all seeds based strictly on validation behavior, ensuring that the hyperparameter choice was locked before test evaluation without consulting held-out splits.

\textbf{Runtime overhead.} We benchmark one complete optimizer update on a single node with 8$\times$ NVIDIA RTX PRO 6000 Blackwell GPUs, using the same RL training settings for all methods. All runs use the same HiGHS backend, 64 single-threaded FBA workers, FP32 execution, and no validation or external logging. QuotientPO requires 23.22 s per update versus 21.32 s for Vanilla GRPO, corresponding to an 8.9\% overhead. Most additional cost comes from verifier-side processing, including canonicalization, QPI and kernel scoring.

\begin{table}[h]
\centering
\small
\caption{\textbf{Per-update runtime} under identical hardware and rollout settings.}
\label{tab:runtime}
\begin{tabular}{lccc}
\toprule
Method & Update (s)  & FBA drain (s) & Overhead \\
\midrule
Vanilla GRPO & 21.323 & 2.229 & -- \\
QuotientPO   & 23.218  & 4.186 & +8.9\% \\
\bottomrule
\end{tabular}
\end{table}

\subsection{Evaluation Details}
\label{app:eval}

All methods are evaluated on 2,212 held-out initial GEMs. For each GEM, we generate 32 stochastic repair sets in FP32 with dropout disabled under fixed per-organism seeds, alongside a separate greedy prediction. Canonical repair cores are uniformly recomputed for all baselines using the deterministic reduction described in Appendix~\ref{app:algorithm}. Success@$K$ is computed over nested prefixes of the 32 rollouts, greedy success is evaluated separately, and diversity metrics are calculated over these frozen outputs. All metrics are computed per GEM and macro-averaged, reporting the mean and sample standard deviation across three training seeds.

\section{Additional Experimental Results}

\subsection{Robustness across reconstruction templates.}
QuotientPO consistently improves Success@32 over Vanilla GRPO across all four CarveMe templates, with gains of +1.96, +2.78, +1.71, and +2.32 percentage points on default, Gram-positive, Gram-negative, and archaeal models, respectively (Table~\ref{tab:template_breakdown}). Successful-core coverage also improves in every template, indicating that the overall gain is not driven by any single reconstruction setting, including the archaeal subset.

\begin{table*}[t]
\centering
\small
\setlength{\tabcolsep}{1.8pt}
\renewcommand{\arraystretch}{1.15}
\caption{
\textbf{Performance across GEM reconstruction templates.}
Results are mean $\pm$ standard deviation across three seeds.
}
\label{tab:template_breakdown}
\begin{tabular}{llcccccc}
\toprule
Template & Method
& Success@1
& Success@8
& Success@32
& Greedy
& Improved
& Success Cores@32 \\
\midrule

\multicolumn{8}{l}{\textit{Default} ($N=697$)} \\
& Vanilla GRPO
& 7.37$\pm$0.58
& 12.72$\pm$0.79
& 16.79$\pm$1.08
& \textbf{7.84$\pm$0.36}
& 28.83$\pm$1.08
& 0.516$\pm$0.003 \\
& MaxEnt
& 5.69$\pm$0.96
& 11.81$\pm$1.58
& 15.93$\pm$1.28
& 7.32$\pm$1.52
& 28.04$\pm$2.33
& 0.462$\pm$0.088 \\
& GFlowNet
& 1.77$\pm$0.22
& 6.07$\pm$0.54
& 10.38$\pm$0.33
& 2.54$\pm$0.30
& 11.48$\pm$0.70
& 0.265$\pm$0.017 \\
& \textbf{QuotientPO}
& \textbf{7.89$\pm$0.50}
& \textbf{14.25$\pm$0.22}
& \textbf{18.75$\pm$0.82}
& 7.80$\pm$0.33
& \textbf{30.01$\pm$1.19}
& \textbf{0.582$\pm$0.006} \\
\midrule

\multicolumn{8}{l}{\textit{Gram-positive} ($N=491$)} \\
& Vanilla GRPO
& \textbf{7.13$\pm$0.20}
& 13.04$\pm$0.35
& 16.70$\pm$2.12
& 8.15$\pm$0.41
& 30.14$\pm$0.72
& 0.573$\pm$0.019 \\
& MaxEnt
& 5.16$\pm$1.73
& 10.93$\pm$1.01
& 16.02$\pm$0.85
& 7.13$\pm$1.43
& 30.17$\pm$1.06
& 0.483$\pm$0.087 \\
& GFlowNet
& 2.58$\pm$0.31
& 7.06$\pm$0.31
& 10.93$\pm$0.12
& 2.72$\pm$0.82
& 14.33$\pm$0.64
& 0.297$\pm$0.017 \\
& QuotientPO
& 6.93$\pm$0.20
& \textbf{14.39$\pm$0.62}
& \textbf{19.48$\pm$0.77}
& \textbf{9.03$\pm$0.77}
& \textbf{31.44$\pm$1.08}
& \textbf{0.604$\pm$0.003} \\
\midrule

\multicolumn{8}{l}{\textit{Gram-negative} ($N=507$)} \\
& Vanilla GRPO
& 7.43$\pm$0.89
& 14.79$\pm$0.99
& 19.26$\pm$0.99
& 8.02$\pm$0.41
& 30.21$\pm$0.49
& 0.507$\pm$0.027 \\
& MaxEnt
& 6.11$\pm$0.59
& 13.48$\pm$1.40
& 17.82$\pm$2.43
& 6.77$\pm$0.97
& 29.50$\pm$1.88
& 0.550$\pm$0.027 \\
& GFlowNet
& 3.81$\pm$0.30
& 8.22$\pm$0.11
& 12.10$\pm$0.30
& 2.89$\pm$0.11
& 13.89$\pm$0.46
& 0.377$\pm$0.008 \\
& QuotientPO
& \textbf{8.15$\pm$0.30}
& 1\textbf{5.98$\pm$0.20}
& \textbf{20.97$\pm$1.01}
& \textbf{8.55$\pm$0.23}
& \textbf{32.08$\pm$0.56}
& \textbf{0.617$\pm$0.032} \\
\midrule

\multicolumn{8}{l}{\textit{Archaeal} ($N=517$)} \\
& Vanilla GRPO
& 7.80$\pm$0.30
& 14.96$\pm$2.34
& 19.34$\pm$2.42
& 7.93$\pm$0.70
& 32.11$\pm$0.23
& 0.409$\pm$0.057 \\
& MaxEnt
& 6.96$\pm$0.89
& 14.77$\pm$1.24
& 19.08$\pm$1.45
& 7.93$\pm$1.01
& 30.17$\pm$2.76
& 0.427$\pm$0.022 \\
& GFlowNet
& 5.22$\pm$0.19
& 12.77$\pm$0.51
& 19.79$\pm$0.40
& 4.84$\pm$0.39
& 21.08$\pm$0.15
& 0.433$\pm$0.017 \\
& QuotientPO
& \textbf{8.00$\pm$0.30}
& \textbf{16.89$\pm$0.81}
& \textbf{21.66$\pm$0.39}
& \textbf{8.77$\pm$0.40}
& \textbf{33.38$\pm$0.36}
& \textbf{0.476$\pm$0.030} \\
\bottomrule
\end{tabular}
\end{table*}

\subsection{Solution Space Geometry: SFT versus Policy Optimization}
\label{app:solution_space}



To analyze exploration geometry, each raw repair is represented by its signed reaction edits under pairwise Jaccard distance $d(A, B) = 1 - |A \cap B| / |A \cup B|$ ($d(\varnothing, \varnothing) = 0$), and 32 stochastic rollouts from SFT-Soft and QuotientPO are jointly embedded with the unedited GEM and a certified MILP teacher reference via 2D metric multidimensional scaling (MDS)\footnote{\url{https://github.com/scikit-learn/scikit-learn}}. Identical edits share averaged coordinates on aligned axes, with all reported teacher distances evaluated in the ambient Jaccard space. Across six representative organisms (Figure~\ref{fig:sft_rl_cases}), SFT-Soft rollouts yield few successful completions and cluster tightly around the certified teacher reference. In contrast, QuotientPO substantially increases repair success and disperses into multiple disconnected clusters, moving beyond the supervised demonstration to recover valid repairs across separated regions of the edit space.

\subsection{Extended Solution-Space Visualizations Across RL Baselines}
Figure~\ref{fig:rl_cases} compares exploration behavior across Vanilla GRPO, MaxEnt, GFlowNet, and QuotientPO projected into both raw edit and quotient mechanism spaces over six representative GEMs. Standard RL baselines and action-entropy exploration suffer from severe mode collapse, frequently rediscovering identical redundant mechanisms despite surface-level action variations. In contrast, QuotientPO consistently disperses across the quotient landscape, discovering substantially more distinct, successful canonical repair cores across diverse biological targets.

\subsection{Success-Controlled Mechanism Coverage}
\label{app:success_controlled}

Successful-core yield conflates solve rates with conditional mechanism diversity. To isolate the latter, we measure the expected distinct canonical cores under an equalized budget of $r$ successful repairs, restricted to paired cohorts where both methods achieve at least $r$ successes among 32 rollouts. 
For $M$ successful repairs with core $a$ occurring $s_a$ times, the expected distinct core count across $r$ draws without replacement is
\[
D_r = \sum_a \left[ 1 - \frac{\binom{M-s_a}{r}}{\binom{M}{r}} \right].
\]
Table~\ref{tab:success_controlled} shows that QuotientPO consistently outperforms Vanilla GRPO across all budgets, with the margin widening from $+0.035$ at $r=2$ to $+0.246$ at $r=16$. This confirms that the diversity gain does not merely reflect a higher solve rate. We truncate at $r=16$ because requiring all 32 rollouts to succeed leaves only 10/8/4 paired GEMs across seeds 42/43/44.

\begin{table}[t]
\centering
\footnotesize
\setlength{\tabcolsep}{4.0pt}
\renewcommand{\arraystretch}{1.15}
\caption{\textbf{Success-controlled mechanism coverage.} Expected distinct canonical cores ($D_r$) evaluated across $r$ successful draws on paired cohorts.}
\label{tab:success_controlled}
\begin{tabular}{@{}ccccc@{}}
\toprule
$r$
& \begin{tabular}[c]{@{}c@{}}Paired GEMs\\(42/43/44)\end{tabular}
& Vanilla GRPO
& QuotientPO \\
\midrule
2  & 295/269/323 & $1.425 \pm 0.019$ & $\mathbf{1.461 \pm 0.020}$   \\
4  & 247/228/258 & $1.960 \pm 0.037$ & $\mathbf{2.050 \pm 0.056}$  \\
8  & 209/186/206 & $2.659 \pm 0.076$ & $\mathbf{2.782 \pm 0.114}$ \\
16 & 147/123/130 & $3.256 \pm 0.033$ & $\mathbf{3.502 \pm 0.203}$  \\
\bottomrule
\end{tabular}
\end{table}

\subsection{Finite-Sample Credit Resolution Analysis}
\label{app:credit_resolution}

\textbf{Metric definitions.}
We compute the training diagnostics from each rollout group of $N=32$ repairs. Let $n_a$ denote the multiplicity of canonical core $a$ in the group. Under exact R\'enyi credit, two distinct cores receive identical credit if and only if they have the same multiplicity, so we identify ties directly by $n_a=n_b$ rather than by floating-point comparison. Weighting each distinct-core pair by the number of corresponding sampled pairs, $w_{ab}=n_an_b$, the exact-credit tie rate is
\[
\mathrm{TieRate}
=
\frac{\sum_{a<b} w_{ab}\,\mathbb I[n_a=n_b]}
{\sum_{a<b} w_{ab}}.
\]
Importantly, $\mathrm{TieRate}$ is an \emph{exact-count} diagnostic for every
method: it measures how often distinct cores receive identical multiplicity-based
credit, $n_a=n_b$, within rollout groups generated by that policy. Thus, the
Kernel R\'enyi-2 row does not measure ties under the kernel credit itself.
Instead, it quantifies how often exact counting would still fail to distinguish
cores sampled from the kernel-trained policy. The complementary
$\mathrm{ResolvedRate}$ measures the fraction of these exact-count ties that
are separated by the graded kernel credit.

To measure how much structure is hidden within these ties, we recompute a kernel credit from the same full rollout group. With signed-Jaccard similarity $K_{ab}=\exp[-(1-J(a,b))/0.25]$, the local kernel density of core $a$ is
\(
\rho_a=
\frac{(n_a-1)+\sum_{b\neq a}n_bK_{ab}}{N-1},
P_K=\frac{1}{N}\sum_a n_a\rho_a,
\)
giving the normalized kernel cost $c_a^K=2(\rho_a/P_K-1)$. Following the training scale, we compare $\widetilde c_a^K=0.02\,\operatorname{clip}(c_a^K,-8,8)$ and define the kernel-resolved rate as the fraction of exact-credit ties satisfying $|\widetilde c_a^K-\widetilde c_b^K|>0.001$:
\[
\mathrm{ResolvedRate}
=
\frac{\sum_{a<b}w_{ab}\,
\mathbb I[n_a=n_b]\,
\mathbb I[|\widetilde c_a^K-\widetilde c_b^K|>0.001]}
{\sum_{a<b}w_{ab}\,\mathbb I[n_a=n_b]}.
\]
Thus, TieRate measures the coarseness of exact finite-sample credit, whereas ResolvedRate is a conditional diagnostic of graded structure hidden inside exact ties and is not a higher-is-better performance metric. The successful-core statistics in Table~\ref{tab:kernel_renyi} are computed separately from successful stochastic repairs only. For a GEM with $M\ge2$ successful draws, let $s_a$ be the number assigned to successful core $a$. Exact collision and mean pairwise Jaccard distance are
\[
\mathrm{Collision}
=
\frac{\sum_a s_a(s_a-1)}{M(M-1)},
\qquad
\overline d_J
=
\frac{\sum_{a<b}2s_as_b\,[1-J(a,b)]}{M(M-1)}.
\]
Repeated draws retain their empirical multiplicity, so duplicate cores increase collision and contribute zero Jaccard distance rather than being deduplicated. All quantities are first computed per GEM and then macro-averaged over GEMs for which the corresponding statistic is defined. Undefined cases are excluded rather than imputed as zero.


\textbf{Full training trajectory.}
Table~\ref{tab:kernel_renyi_mechanism} traces credit resolution across updates 0 to 864. As training proceeds, 
Exact R\'enyi-2 becomes increasingly coarse under finite rollout groups, with its exact-count tie rate rising from 5.67\% to 10.23\%. Kernel R\'enyi-2 instead reshapes the rollout distribution toward finer exact-count structure, ending at 7.99\%. More importantly, exact multiplicity substantially understates the available geometric signal: at step 864, kernel credit resolves 78.20\% of exact-count ties in Exact R\'enyi-2 and 84.86\% in Vanilla GRPO. These results demonstrate that kernel relaxation promotes a less degenerate rollout geometry and extracts graded credit among otherwise indistinguishable cores, translating directly into superior downstream performance (Table~\ref{tab:quotient_objective_downstream}).

\begin{table*}[h]
\centering
\footnotesize
\setlength{\tabcolsep}{3.6pt}
\renewcommand{\arraystretch}{1.08}

\caption{\textbf{Full training trajectory of finite-sample credit resolution.} Exact-count tie rates measure distinct-core pairs indistinguishable under multiplicity credit within rollout groups; kernel-resolved rates report the fraction of these ties separated by graded kernel credit.}
\label{tab:kernel_renyi_mechanism}

\begin{tabular}{@{}lcccccccc@{}}
\toprule
&
\multicolumn{4}{c}{\textbf{Exact-credit ties (\%)}}
&
\multicolumn{4}{c}{\textbf{Kernel-resolved ties (\%)}}
\\
\cmidrule(lr){2-5}
\cmidrule(lr){6-9}

\textbf{Method}
& \textbf{0}
& \textbf{200}
& \textbf{400}
& \textbf{864}
& \textbf{0}
& \textbf{200}
& \textbf{400}
& \textbf{864}
\\
\midrule

Vanilla GRPO
& 5.67 & 8.19 & 8.31 & 8.59
& 30.61 & 60.55 & 57.50 & 84.86
\\

Discrete Shannon
& 5.67 & 8.60 & 9.99 & 9.72
& 30.61 & 59.23 & 59.77 & 80.20
\\

Exact R\'enyi-2
& 5.67 & 7.34 & 9.06 & 10.23
& 30.61 & 59.71 & 63.71 & 78.20
\\

Kernel R\'enyi-2
& 5.67 & 8.57 & 8.60 & \textbf{7.99}
& 30.61 & 63.81 & 57.23 & 72.78
\\

\bottomrule
\end{tabular}

\vspace{-0.2cm}
\end{table*}

\begin{table}[h]
\centering
\small
\caption{
\textbf{Downstream performance of quotient exploration objectives.}
}
\label{tab:quotient_objective_downstream}
\begin{tabular}{lcc}
\toprule
Method & Success@32 $\uparrow$ & Success Cores@32 $\uparrow$ \\
\midrule
Vanilla GRPO
& 17.93 $\pm$ 1.54
& 0.50 $\pm$ 0.01 \\

Discrete Shannon
& 18.61 $\pm$ 1.09
& 0.52 $\pm$ 0.01 \\
Exact R\'enyi-2
& 18.04 $\pm$ 1.08
& 0.54 $\pm$ 0.03 \\

\textbf{Kernel R\'enyi-2}
& \textbf{20.10 $\pm$ 0.52}
& \textbf{0.57 $\pm$ 0.02} \\
\bottomrule
\end{tabular}
\end{table}

\subsection{Multi-Seed Robustness and Ablation Results}
\label{app:ablation}

Table~\ref{tab:ablation_full} provides the complete evaluations corresponding to Table~\ref{tab:ablations}, confirming consistent improvements for QuotientPO.

\begin{table*}[t]
\centering
\footnotesize
\setlength{\tabcolsep}{3.5pt}
\renewcommand{\arraystretch}{1.10}

\caption{
\textbf{Multi-seed robustness and ablation results.}
Success@32 (\%) are reported as mean $\pm$ standard deviation over three seeds.
}
\label{tab:ablation_full}

\begin{minipage}[t]{0.47\textwidth}
\centering
\textbf{(a) RLOO Base}

\setlength{\tabcolsep}{2pt}        
\renewcommand{\arraystretch}{1.2}
\begin{tabular}{@{}lc@{}}
\toprule
\textbf{Method} & \textbf{Success@32} \\
\midrule
Vanilla
    & 15.75 $_{\pm \text{1.20}}$  \\
MaxEnt
    & 14.29 $_{\pm \text{0.95}}$  \\
GFlowNet
    & 11.57 $_{\pm \text{0.14}}$ \\
\textbf{QuotientPO}
    & \textbf{17.21 $_{\pm \text{0.81}}$}  \\
\bottomrule
\end{tabular}
\end{minipage}
\hfill
\begin{minipage}[t]{0.47\textwidth}
\centering
\textbf{(b) Policy Architecture}

\setlength{\tabcolsep}{2pt}        
\renewcommand{\arraystretch}{1.05}
\begin{tabular}{@{}llc@{}}
\toprule
\textbf{Architecture} & \textbf{RL} & \textbf{Success@32} \\
\midrule
\multirow{2}{*}{Soft}
    & Vanilla
    & 17.93 $_{\pm \text{1.54}}$ \\
    & QuotientPO
    & 20.10 $_{\pm \text{0.52}}$ \\
\midrule
\multirow{2}{*}{Hard}
    & Vanilla
    & 5.92 $_{\pm \text{0.09}}$ \\
    & QuotientPO
    & 6.65 $_{\pm \text{0.23}}$ \\
\bottomrule
\end{tabular}
\end{minipage}

\vspace{0.35cm}

\begin{minipage}[t]{0.47\textwidth}
\centering
\textbf{(c) Design Ablation}

\setlength{\tabcolsep}{2pt}        
\renewcommand{\arraystretch}{1.08}
\begin{tabular}{@{}lc@{}}
\toprule
\textbf{Variant} & \textbf{Success@32} \\
\midrule
\textbf{Default}
    & 20.10 $_{\pm \text{0.52}}$ \\
Raw-space kernel 
 & 18.16 $_{\pm \text{0.61}}$ \\
w/o Quotient term
    & 17.54 $_{\pm \text{1.81}}$\\
w/o QPI term
    & 19.62 $_{\pm \text{0.72}}$ \\
w/o Redundancy penalty
    & 19.94 $_{\pm \text{1.79}}$ \\
\bottomrule
\end{tabular}
\end{minipage}
\hfill
\begin{minipage}[t]{0.47\textwidth}
\centering
\textbf{(d) Hyperparameter Sensitivity}

\setlength{\tabcolsep}{2pt}        
\renewcommand{\arraystretch}{0.9}
\begin{tabular}{@{}llc@{}}
\toprule
\textbf{Param.} & \textbf{Setting} & \textbf{Success@32} \\
\midrule
\textbf{Default}
    & $h=0.25,\ \beta=0.02$
    & 20.10 $_{\pm \text{0.52}}$ \\
\midrule
\multirow{2}{*}{$h$}
    & 0.125
    & 19.41 $_{\pm \text{0.70}}$ \\
    & 0.5
    & 18.57 $_{\pm \text{0.74}}$ \\
\midrule
\multirow{2}{*}{$\beta$}
    & 0.01
    & 19.05 $_{\pm \text{1.12}}$ \\
    & 0.04
    & 20.00 $_{\pm \text{0.74}}$ \\
\bottomrule
\end{tabular}
\end{minipage}

\vspace{-0.2cm}
\end{table*}

\section{Biological Analysis and Case Studies}
\label{app:case_study}

Six additional cases complement the main-text sucrose and nitrate examples, spanning two cases across three progressive coverage tiers. Each comparison evaluates the identical GEM, checkpoint, and seed over $N=32$ stochastic rollouts per method ($V$: Vanilla GRPO; $Q$: QuotientPO). 
For each focal alternative (the second pathway in each case), we track its sample-level representation along the three-stage evaluation pipeline: raw proposals containing the focal module (\(R\)), verified successful raw repairs containing the focal module (\(S\)), and successful canonicalized repairs in which the focal module remains present (\(C\)), satisfying \(R \geq S \geq C\). Importantly, \(R\), \(S\), and \(C\) count sample occurrences rather than distinct canonical cores; the total number of distinct successful canonical cores is reported separately as \(K_{\mathrm{core}}\). The three coverage tiers characterize the stage at which Vanilla fails to retain the focal mechanism while QuotientPO does (\(C_Q>0\)): Tier 1 (retention gap: \(S_V>0, C_V=0\)) captures focal alternatives that already occur in successful Vanilla proposals but are eliminated during canonicalization; Tier 2 (viability gap: \(R_V>0, S_V=0\)) captures focal modules proposed by Vanilla only within unsuccessful repairs; and Tier 3 (discovery gap: \(R_V=0\)) captures focal modules absent from Vanilla raw proposals.

\section{Limitations and Future Work}
\label{app:limitations}

\textbf{Limitations.}
Our study adopts a controlled setting to isolate phenotype-supervised GEM repair. All GEMs are reconstructed with CarveMe in the BiGG reaction system, which ensures a consistent reaction vocabulary and modeling convention across organisms but does not yet test transfer across reconstruction pipelines. We also use standardized category-specific media as reproducible proxies for BacDive growth conditions. This abstraction is necessary because many BacDive assays provide heterogeneous or incomplete culture descriptions, but it does not capture all strain-specific environmental details. Finally, our verifier focuses on steady-state metabolic feasibility and does not model regulatory or kinetic effects.


\textbf{Future work.}
Future work can extend the framework to GEMs reconstructed by other pipelines, richer or uncertainty-aware culture-condition models, and additional biological constraints beyond FBA. More broadly, combining quotient-space repair with targeted experimental validation could connect phenotype-consistent computational repairs to testable metabolic hypotheses. Finally, quotient-space exploration may extend well beyond metabolic modeling to broader scientific domains such as molecular design, reaction mechanism discovery, and symbolic physical modeling, where macroscopic observations underdetermine underlying structures and search will benefit from operating directly over functional equivalence classes.

\makeatletter
\setlength{\@fptop}{0pt}
\setlength{\@fpsep}{10pt plus 2pt minus 2pt}
\setlength{\@fpbot}{0pt plus 1fil}
\makeatother

\begin{figure*}[t]
    \centering
    \includegraphics[width=\textwidth]{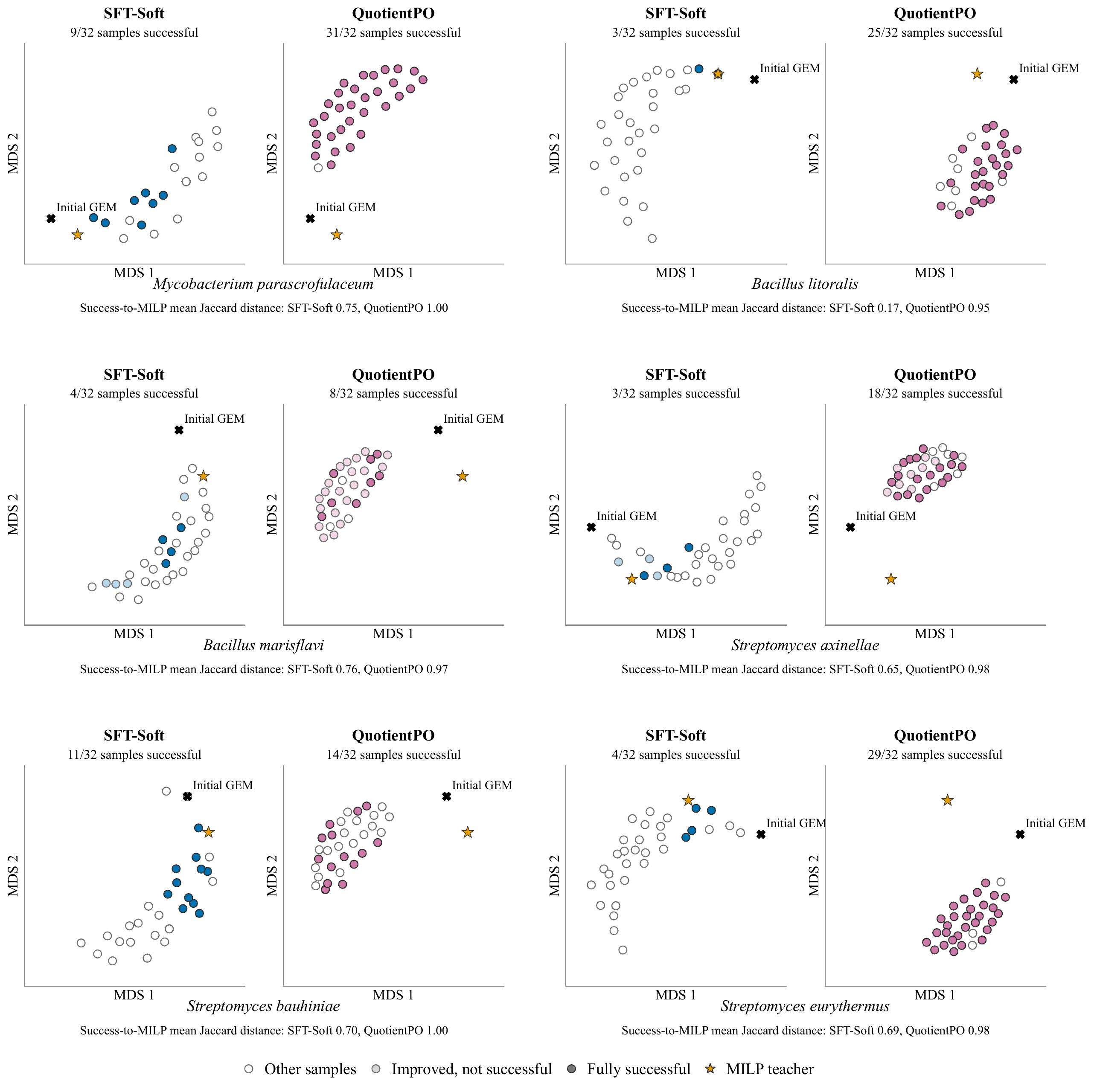}
    \caption{
    \textbf{Cases of SFT versus RL Optimization.}
    }
    \label{fig:sft_rl_cases}
    \vspace{-0.2cm}
\end{figure*}

\begin{figure*}[t]
    \centering
    \includegraphics[width=\textwidth]{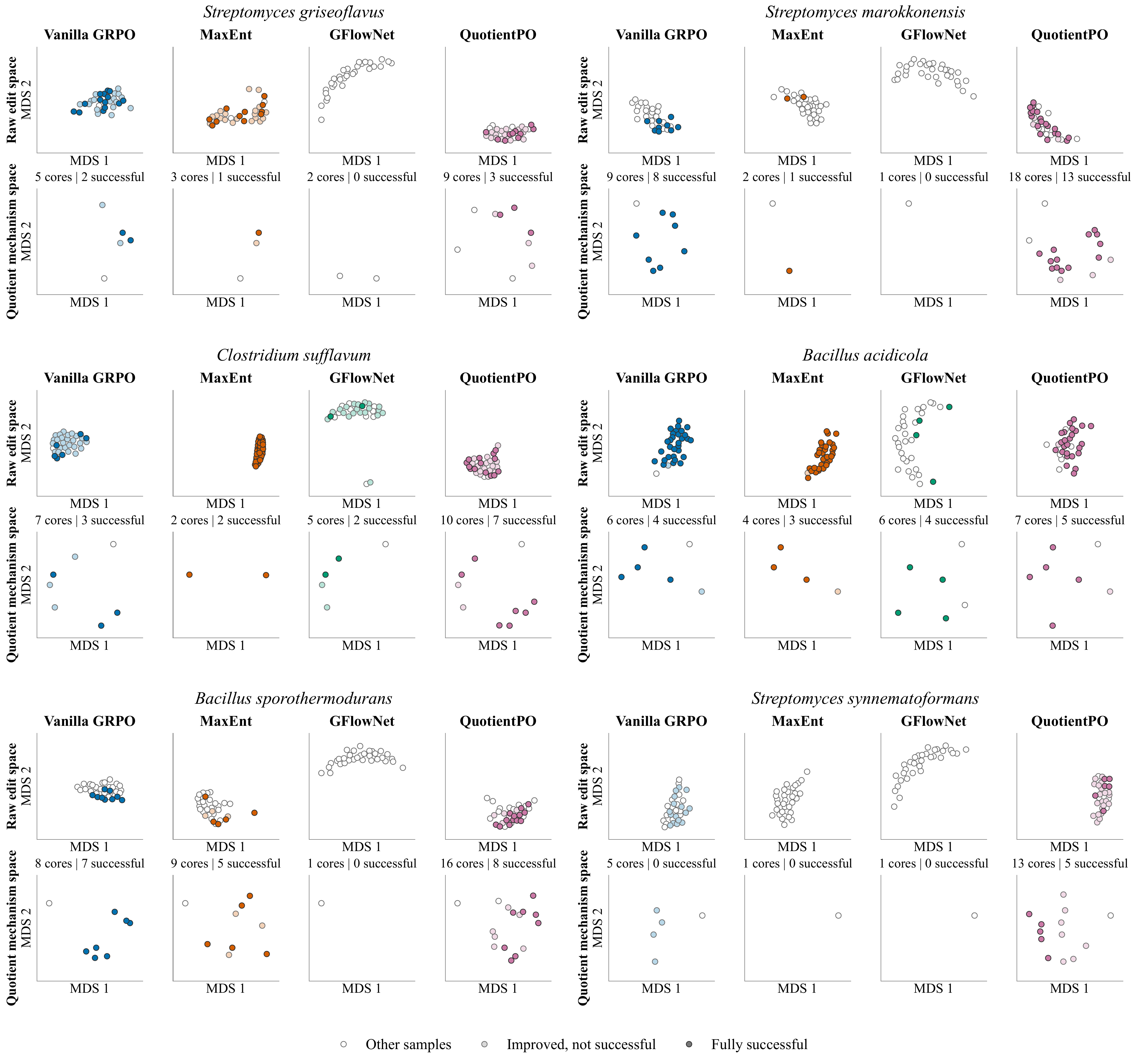}
    \caption{
    \textbf{Cases of solution-space visualization across RL baselines.}
    }
    \label{fig:rl_cases}
    \vspace{-0.2cm}
\end{figure*}


\begin{table*}[t]
\centering
\setlength{\belowcaptionskip}{8pt}
\caption{\textbf{Tier 1: retained-solution coverage.}
The focal alternative appears in successful Vanilla raw proposals ($S_V>0$)
but is retained only by QuotientPO after canonical reduction
($C_V=0$, $C_Q>0$).}
\label{tab:case-tier1}
\small
\setlength{\tabcolsep}{0pt}
\renewcommand{\arraystretch}{1.14}
\setlength{\abovedisplayskip}{4pt}
\setlength{\belowdisplayskip}{4pt}
\setlength{\abovedisplayshortskip}{4pt}
\setlength{\belowdisplayshortskip}{4pt}
\begin{tabularx}{\linewidth}{@{}>{\raggedright\arraybackslash}X@{\hspace{14pt}}>{\raggedright\arraybackslash}X@{}}
\toprule
\textbf{Repair alternatives \& encoded chemistry}
& \textbf{Functional verification \& policy diagnosis} \\
\midrule
\multicolumn{2}{@{}p{\linewidth}@{}}{%
\textbf{Case 1 (C1):} \textit{Mycobacterium terrae}
\par
\textbf{Target:} Urea utilization\par
\textbf{Focal pipeline ($R\to S\to C$):}
$V$: $13\to2\to0$ \quad vs\quad $Q$: $13\to2\to1$\par
\textbf{Yield ($N_{\mathrm{succ}}\mid K_{\mathrm{core}}$):}
$V$: $4\mid1$ \qquad $Q$: $4\mid2$
} \\
\midrule
\textbf{Shared deletions (3):}\par
\texttt{-HEX1}, \texttt{-MALTabc}, \texttt{-NARK}\par\smallskip
\textbf{Pathway 1: Uptake blockade}\par
\texttt{-UREAt}
\[
\mathrm{urea}_{e}\longrightarrow\mathrm{urea}_{c}
\]
\textbf{Pathway 2: Carboxylation blockade (focal)}\par
\texttt{-UREASE}
\[
\begin{aligned}
&\mathrm{urea}+\mathrm{ATP}+\mathrm{HCO}_{3}^{-}\\
&\quad\longrightarrow\mathrm{allophanate}\\
&\qquad+\mathrm{ADP}+\mathrm{P}_{\mathrm{i}}+\mathrm{H}^{+}
\end{aligned}
\]
\textit{Note:} \texttt{UREASE} encodes urea carboxylase in the GEM.
Both canonical cores contain exactly four deletions.
&
\textbf{Phenotypic verification}\par
Both cores reduce urea-supported growth from $0.202$ to zero while
preserving nitrate-supported growth of $0.545$ and all remaining targets.
Reintroducing the focal deletion restores the failed urea-growth phenotype,
confirming its functional necessity.\par\medskip
\textbf{Search behavior diagnosis}\par
Vanilla produces successful raw proposals containing this intracellular
branch ($S_V=2$), but completely eliminates it during canonical reduction
($C_V=0$). In contrast, QuotientPO successfully retains both
functional intervention sites ($C_Q=1$).
\\
\addlinespace[8pt]
\midrule
\multicolumn{2}{@{}p{\linewidth}@{}}{%
\textbf{Case 2 (C2):} \textit{Streptomyces inusitatus}
\par
\textbf{Target:} Alternative proline supply\par
\textbf{Focal pipeline ($R\to S\to C$):}
$V$: $32\to32\to0$ \quad vs\quad $Q$: $32\to30\to2$\par
\textbf{Yield ($N_{\mathrm{succ}}\mid K_{\mathrm{core}}$):}
$V$: $32\mid2$ \qquad $Q$: $30\mid2$
} \\
\midrule
\textbf{Shared deletion (1):} \texttt{-UREAt}\par\smallskip
\textbf{Pathway 1: Direct conversion}\par
\texttt{+ORNCD}
\[
\mathrm{ornithine}\longrightarrow\mathrm{proline}+\mathrm{NH}_{4}^{+}
\]
\textbf{Pathway 2: Semialdehyde/P5C route (focal)}\par
\texttt{+G5SADs}
\[
\begin{aligned}
\mathrm{GSA}&\longrightarrow\mathrm{P5C}+\mathrm{H}_{2}\mathrm{O}+\mathrm{H}^{+}\\
\mathrm{P5C}&\xrightarrow{\text{existing P5CR/P5CRx}}\mathrm{proline}
\end{aligned}
\]
\textit{Note:} GSA denotes glutamate-5-semialdehyde;
P5C denotes $\Delta^1$-pyrroline-5-carboxylate.
Both canonical cores contain two edits.
&
\textbf{Phenotypic verification}\par
Both two-edit cores restore proline supply and recover citrate-supported
growth from zero to $0.831$ while preserving the inactive urea target.
Flux witnesses independently confirm active throughput along both pathways.\par\medskip
\textbf{Search behavior diagnosis}\par
Vanilla already contains the focal G5SADs route in successful proposals but loses it after canonicalization (\(C_V=0\)); QuotientPO retains it (\(C_Q=2\)). Thus, this is a retention rather than discovery gap. ORNCD and ORNCD$\_$2 are treated as one functional pathway family.
\\
\bottomrule
\end{tabularx}
\end{table*}


\clearpage
\begin{table*}[t]
\centering
\setlength{\belowcaptionskip}{8pt}
\caption{\textbf{Tier 2: successful-combination coverage.}
Vanilla proposes the focal module ($R_V>0$) but does not place it in a
successful complete repair ($S_V=0$), whereas QuotientPO does.}
\label{tab:case-tier2}
\small
\setlength{\tabcolsep}{0pt}
\renewcommand{\arraystretch}{1.14}
\setlength{\abovedisplayskip}{4pt}
\setlength{\belowdisplayskip}{4pt}
\setlength{\abovedisplayshortskip}{4pt}
\setlength{\belowdisplayshortskip}{4pt}
\begin{tabularx}{\linewidth}{@{}>{\raggedright\arraybackslash}X@{\hspace{14pt}}>{\raggedright\arraybackslash}X@{}}
\toprule
\textbf{Repair alternatives \& encoded chemistry}
& \textbf{Functional verification \& policy diagnosis} \\
\midrule
\multicolumn{2}{@{}p{\linewidth}@{}}{%
\textbf{Case 3 (C3):} \textit{Lactobacillus farraginis} \par
\textbf{Target:} Alternative phosphate donors for UTP synthesis\par
\textbf{Focal pipeline ($R\to S\to C$):}
$V$: $4\to0\to0$ \quad vs\quad $Q$: $10\to1\to1$\par
\textbf{Yield ($N_{\mathrm{succ}}\mid K_{\mathrm{core}}$):}
$V$: $0\mid0$ \qquad $Q$: $2\mid2$
} \\
\midrule
\textbf{Shared edits (3):}\par
\texttt{+MOHMT}, \texttt{+NTRIR2y}, \texttt{-NTRIR3pp}\par\smallskip
\textbf{Pathway 1: ATP donor}\par
\texttt{+NDPK2}
\[
\mathrm{ATP}+\mathrm{UDP}\longrightarrow\mathrm{ADP}+\mathrm{UTP}
\]
\textbf{Pathway 2: PEP donor (focal)}\par
\texttt{+PYK2}
\[
\begin{aligned}
&\mathrm{PEP}+\mathrm{UDP}+\mathrm{H}^{+}\\
&\quad\longrightarrow\mathrm{pyruvate}+\mathrm{UTP}
\end{aligned}
\]
Both complete cores contain four edits.
&
\textbf{Phenotypic verification}\par
The alternatives provide UDP-to-UTP phosphorylation using different
phosphate donors.\par\smallskip
Both QuotientPO cores satisfy all six targets, with glucose growth
$0.540171$ and $0.539515$, respectively, and nitrate growth $0.888279$.\par\smallskip
Flux witnesses carry approximately $0.16655$ through \texttt{NDPK2}
or $0.16635$ through \texttt{PYK2}.\par\medskip
\textbf{Search behavior diagnosis}\par
Vanilla proposes \texttt{PYK2} four times, but these proposals lack the
coordinated repair required for success.
Adding \texttt{NTRIR2y} restores the nitrate-positive target.
\\
\addlinespace[8pt]
\midrule
\multicolumn{2}{@{}p{\linewidth}@{}}{%
\textbf{Case 4 (C4):} \textit{Streptomyces ferralitis}\par
\textbf{Target:} Alternative nitrogen donors for asparagine\par
\textbf{Focal pipeline ($R\to S\to C$):}
$V$: $24\to0\to0$ \quad vs\quad $Q$: $29\to3\to1$\par
\textbf{Yield ($N_{\mathrm{succ}}\mid K_{\mathrm{core}}$):}
$V$: $0\mid0$ \qquad $Q$: $3\mid3$
} \\
\midrule
\textbf{Shared edits (3):}\par
\texttt{+ARGDC}, \texttt{-FFSD}, \texttt{-SUCRtex}\par\smallskip
\textbf{Pathway 1: Glutamine donor}\par
\texttt{+ASNS1}
\[
\begin{aligned}
&\mathrm{Asp}+\mathrm{Gln}+\mathrm{ATP}+\mathrm{H}_{2}\mathrm{O}\\
&\quad\longrightarrow\mathrm{Asn}+\mathrm{Glu}\\
&\qquad+\mathrm{AMP}+\mathrm{PP}_{\mathrm{i}}+\mathrm{H}^{+}
\end{aligned}
\]
\textbf{Pathway 2: Ammonium donor (focal)}\par
\texttt{+ASNS2}
\[
\begin{aligned}
&\mathrm{Asp}+\mathrm{NH}_{4}^{+}+\mathrm{ATP}\\
&\quad\longrightarrow\mathrm{Asn}+\mathrm{AMP}\\
&\qquad+\mathrm{PP}_{\mathrm{i}}+\mathrm{H}^{+}
\end{aligned}
\]
&
\textbf{Phenotypic verification}\par
The two reactions complete asparagine synthesis using different nitrogen
donors.\par\smallskip
Both QuotientPO cores satisfy the glucose-positive and sucrose-negative
targets, with glucose growth $1.342551$ and $1.345126$.\par\medskip
\textbf{Search behavior diagnosis}\par
Vanilla frequently proposes the focal \texttt{ASNS2} branch, but the
corresponding repair sets remain unsuccessful.
One such proposal also introduces extracellular sucrose hydrolysis through
\texttt{SUCRe}; removing this additional bypass restores the
sucrose-negative target.\par\smallskip
Thus, successful use of the focal biosynthetic branch depends on its
surrounding repair combination.
\\
\bottomrule
\end{tabularx}
\end{table*}


\clearpage
\begin{table*}[t]
\centering
\setlength{\belowcaptionskip}{8pt}
\caption{\textbf{Tier 3: raw-module coverage.}
The complete focal module is absent from Vanilla raw proposals
($R_V=0$) but appears in successful QuotientPO repairs.}
\label{tab:case-tier3}
\small
\setlength{\tabcolsep}{0pt}
\renewcommand{\arraystretch}{1.14}
\setlength{\abovedisplayskip}{4pt}
\setlength{\belowdisplayskip}{4pt}
\setlength{\abovedisplayshortskip}{4pt}
\setlength{\belowdisplayshortskip}{4pt}
\begin{tabularx}{\linewidth}{@{}>{\raggedright\arraybackslash}X@{\hspace{14pt}}>{\raggedright\arraybackslash}X@{}}
\toprule
\textbf{Repair alternatives \& encoded chemistry}
& \textbf{Functional verification \& policy diagnosis} \\
\midrule
\multicolumn{2}{@{}p{\linewidth}@{}}{%
\textbf{Case 5 (C5):} \textit{Streptomyces olivochromogenes}\par
\textbf{Target:} Alternative control of citrate-supported growth\par
\textbf{Focal pipeline ($R\to S\to C$):}
$V$: $0\to0\to0$ \quad vs\quad $Q$: $2\to2\to2$\par
\textbf{Yield ($N_{\mathrm{succ}}\mid K_{\mathrm{core}}$):}
$V$: $2\mid1$ \qquad $Q$: $8\mid3$
} \\
\midrule

\textbf{Shared edit (1):}\par
\texttt{-UREAt}\par\smallskip

\textbf{Alternative 1: Entry restriction}\par
\texttt{-CITt14}, \texttt{-FE3t}\par
CITt14 directly supports citrate uptake, while FE3t provides a reversible
iron route coupled to the remaining FEDCabc citrate-entry pathway.\par\smallskip

\textbf{Alternative 2: Intracellular node (focal)}\par
\texttt{-SUCOAS}
\[
\mathrm{succinyl\text{-}CoA}+\mathrm{ADP}+\mathrm{P_i}
\longrightarrow
\mathrm{succinate}+\mathrm{CoA}+\mathrm{ATP}.
\]

&
\textbf{Phenotypic verification}\par
Both cores eliminate the initial urea and citrate false positives while
preserving glucose- and sucrose-supported growth.
The entry-restriction core retains glucose/sucrose growth at
$0.630548/1.261097$, whereas the SUCOAS core retains
$0.436023/0.872046$.\par\smallskip

Restoring either \texttt{CITt14} or \texttt{FE3t} breaks only the
citrate-negative target in the first core. Restoring \texttt{SUCOAS}
does the same in the focal core. Restoring \texttt{UREAt} affects only
the urea-negative target.\par\medskip

\textbf{Search behavior diagnosis}\par
Vanilla never proposes \texttt{-SUCOAS} in this rollout panel.
QuotientPO discovers it twice. Both samples are successful and
canonicalize to the same focal core.
Thus, this comparison exposes an intracellular repair alternative
absent from Vanilla raw proposals, rather than two distinct mechanisms.
\\
\addlinespace[8pt]
\midrule
\multicolumn{2}{@{}p{\linewidth}@{}}{%
\textbf{Case 6 (C6):} \textit{Streptomyces durbertensis}\quad \par
\textbf{Target:} Alternative quinone-precursor supply\par
\textbf{Focal pipeline ($R\to S\to C$):}
$V$: $0\to0\to0$ \quad vs\quad $Q$: $5\to1\to1$\par
\textbf{Yield ($N_{\mathrm{succ}}\mid K_{\mathrm{core}}$):}
$V$: $1\mid1$ \qquad $Q$: $7\mid5$
} \\
\midrule
\textbf{Shared edits (5):}\par
\texttt{+HSERTA}, \texttt{+HSST}, \texttt{+THZPSN2\_SC},
\texttt{-NO3abc}, \texttt{-NO3tex}\par\smallskip
\textbf{Pathway 1: Aromatic branch}\par
\texttt{+PHPYROX}
\[
\begin{aligned}
&\mathrm{phenylpyruvate}\longrightarrow\mathrm{benzaldehyde}\\
&\quad\longrightarrow\mathrm{benzoate}\longrightarrow\mathrm{benzoyl\text{-}CoA}\\
&\quad\xrightarrow{\text{reverse BSCT}}\mathrm{OSB\text{-}CoA}
\end{aligned}
\]
\textbf{Pathway 2: Three-ADD branch (focal)}\par
\texttt{+SHCHCS}, \texttt{+SUCBZS}, \texttt{+SUCBZL\_1}
\[
\begin{aligned}
&\mathrm{isochorismate}+\mathrm{SSA}\\
&\quad\longrightarrow\mathrm{SHCHC}\longrightarrow\mathrm{OSB}\\
&\quad\longrightarrow\mathrm{OSB\text{-}CoA}
\end{aligned}
\]
\textit{Note:} SSA denotes succinic semialdehyde and OSB denotes
$o$-succinylbenzoate.
&
\textbf{Phenotypic verification}\par
Both branches converge on OSB-CoA and downstream menaquinone-precursor
supply.\par\smallskip
The resulting cores satisfy all four targets and yield glucose growth of
approximately $0.454946$.\par\smallskip
Minimum-total-flux witnesses use the corresponding precursor branch at
approximately $4.55\times10^{-5}$, matching the quinone demand induced
by biomass production.\par\medskip
\textbf{Search behavior diagnosis}\par
The complete three-ADD module does not occur in Vanilla raw proposals,
whereas QuotientPO proposes it five times and retains one successful
canonical core.
Its five total distinct cores therefore correspond to multiple edit sets,
rather than five distinct quinone-supply mechanisms.
\\
\bottomrule
\end{tabularx}
\end{table*}

\end{document}